\documentclass{article}

\PassOptionsToPackage{numbers, compress}{natbib}
\usepackage[preprint]{neurips_2026}

\usepackage[utf8]{inputenc}
\usepackage[T1]{fontenc}
\usepackage{hyperref}
\usepackage{url}
\usepackage{booktabs}
\usepackage{amsfonts}
\usepackage{nicefrac}
\usepackage{microtype}
\usepackage{xcolor}

\usepackage{graphicx}
\usepackage{subcaption}

\usepackage{amsmath}
\usepackage{amssymb}
\usepackage{mathtools}
\usepackage{amsthm}

\usepackage{algorithm}
\usepackage{algorithmic}

\usepackage{pgfplots}
\pgfplotsset{compat=1.14}
\graphicspath{{./img/}{./images/}}
\usepackage{bbm}
\usepackage{comment}
\usepackage{bm}
\usepackage{multirow}
\usepackage{listings}
\usepackage{letltxmacro}
\newcommand*{\SavedLstInline}{}
\LetLtxMacro\SavedLstInline\lstinline
\DeclareRobustCommand*{\lstinline}{%
  \ifmmode
    \let\SavedBGroup\bgroup
    \def\bgroup{%
      \let\bgroup\SavedBGroup
      \hbox\bgroup
    }%
  \fi
  \SavedLstInline
}

\usepackage{tikz}
\usetikzlibrary{arrows.meta}
\usetikzlibrary{angles}
\usepackage{tikz-network}
\usetikzlibrary{patterns,decorations.pathreplacing,math}
\usetikzlibrary{decorations.markings}

\tikzset{
    arrowhead/.pic = {
    \draw[thick, rotate = 45] (0,0) -- (#1,0);
    \draw[thick, rotate = 45] (0,0) -- (0, #1);
    }
}

\definecolor{mygreen}{RGB}{34,139,34}
\definecolor{mypink}{RGB}{255,100,203}

\usepackage[capitalize,noabbrev]{cleveref}

\theoremstyle{plain}
\newtheorem{theorem}{Theorem}[section]

\newtheorem{lemma}[theorem]{Lemma}

\theoremstyle{definition}

\theoremstyle{remark}

\newif\ifshowdraftcomments
\showdraftcommentstrue %\showdraftcommentsfalse

\ifshowdraftcomments
  \usepackage[textsize=tiny]{todonotes}

\else
  \usepackage[disable,textsize=tiny]{todonotes}

\fi

\newcommand{\AlgComment}[1]{\hfill{\footnotesize\textcolor{blue}{$\triangleright$~#1}}}

\title{Manifold Random Features}

\author{
Ananya Parashar, Derek Long, Dwaipayan Saha\\
Department of Industrial Engineering and Operations Research \\
Columbia University, New York, NY 10027\\
\texttt{\{ap4658,dl3538,ds4386\}@columbia.edu}
\And
Krzysztof Choromanski\\
Department of Industrial Engineering and Operations Research \\
Columbia University, New York, NY 10027\\
\texttt{\{ap4658,dl3538,ds4386\}@columbia.edu}
}

\author{\hspace{-1.mm}
Ananya Parashar\textsuperscript{$1\,^{*}$}, Derek Long\textsuperscript{$1\,^{*}$}, Dwaipayan Saha\textsuperscript{$1\,^{*}$}, Krzysztof Choromanski\textsuperscript{$1,2$ \thanks{equal contribution} $\,\,$\thanks{Senior lead}}
\vspace{1.3mm}\\
\normalfont
\textsuperscript{$1$}Columbia University, 
\textsuperscript{$2$}Google DeepMind
}

\begin{document}

\maketitle

\begin{abstract}
We present a new paradigm for creating random features to approximate bi-variate functions (in particular, kernels) defined on general manifolds. This new mechanism of \textit{Manifold Random Features} (MRFs) leverages discretization of the manifold and the recently introduced technique of \textit{Graph Random Features} (GRFs; \citep{grfs-1}) to learn continuous fields on manifolds. Those fields are used to find continuous approximation mechanisms that otherwise, in general scenarios, cannot be derived analytically. MRFs provide positive and bounded features, a key property for accurate, low-variance approximation. We show deep asymptotic connection between GRFs, defined on discrete graph objects, and continuous random features used for regular kernels. As a by-product of our method, we re-discover recently introduced mechanism of Gaussian kernel approximation applied in particular to improve linear-attention Transformers, considering simple random walks on graphs and by-passing original complex mathematical computations. We complement our algorithm with a rigorous theoretical analysis and verify in thorough experimental studies.
\end{abstract}

%%%%%%%%%%%%%%%%%%%%%%%%%%%%%%%%%%%%%%%%%%%%%%%%%%%%%%%%%%%%%%%%%%%%%%%%%%%%%%%
% Main paper content from the ICML manuscript
%%%%%%%%%%%%%%%%%%%%%%%%%%%%%%%%%%%%%%%%%%%%%%%%%%%%%%%%%%%%%%%%%%%%%%%%%%%%%%%
\section{Introduction \& Related Work}
\label{sec:intro_related}

Random features (RFs) \citep{rf-1, rf-2, rf-3, rf-4, rf-7, rf-6, hrfs, de-rfs, unreas, uortmc, gcmcsampling, ortrfs} provide powerful techniques for translating potentially highly non-linear bi-variate functions into a simple dot-product (linear) kernel via randomized nonlinear transformations $\phi:\mathbb{R}^{d} \rightarrow \mathbb{R}^{m}$ applied separately to two input vectors. Thus RFs lead to effective mappings between nonlinear methods in the original spaces and well-understood and computationally more efficient linear methods in the $\phi$-transformed spaces. Examples range from classic algorithms such as SVM \citep{timit, rfsvm} to feedforward fully-connected architectures \citep{SNNK, kim2024magnituder} and attention techniques in Transformers \citep{performers, slim-performers, rf-attention, unified-rf-attention, linear-attention, hedgehog}. 

\begin{figure*}[!h]
    \centering
    \includegraphics[width=0.99\linewidth]{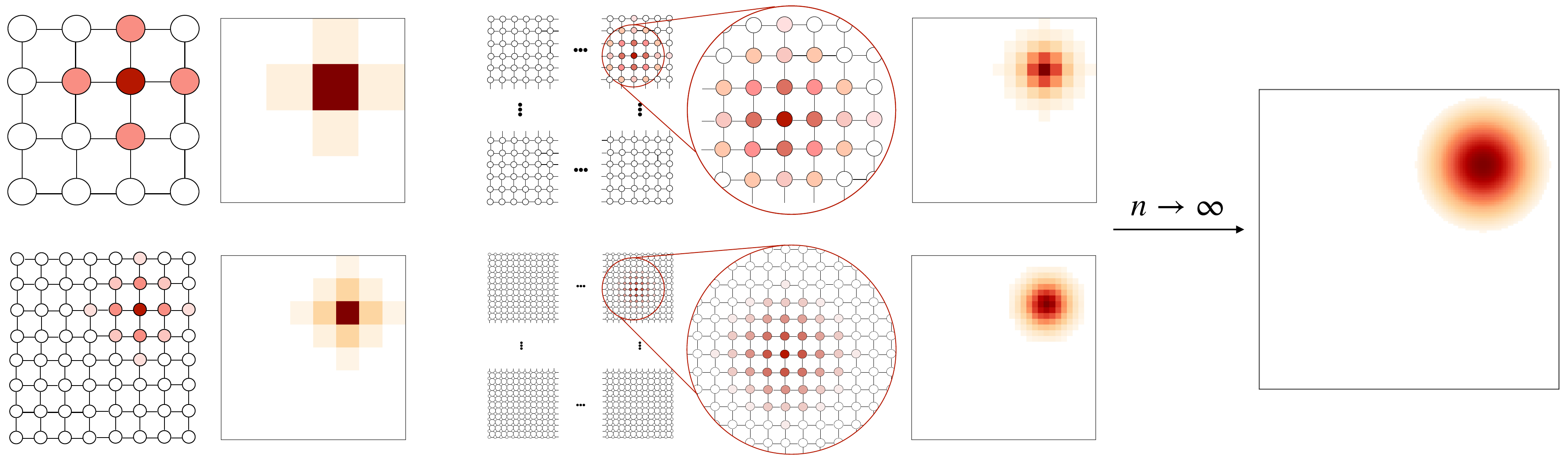}
    \caption{{Four grid-graphs of sizes: $4 \times 4$, $8 \times 8$, $16 \times 16$, $32 \times 32$} with the distinguished vertex $\mathbf{z}$ and its corresponding \textit{signature vectors} obtained by applying GRF algorithm \citep{grfs-1}. The signature vectors are represented by color-coding different vertices (with more intense shades corresponding to large values and the most intense used to color vertex $\mathbf{z}$). Next to the graphs, those signature vectors are also represented by color-coding unit-squares. As the resolution of the grid goes to infinity, those representations converge to the continuous object, namely function $g_{i}$ from Eq. \ref{eq:main_rfs}.}
    \label{fig:euclidean}
\vspace{-3mm}
\end{figure*}

Constructing efficient mappings $\phi_{1},\phi_{2}:\mathbb{R}^{d} \rightarrow \mathbb{R}^{m}$ (in many applications, a symmetric setting $\phi_{1}=\phi_{2}=\phi$ is considered, but a priori it is not a strict requirement) to approximate given bi-variate function $F:\mathbb{R}^{d} \times \mathbb{R}^{d} \rightarrow \mathbb{R}$   (with error $\epsilon \rightarrow0$ as $m, d \rightarrow \infty$, often unbiasedly), as:
\begin{equation}
F(\mathbf{x}, \mathbf{y}) \approx \phi_{1}(\mathbf{x})^{\top}\phi_{2}(\mathbf{y}),    
\end{equation}
is a nontrivial task, even for the well-known kernel-functions, which are some of the most natural targets of the RF-methods. For example, a standard method applying trigonometric random features to approximate Gaussian and softmax kernels, proposed in \cite{rf-1}, in practice does not work well in linear-attention Transformers, since it can produce negative features for the kernel that is strictly positive. The so-called \textit{positive random features} from \citep{performers} address this challenge, but further refinements (e.g. the boundedness of the RFs) are needed to improve approximation quality (see: \citep{chefs}). The problem becomes even more challenging when bi-variate functions $F$ are defined on non-Euclidean spaces, e.g. on manifolds (see: Fig. \ref{fig:first_figure}). 

In the Euclidean setting, a standard approach to constructing mappings $\phi_{1},\phi_{2}$ is by re-writing function $F:\mathbb{R}^{d} \times \mathbb{R}^{d} \rightarrow \mathbb{R}$ as follows for some $g_{1},g_{2}: \mathbb{R}^{d} \times \mathbb{R}^{d} \rightarrow \mathbb{R}$:
\begin{equation}
\label{eq:main_rfs}
F(\mathbf{x},\mathbf{y}) = \int_{\mathbb{R}^{d}} g_{1}(\mathbf{x}, \mathbf{\omega})g_{2}(\mathbf{y}, \mathbf{\omega}) d\omega.   
\end{equation}
Given a probabilistic distribution $P \in \mathcal{P}(\mathbb{R}^{d})$ with the corresponding density function $p:\mathbb{R}^{d} \rightarrow \mathbb{R}_{\geq 0}$, mappings $\phi_{1},\phi_{2}:\mathbb{R}^{d} \rightarrow \mathbb{R}^{m}$ can be then defined as follows for $i \in \{1, 2\}$:
\begin{align}
\begin{split}
\phi_{i}(\mathbf{z})  = \frac{1}{\sqrt{m}}  
\left(\frac{g_{i}(\mathbf{z},\omega_{1})}{\sqrt{p(\omega_{1})}},...,\frac{g_{i}(\mathbf{z},\omega_{m})}{\sqrt{p(\omega_{m})}}\right),
\end{split}    
\end{align}
where $\omega_{1},...,\omega_{m} \sim P$.
It is easy to see that the following holds: $F(\mathbf{x},\mathbf{y}) = \mathbb{E}[\phi_{1}(\mathbf{x})^{\top}\phi_{2}(\mathbf{y})]$. It is assumed that $p(\omega)>0$ for all $\omega \in \mathbb{R}^{d}$. If one wants to approximate $F$ only on a set: $\mathcal{B}_{F} \times \mathcal{B}_{F}$ for some bounded $\mathcal{B}_{F}$, then in practice one can consider only ``truncated'' distributions $\mathcal{P}$ with $p$ defined only on some (sufficiently large) bounded $\mathcal{B}_{p}$, such that 
$|\int_{\mathbb{R}^{d} \backslash \mathcal{B}_{p}} g_{1}(\mathbf{x}, \mathbf{\omega})g_{2}(\mathbf{y}, \mathbf{\omega}) d\omega|$ is sufficiently small.
Finding the representations given by Eq. \ref{eq:main_rfs} is usually a nontrivial task and generalizations to non-Euclidean spaces are even more challenging.

In this paper, we present a new paradigm for creating random features to approximate bi-variate functions (e.g. positive-definite kernels) defined on general manifolds. This new mechanism of \textit{Manifold Random Features} (MRFs) leverages discretization of the manifold and the recently introduced technique of \textit{Graph Random Features} (GRFs; \citep{grfs-1, grfs-general, grfs-gps, grfs-plusplus, rrws}) to learn continuous fields on manifolds, corresponding to functions $g_{i}$ from Eq. \ref{eq:main_rfs}, but in the general (not necessarily Euclidean) setting. It turns out that these fields can be thought of as continuous extensions of the combinatorial objects, the so-called \textit{signature vectors}, the building blocks of the GRF mechanism (see Figure \ref{fig:euclidean}). They are used to find continuous approximation mechanisms that otherwise, in general scenarios, cannot be derived analytically. MRFs provide positive and bounded features, a key property for accurate, low-variance approximation. We show deep asymptotic connection between GRFs, defined on discrete graph objects, and continuous random features used for regular kernels. As a by-product of our method, we rediscover recently introduced mechanism of Gaussian kernel approximation \citep{chefs} applied in particular to improve linear-attention Transformers, considering simple random walks on graphs and by-passing original complex computations. We complement our algorithm with a rigorous theoretical analysis and verify it in thorough experimental studies. 

%\begin{figure*}
%    \centering
%    \includegraphics[width=0.99\linewidth]{img/placeholder-2d-convergence.pdf} 
%    \caption{\small{...}}
%    \label{fig:2d-convergence}
%\end{figure*} 

% \vspace{-2mm}
\section{Manifold Random Features}\label{sec:mrfs}

\subsection{Preliminaries: Graph Random Features}
\label{sec:grfs}
MRFs are obtained by training continuous random feature mechanisms via discrete supervision given by \textit{Graph Random Features} (GRFs). Our first step is to provide below a gentle introduction to GRFs.

We will consider weighted, undirected N-vertex graphs $\mathrm{G}(\mathrm{V},\mathrm{E}, \mathbf{W}=[w(i,j)]_{i,j \in \mathrm{V}})$, where (1) $\mathrm{V}$ is a set of vertices, (2) $\mathrm{E} \subseteq \mathrm{V} \times \mathrm{V}$ is a set of undirected edges ($(i, j) \in \mathrm{E}$ indicates that there is an edge between $i$ and $j$ in $\mathrm{G}$), and (3) $\mathbf{W} \in \mathbb{R}_{\geq 0}^{N \times N}$ is a weighted adjacency matrix (no-edges encoded by zeros). 

We consider the following kernel matrix $\mathbf{K}_{\boldsymbol{\alpha}}(\mathbf{W})\in \mathbb{R}^{N \times N}$, where $\boldsymbol{\alpha}=(\alpha_{k})_{k=0}^{\infty}$ and $\alpha_{k} \in \mathbb{R}$:
\begin{equation}
\label{eq:main}
\mathbf{K}_{\boldsymbol{\alpha}}(\mathbf{W}) = \sum_{k=0}^{\infty} \alpha_{k} \mathbf{W}^{k}.    
\end{equation}
For bounded $(\alpha_{k})_{k=0}^{\infty}$ and $\|\mathbf{W}\|_{\infty}$ small enough, the above sum converges.
The matrix $\mathbf{K}_{\boldsymbol{\alpha}}(\mathbf{W})$ defines a kernel on the nodes of the graph. Importantly, Eq.~\ref{eq:main} covers as special cases several classes of graph kernels, in particular \textit{graph diffusion/heat kernels}, that will play an important role in our analysis later (see: Sec. \ref{sec:diffusion_kernel}).

GRFs provide a way to express $\mathbf{K}_{\boldsymbol{\alpha}}(\mathbf{W})$ (in expectation) as
$\mathbf{K}_{\boldsymbol{\alpha}}(\mathbf{W}) \overset{\mathbb{E}}{=} \mathbf{K}_{1}\mathbf{K}_{2}^{\top}$,
for independently sampled $\mathbf{K}_{1},\mathbf{K}_{2} \in \mathbb{R}^{N \times d}$ and some $d \leq N$.
%In practical applications, $\mathbf{K}_{1},\mathbf{K}_{2}$ are often sparse. 
This factorization gives an efficient (sub-quadratic) and unbiased approximation of the matrix-vector products $\mathbf{K}_{\boldsymbol{\alpha}}(\mathbf{W})\mathbf{x}$ as $\mathbf{K}_{1}(\mathbf{K}_{2}^{\top}\mathbf{x})$, if $\mathbf{K}_{1},\mathbf{K}_{2}$ are sparse or $d=o(N)$.
This is often the case in practice.
If this does not hold, explicitly materializing $\mathbf{K}_1 \mathbf{K}_2^\top$ enables one to approximate $\mathbf{K}_{\boldsymbol{\alpha}}(\mathbf{W})$ in quadratic (versus cubic) time.
Next, we describe the base GRF method for constructing sparse $\mathbf{K}_{1}, \mathbf{K}_{2}$ for $d=N$. 
Extensions for $d=o(N)$, using the Johnson-Lindenstrauss Transform \citep{jlt-main}, can be found in \citep{grfs-1}.
Each $\mathbf{K}_{j}$ for $j \in \{1,2\}$ is obtained by row-wise stacking of the vectors $\phi_{f}(i) \in \mathbb{R}^{N}$ for $i \in \mathrm{V}$, where $f:\mathbb{R} \rightarrow \mathbb{R}$ is the graph-kernel specific \textit{modulation function}. 
The procedure to construct random vectors $\phi_{f}(i)$ is given in Algorithm \ref{alg:constructing_rfs_for_k_alpha}. 
Intuitively, one samples an ensemble of random walks (RWs) from each node $i \in \mathrm{V}$. 
Every time a RW visits a node, the scalar value in that node is updated by the renormalized \textit{load}, a scalar stored by each walker, updated during each vertex-to-vertex transition. The renormalization is encoded by the \textit{modulation function}.

After all walks terminate, the vector $\phi_{f}(i)$ is obtained by the concatenation of all the scalars from the discrete scalar field, followed by a simple renormalization. 
For unbiased estimation, $f:\mathbb{N} \rightarrow \mathbb{C}$ needs to satisfy $\sum_{p=0}^{k}f(k-p)f(p) = \alpha_{k}$, 
% \begin{equation}
% \sum_{p=0}^{k}f(k-p)f(p) = \alpha_{k},    
% \end{equation}
for $k=0,1,...$ (see Theorem 2.1 in \citep{grfs-general}). We conclude that modulation function $f$ is obtained by de-convolving sequence $\boldsymbol{\alpha}=(\alpha_{k})_{k=0}^{\infty}$ defining graph kernel.
% \vspace{-3mm}
\subsection{Constructing MRFs}
\label{subsec:mrfs}
% \vspace{-1.8mm}
For a given Riemannian manifold $\mathcal{M}$, and a bi-variate function $F:\mathcal{M} \times \mathcal{M} \rightarrow \mathbb{R}$, we are looking for the function $\phi_{\mathcal{M},F}:\mathcal{M} \rightarrow \mathbb{R}^{m}$, such that for $x,y \in \mathcal{M}$:
\begin{equation}
F(x,y) \approx \phi_{\mathcal{M},F}(x)^{\top}\phi_{\mathcal{M},F}(y).    
\end{equation}
We refer to $\phi_{\mathcal{M},F}(z)$ as a \textit{manifold random features vector} for $z$.
We are ready to explain how MRFs are constructed. We start with the pre-processing procedure, conducted only once for a given manifold $\mathcal{M}$, before inputs $z_{1},...,z_{m}$ (for which vectors $\phi_{\mathcal{M},F}(z_{i})$ need to be computed) are even known.
The procedure consists of the following two steps:

\textbf{Discretizations of the input manifold $\mathcal{M}$:} We choose a resolution parameter $N \in \mathbb{N}$. We then choose a finite set of points
$\mathrm{V}_N = \{x_1,\dots,x_N\} \subset \mathcal{M}$
that discretizes $\mathcal{M}$; for example, $\mathrm{V}_N$ can arise from a mesh, a quasi--uniform point cloud, or samples from the Riemannian volume measure. We equip $\mathrm{V}_N$ with a weighted neighborhood graph $\mathrm{G}_N = (\mathrm{V}_N,\mathrm{E}_N,\mathbf{W}_N)$ that encodes the local geometry of $\mathcal{M}$, for instance by connecting each $x_i$ to its $k$ nearest neighbors with weights $\mathbf{W}_N(i,j) = \exp(-\| x_i - x_j\|^2/\sigma^2)$. We then consider a discretized version $F^{\mathrm{disc}}:\mathrm{V}_{N} \times \mathrm{V}_{N} \rightarrow \mathbb{R}$ of the original bi-variate function $F:\mathcal{M} \times \mathcal{M} \rightarrow \mathbb{R}$, with the matrix $\mathbf{F}=[F(x_{i},x_{j})]_{i,j=1,...,N}$ adhering to the form from Eq. \ref{eq:main}. In the most general setting, $F^{\mathrm{disc}}$ can be found by a regular regression algorithm trained on the dataset $\mathcal{D}=\{x_{i},F(x_{i})\}_{i=1}^{N}$ to learn a finite sequence of coefficients $(\alpha_{0},...,\alpha_{K})$ for some $K>0$. However in many prominent cases of bi-variate functions defined on manifolds, this is not necessary.
For instance, for the diffusion/heat kernels $F$ defined on manifolds, the corresponding $F^{\mathrm{disc}}$
are the diffusion/heat kernels on the corresponding graphs $\mathrm{G}_{N}$, with the coefficients $\alpha_{k}$ described by a closed-form formula. We describe this prominent special case in more depth in Sec. \ref{sec:diffusion_kernel}. 

\textbf{Training $g$-functions on manifolds:} We start this phase by constructing a training dataset $\mathcal{T} = \{(x, \omega, \phi_{f}(x)[\omega])\}$, where $x,\omega \in \mathrm{V}_{N}$ and $\phi_{f}(x)[\omega]$ stands for the value of the signature vector $\phi_{f}(x)$ in node $\omega$. 
In practice, $\mathcal{T}$ does not need to include the full $N^2$ set of pairs $(x,\omega)\in \mathrm{V}_n\times \mathrm{V}_n$ since diffusion-type kernels decay rapidly between far-apart nodes and signature vectors corresponding to neighboring nodes are similar.
The signature vectors $\{\phi_{f}(x)\}_{x \in \mathrm{V}_{N}}$ are computed with the use of Algorithm 1 (the GRF method). We construct $\mathcal{T}$ to train a function $g^{\mathcal{M},f}_{\theta}:\mathcal{M} \times \mathcal{M} \rightarrow \mathbb{R}_{\geq 0}$, parameterized by $\theta$ and satisfying:
\begin{equation}
F(x, y) \approx \int_{\mathcal{M}} g^{\mathcal{M},f}_{\theta}(x,\omega)g^{\mathcal{M}, f}_{\theta}(y,\omega)d\omega.    
\end{equation}
Note that $g^{\mathcal{M},f}_{\theta}$ is the generalization of the $g$-functions from Eq. \ref{eq:main_rfs}. We model it as a neural network.

We train $g^{\mathcal{M},f}_\theta$ to regress the discrete signature vectors on $\mathrm{V}_N \times \mathrm{V}_N$, and clamp the final output of $g^{\mathcal{M},f}_{\theta}$ at zero. We define the following objective to minimize (where $\mathcal{L}$ is a $\mathbb{R}^{m} \times \mathbb{R}^{m} \rightarrow \mathbb{R}$ loss function, e.g. $L_{2}$-loss):
\[
\min_\theta \; \frac{1}{|\mathcal{T}|} \sum_{(x,\omega,\phi_{f}(x)[\omega])\in\mathcal{T}}
\mathcal{L}\bigl(g^{\mathcal{M},f}_{\theta}(x,\omega),\, \phi_{f}(x)[\omega]\bigr).
\]
After training, $g^{\mathcal{M},f}_{\theta}$ provides a smooth surrogate for discrete signatures: for $x,\omega\in \mathrm{V}_N$, $g^{\mathcal{M},f}_{\theta}(x,\omega) \approx \phi_{f}(x)(\omega)$. That completes pre-processing.

\textbf{Inference:} Trained $g^{\mathcal{M},f}_{\theta}$ provides a way to approximate $F(x,y)$ for any $x,y \in \mathcal{M}$ as follows:
\begin{equation}
\label{eq:approx_sum}
F(x,y) \approx \sum_{\omega \in V_{N}} g^{\mathcal{M},f}_{\theta}(x,\omega)g^{\mathcal{M},f}_{\theta}(y,\omega).   
\end{equation}
This directly leads to the following definition of $\phi_{\mathcal{M},F}$:
\begin{equation}
\phi_{\mathcal{M},F}(z) = \frac{1}{\sqrt{m\kappa}}\left(\frac{g_{\theta}^{\mathcal{M},f}(z, \omega_{1})}{\sqrt{p(\omega_{1})}},...,\frac{g_{\theta}^{\mathcal{M},f}(z, \omega_{m})}{\sqrt{p(\omega_{m})}}\right),    
\end{equation}
where either: (1) $p \equiv 1$, $m=N$, $\kappa=\frac{1}{m}$, $\{\omega_{1},...,\omega_{m}\}=\mathrm{V}_{N}$ or: (2) $p$ is the density function of the probability distribution $\mathcal{P}(\mathrm{V}_{N})$ on $\mathrm{V}_{N}$, $\kappa=1$ and $\omega_{1},...,\omega_{m} \overset{\mathrm{iid}}{\sim} \mathcal{P}(\mathrm{V}_{N})$. The latter variant leads to an unbiased estimation of the sum from Eq. \ref{eq:approx_sum} (which is exactly what the former variant is designed to compute). It provides $m$-dimensional RF-vectors, instead of $N$-dimensional and thus can be used for dimensionality reduction.

\subsubsection{Diffusion/heat kernels}
\label{sec:diffusion_kernel}

Let $(\mathcal{M},g)$ be a compact $d$--dimensional Riemannian manifold and let $\Delta_{\mathcal{M}}$ be its Laplace--Beltrami operator \citep{laplace-beltrami}. The diffusion/heat kernel on $\mathcal{M}$ is defined as follows, for a given parameter $t>0$:
\[
\mathrm{K}^{\mathrm{heat}}_t(x,y) = \exp\bigl(t\Delta_{\mathcal{M}}\bigr)(x,y), \qquad x,y \in \mathcal{M}.
\]

Let $\mathbf{L}_N$ denote the corresponding (rescaled) random--walk graph Laplacian on the discretized $\mathcal{M}$ (corresponding to the graph $\mathrm{G}_{N}=(\mathrm{V}_N=\{x_{1},...,x_{N}\}, \mathrm{E}_N, \mathbf{W}_N)$), chosen so that $\mathbf{L}_N$ converges to $-\Delta_{\mathcal{M}}$ in the limit, as $N\to\infty$. The discrete graph diffusion/heat kernel at time $t>0$ is defined as follows:
\begin{equation}
\label{eq:heat-kernel}
\mathrm{K}^{\mathrm{heat}}_{N,t}(x_i,x_j) := \bigl[\exp(-t \mathbf{L}_{N})\bigr](i,j).
\end{equation}
It converges pointwise to the manifold heat kernel $\mathrm{K}^{\mathrm{heat}}_t$ as the discretization is refined. It is easy to see that $\mathrm{K}^{\mathrm{heat}}_{N,t}$ adheres to the form from Eq. \ref{eq:main}.

The class of manifold heat kernels provides a good insight into challenges related to using those constructs in machine learning that this paper addresses.

Let $(\mathcal{M},g)$ be a compact $d$--dimensional Riemannian manifold.
Since $\mathcal{M}$ is compact, the Laplace--Beltrami operator $-\Delta_{\mathcal{M}}$ has a discrete,
nonnegative spectrum
\[
0 = \lambda_0 \le \lambda_1 \le \lambda_2 \le \cdots \uparrow \infty,
\]
with the corresponding orthonormal eigenfunctions
$\{\phi_k\}_{k=0}^\infty \subset L^2(\mathcal{M})$ satisfying
$-\Delta_{\mathcal{M}} \phi_k = \lambda_k \phi_k$.
The heat kernel then admits the following expansion:
\begin{equation}
  \mathrm{K}^{\mathrm{heat}}_{t}(x,y)
  \;=\;
  \sum_{k=0}^{\infty}
    e^{-\lambda_k t}\,
    \phi_k(x)\,\phi_k(y),
  \quad
  t>0.
\label{eq:manifold_heat_kernel_spectral}
\end{equation}
This expansion converges absolutely and uniformly for any $t>0$ on compact manifolds. For reference of this result, see Theorem 10.13 of \cite{grigoryan2009heat}.

A naive approach to computing $\mathrm{K}^{\mathrm{heat}}_{t}(x, y)$ is via the truncation of the infinite series from Eq. \ref{eq:manifold_heat_kernel_spectral}.
For a discretization $\{x_1,\dots,x_N\}\subset\mathcal{M}$, one can numerically approximate the first $M$ eigenpairs of $-\Delta_\mathcal{M}$—or of a consistent graph Laplacian approximation $\mathbf{L}_N$—and build the truncated kernel:
\begin{equation}\label{eq:truncated_spectral_kernel}
\widehat{\mathrm{K}}_{t}^{\mathrm{heat}}(x, y)
:= \sum_{k=0}^{M-1} e^{-\lambda_k t}\phi_k(x)\phi_k(y),
\end{equation}
which converges to (\ref{eq:manifold_heat_kernel_spectral}) as $M\to\infty$.
Thus even for this well-studied class of manifold bi-variate functions, the (approximate) computation of the kernel matrix $\mathbf{K}^{\mathrm{heat}}_{t}=[\mathrm{K}^{\mathrm{heat}}(x_i, x_j)]_{i,j=1,...,N} \in \mathbb{R}^{N \times N}$ requires nontrivial spectral calculations of time complexity cubic in $N$. Furthermore, as it is the case in general for kernel methods not leveraging RFs, downstream applications (involving multiplications with matrices $\mathbf{K}^{\mathrm{heat}}_{t}$ or taking their inverses) require time $\Omega(N^{2})$. In contrast, RF-based methods require sub-quadratic time (e.g. linear or log-linear for multiplications with kernel matrices, see: \citep{recycling}). Note that for particularly symmetric manifolds $\mathcal{M}$, such as $2$-dimensional spheres in $\mathbb{R}^{3}$, there exist more explicit formulae for $\mathrm{K}^{\mathrm{heat}}_{t}(x, y)$ (see: Appendix \ref{app:spheres}), but they still require spectral computations.

% \ap{Moved this sec here}
\textbf{Positivity \& Boundedness:} We reiterate that GRFs, by definition (see: Algorithm \ref{alg:constructing_rfs_for_k_alpha}), provide positive random features. Furthermore, our NNs (using GRFs as teachers) are forced to produce such features (via inference-time clamping). Their boundedness follows from the fact that, in our applications (see: Sec. \ref{sec:experiments}), the considered NNs are continuous functions defined on compact manifolds. As mentioned before, both properties (positivity \& boundedness) play an important role in getting accurate estimators.  

As we already discussed in this section, general bi-variate functions $F$ can be handled by MRFs via standard ML regression methods used to fit $F$ to formula from Eq. \ref{eq:main}. We also would like to emphasize that matrices $\mathbf{W}$ used there can be thought of as ``generalized'' weighted adjacency matrices, obtained from original weighted adjacency matrices via various transformations (e.g. degree-based normalization), as discussed in \citep{grfs-general}. Thus in this paper we focused on the (already rich) family of functions defined by Eq. \ref{eq:main}.

\subsection{The curious case of the Gaussian kernel} \label{sec:grf_euclidean}
MRFs are designed to efficiently apply bi-variate functions defined in non-Euclidean spaces. However, as a by-product of the presented methods, quite unexpectedly, we also obtain RF-methods for approximating regular Gaussian kernels defined in the Euclidean spaces, that were only recently introduced. We do it by applying GRFs on grid-graphs. We do think that this result is of independent interest, since it shows how combinatorial tools can lead to purely continuous ML constructions. We provide more details here.

Consider the
following setting in the $d$-dimensional space: the Euclidean hypercube $[0,1]^d$ equipped with the Gaussian (RBF) kernel of the following form:
\begin{equation}
  \mathrm{K}^{\mathrm{Gauss}}_{\sigma}(\mathbf{x},\mathbf{y})
  \;=\;
  \exp\!\left(-\frac{\|\mathbf{x}-\mathbf{y}\|^2}{2\sigma^2}\right),
\end{equation}
for a fixed $\sigma>0$ and $x,y\in[0,1]^d$. We discretize $[0,1]^d$ with the grid of  $n$ rows/columns (so that $N=n^d$):
\begin{equation}
  \mathrm{V}_n
  :=
  \Bigl\{
    (k_1/n,\dots,k_d/n)
    : k_j \in \{0,\dots,n-1\}
  \Bigr\},
\end{equation}
and side length $h_n := 1/n$. We connect each grid point to its $2d$ nearest neighbors using wrap–around at the boundary, obtaining a $d$–dimensional discrete grid graph with vertex set $\mathrm{V}_n$. Let $\mathbf{L}_{n}$ denote the rescaled random–walk graph Laplacian on this graph, constructed as in Theorem \ref{thm:convergenceoflaplacian}, so that $\mathbf{L}_{n}$ converges to the continuous Laplacian $-\Delta$ on the discrete grid with wrap-around (the negated Laplace-Beltrami operator) in the limit $n\to\infty$. Take kernel $\mathrm{K}^{\mathrm{heat}}_{N,t}$ from Eq. \ref{eq:heat-kernel} for $t=\frac{\sigma^{2}}{2}$. As discussed in Sec. \ref{sec:diffusion_kernel},  $\mathrm{K}^{\mathrm{heat}}_{N,t}$ converges pointwise to the heat kernel on $[0, 1]^{d}$, which turns out to be the Gaussian kernel $\mathrm{K}^{\mathrm{Gauss}}_{\sigma}$  (to be more specific, $\mathrm{K}^{\mathrm{heat}}_{N,t}$ converges to the periodized Gaussian kernel; in the non–periodic
Euclidean limit this reduces to the standard Gaussian kernel
$\mathrm{K}^{\mathrm{Gauss}}_{\sigma}$, see: Theorem~\ref{thm:discrete-diffusion-torus-convergence-gaussian}). 

For each grid point $x\in \mathrm{V}_n$ we now construct a
signature vector $\phi_f(x)\in\mathbb{R}_{\geq 0}^{\mathrm{V}_n}$ corresponding to kernel $\mathrm{K}^{\mathrm{heat}}_{N,t}$, as in Algorithm \ref{alg:constructing_rfs_for_k_alpha}. Those signature vectors are intrinsically related to the $g$-functions in the representations of Gaussian kernels applied in new RF-mechanisms to approximate them. By applying techniques from \citep{chefs}, we obtain the following positive and bounded features representation (for completeness, we also provide a proof, see: Theorem~\ref{thm:gaussian-feature-representation-kernel}):
\begin{equation}
  \mathrm{K}^{\mathrm{Gauss}}_{\sigma}(\mathbf{x},\mathbf{y})
  \;=\;
  \int_{\mathbb{R}^d}
      g_\sigma(\mathbf{x},\omega)\,
      g_\sigma(\mathbf{y},\omega)\,d\omega,   
      \quad
      g_\sigma(\mathbf{x},\omega)
      :=
      \Bigl(\frac{2}{\pi\sigma^2}\Bigr)^{\!d/4}
      \exp\Bigl(-\frac{\|\mathbf{x}-\omega\|^2}{\sigma^2}\Bigr).
\end{equation}
To relate the discrete signature vectors $\phi_f(x)$ to the continuous features $g_\sigma(\mathbf{x},*)$, we interpret the sum over $\omega\in \mathrm{V}_n$ as a Riemann approximation of the integral in the expression above.
We show in Appendix~\ref{app:gaussian-grid} after multiplying by the grid–dependent constant
$c_{d,\sigma,n}:=(2\pi\sigma^2)^{d/4}n^{d/2}$,
the rescaled vectors $\psi_f(x)
  :=c_{d,\sigma,n}\cdot\phi_f(x)$
converge to $g_\sigma(\mathbf{x},\omega)$ evaluated
on the grid points, as $n\to\infty$.
Equivalently, for any $\mathbf{x},\mathbf{y}\in[0,1]^d$ and the corresponding $x,y \in \mathrm{V}_{n}$ of the grid-graph, we have:
\begin{equation}
  \lim_{n\to\infty}
  \bigl\langle\psi_f(x),\psi_f(y)\bigr\rangle
  \;=\;
  \mathrm{K}^{\mathrm{Gauss}}_{\sigma}(\mathbf{x},\mathbf{y}).
\end{equation}
We conclude that the signature vectors on increasingly refined
grids realize manifold random features for the Gaussian
kernel on the corresponding continuous cube.

\section{Experiments}\label{sec:experiments}

We validate MRFs in several scenarios. 
As a warm-up and sanity check for the discrete-to-continuous construction, Appendix~\ref{app:gaussian-grid} evaluates MRFs on regular Euclidean grids, where the limiting target is the Gaussian/RBF kernel (see: Sec. \ref{sec:grf_euclidean}). As the grid is refined, the rescaled GRF signatures converge empirically to the analytic positive Gaussian feature map, and the induced inner products converge to the Gaussian kernel. This experiment verifies the mechanism predicted by the theory.
% in Sec. \ref{subsec:sv-convergence} we consider the Euclidean setting with Gaussian kernels (see: Sec. \ref{sec:grf_euclidean}).
We then transition to non-Euclidean spaces, the main target of MRFs, testing MRFs on various 2D surfaces in 3D space in Sec. \ref{sec:manifolds-synthetic}, downstream interpolation tasks on meshes (representing rigid and deformable objects) in Sec. \ref{subsec:vertex_normals}, and then on attention using manifold-valued tokens as well as higher-dimensional space and beyond in Sec. \ref{sec:attention}. Note that MRFs require an offline preprocessing stage: graph construction, random-walk supervision, and training of \(g_\theta\). This cost is worthwhile when the learned features are reused for many out-of-sample evaluations or downstream inference steps, but may not be justified for one-off small problems. Experiments were run using a TPU v6e and an A100 GPU (for neural network training only).

\subsection{2D Surfaces in 3D Spaces}
\label{sec:manifolds-synthetic}
In Table~\ref{table:main}, we evaluate the approximation of $g$-functions and diffusion/heat kernels with MRFs on four embedded non-linear geometries: a sphere, an ellipsoid, a M\"obius strip, and a torus (or donut; detailed results for this shape are in the Appendix: Sec. \ref{sec:torus}). More details of this procedure are in Appendix \ref{app:surfaces}.

\begin{table*}[t]
\centering
\caption{
Summary of MRF accuracy, intrinsic-kernel reconstruction, and inference speed on each manifold.
The columns \(R^2\), Mean RE, and MSE evaluate the learned surrogate \(g_\theta(x,\cdot)\) against GRF signature values on validation nodes.
The kernel columns report relative Frobenius error for reconstructing the intrinsic manifold heat kernel: ``Best Ambient RF'' is the best tuned ambient Euclidean RF baseline among RFF/ORF/PRF/PORF, allowing up to \(4\times\) more features.
Timing is measured on kernel evaluation for 512 out-of-sample manifold points; ``Spectral'' denotes the baseline spectral decomposition method.
Full ambient RF results are in Appendix~\ref{app:ambient-rf}, Table~\ref{tab:tuned-sweep-best}.
}
\label{table:main}
\resizebox{\textwidth}{!}{%
\begin{tabular}{lccccccccc}
\toprule
\textbf{Manifold}
& \multicolumn{3}{c}{\textbf{Surrogate validation}}
& \multicolumn{2}{c}{\textbf{Kernel rel. Frob. error} \(\downarrow\)}
& \multicolumn{3}{c}{\textbf{Kernel evaluation time}} \\
\cmidrule(lr){2-4}
\cmidrule(lr){5-6}
\cmidrule(lr){7-9}
& \(R^2 \uparrow\)
& Mean RE \(\downarrow\)
& MSE \(\downarrow\)
& MRF \(\downarrow\)
& Best Ambient RF \(\downarrow\)
& MRFs Time \(\downarrow\)
& Spectral Time \(\downarrow\)
& Speed-up \(\uparrow\) \\
\midrule
Sphere
& 0.997
& 0.23
& 26.9
& \textbf{0.067}
& 0.101
& \(0.066\pm 3.3\times10^{-5}\)
& \(3.41\pm 3.1\times10^{-5}\)
& \textbf{55.2} \\

Ellipsoid
& 0.995
& 0.37
& 46.7
& \textbf{0.047}
& 0.291
& \(0.058 \pm 2.0\times10^{-5}\)
& \(3.59\pm 1.7\times10^{-5}\)
& \textbf{61.2} \\

M\"obius strip
& 0.997
& 0.11
& 17.9
& \textbf{0.035}
& 0.489
& \(0.059 \pm 2.3\times10^{-5}\)
& \(2.14\pm 1.7\times10^{-5}\)
& \textbf{37.2} \\

Torus
& 0.983
& 0.44
& 115.7
& \textbf{0.062}
& 0.417
& \(0.061 \pm 2.9\times10^{-5}\)
& \(3.62\pm 3.1\times10^{-5}\)
& \textbf{58.1} \\
\bottomrule
\end{tabular}
}
\end{table*}

\begin{figure*}[h]
    \includegraphics[width=0.33\linewidth]{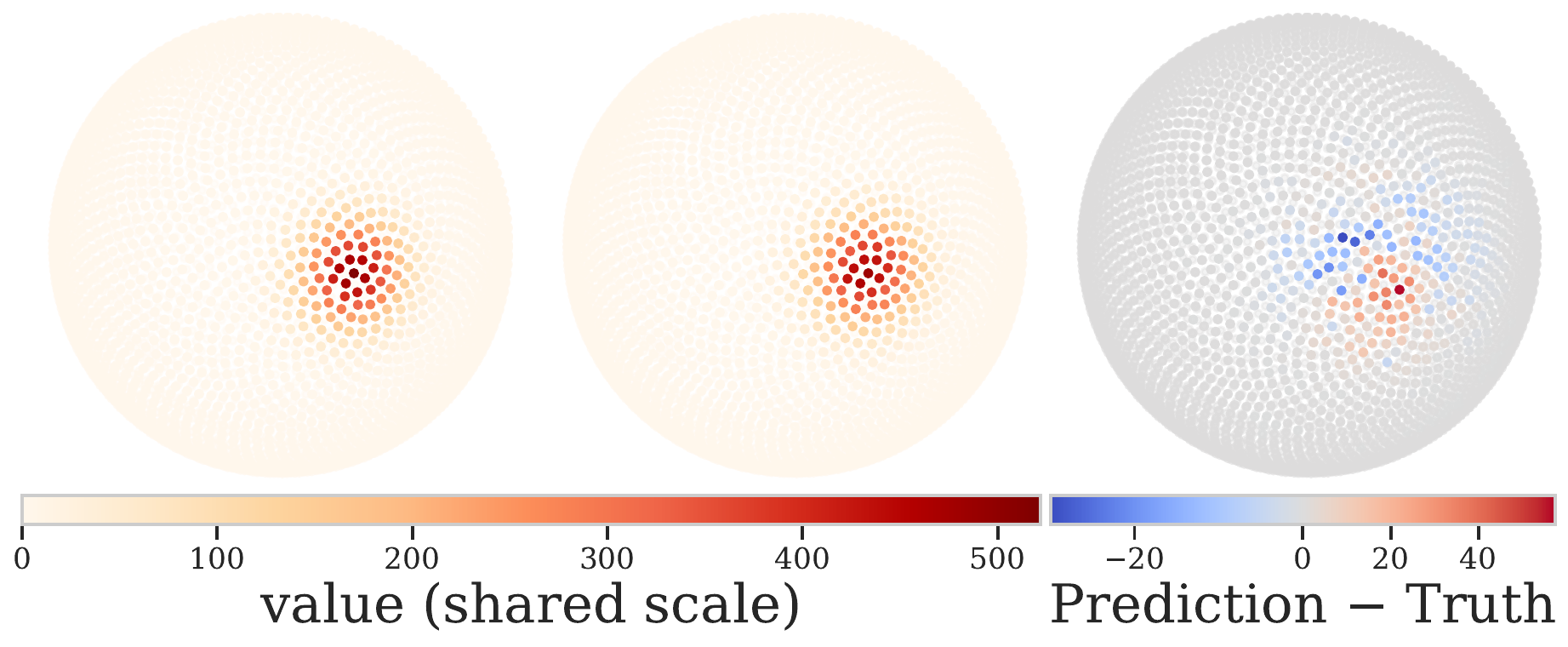}
    \includegraphics[width=0.33\linewidth]{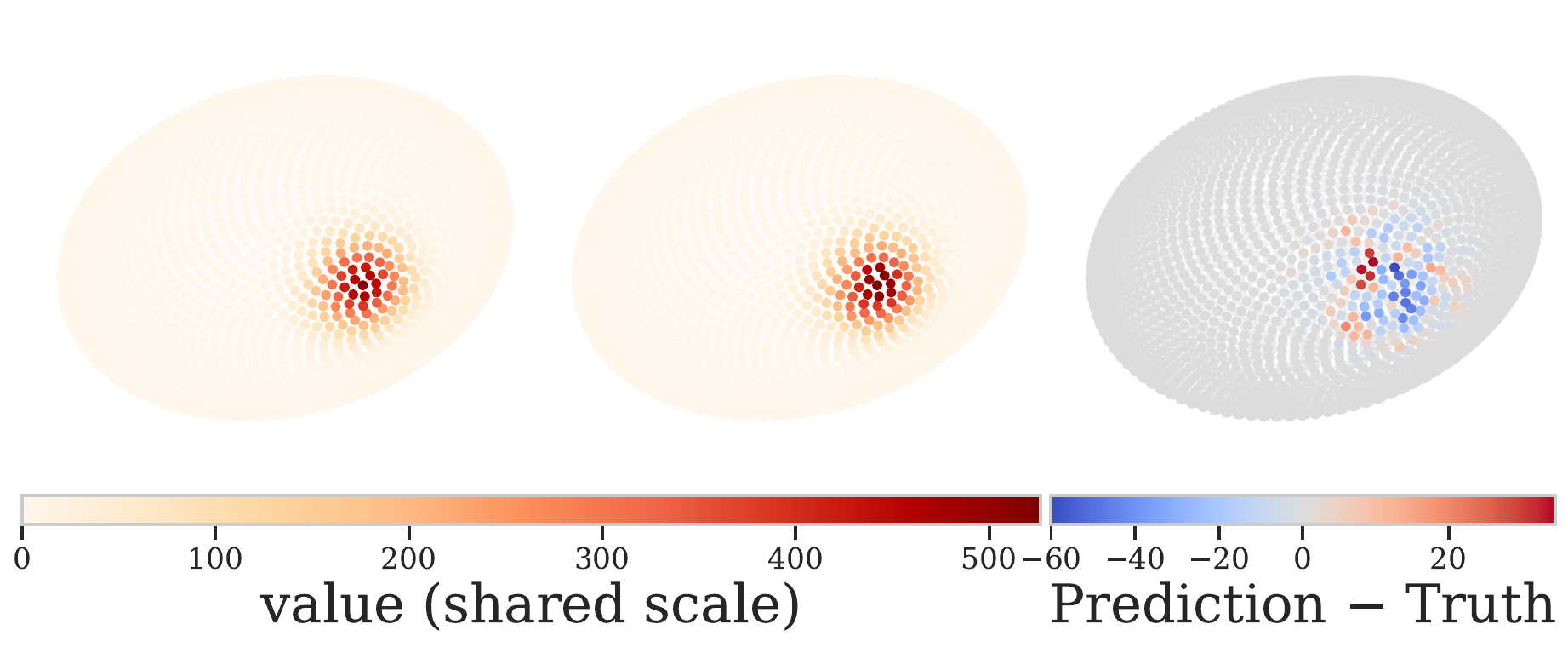}
    \includegraphics[width=0.33\linewidth]{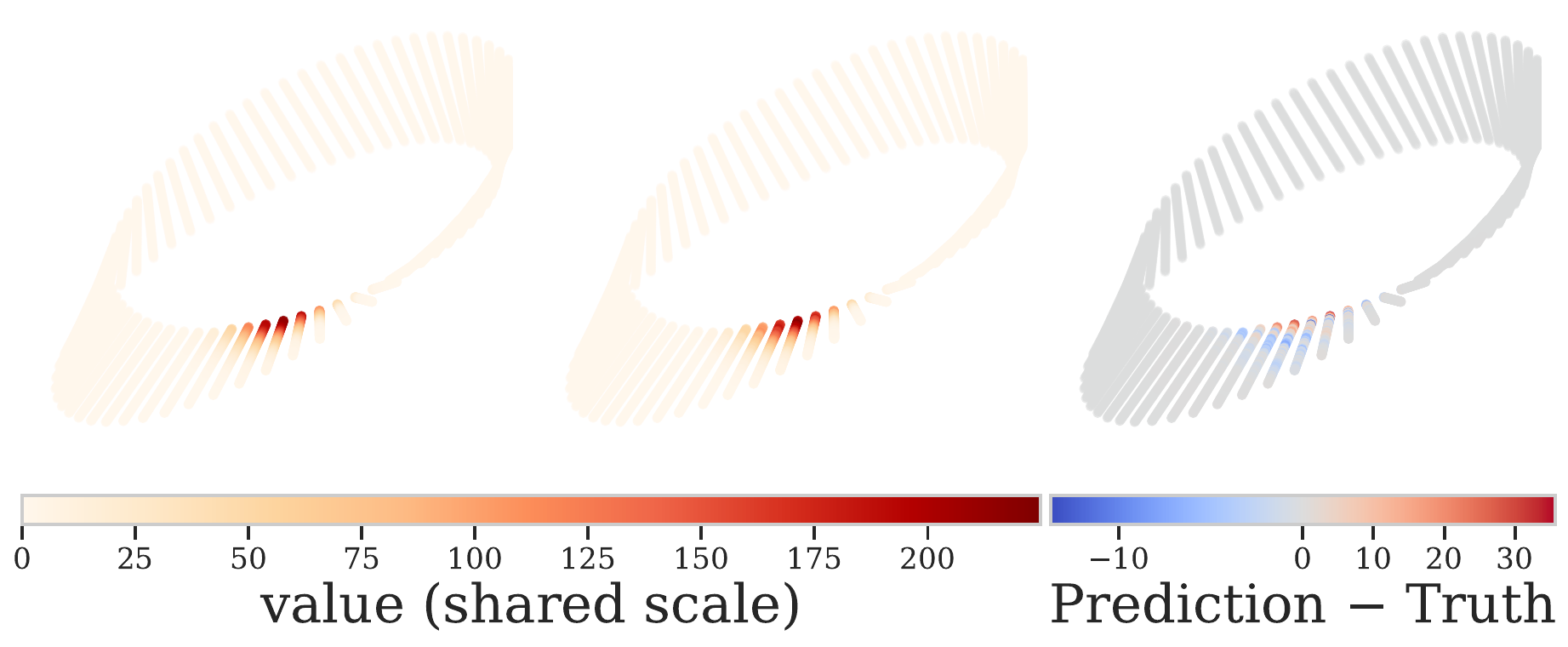}
    \includegraphics[width=0.33\linewidth]{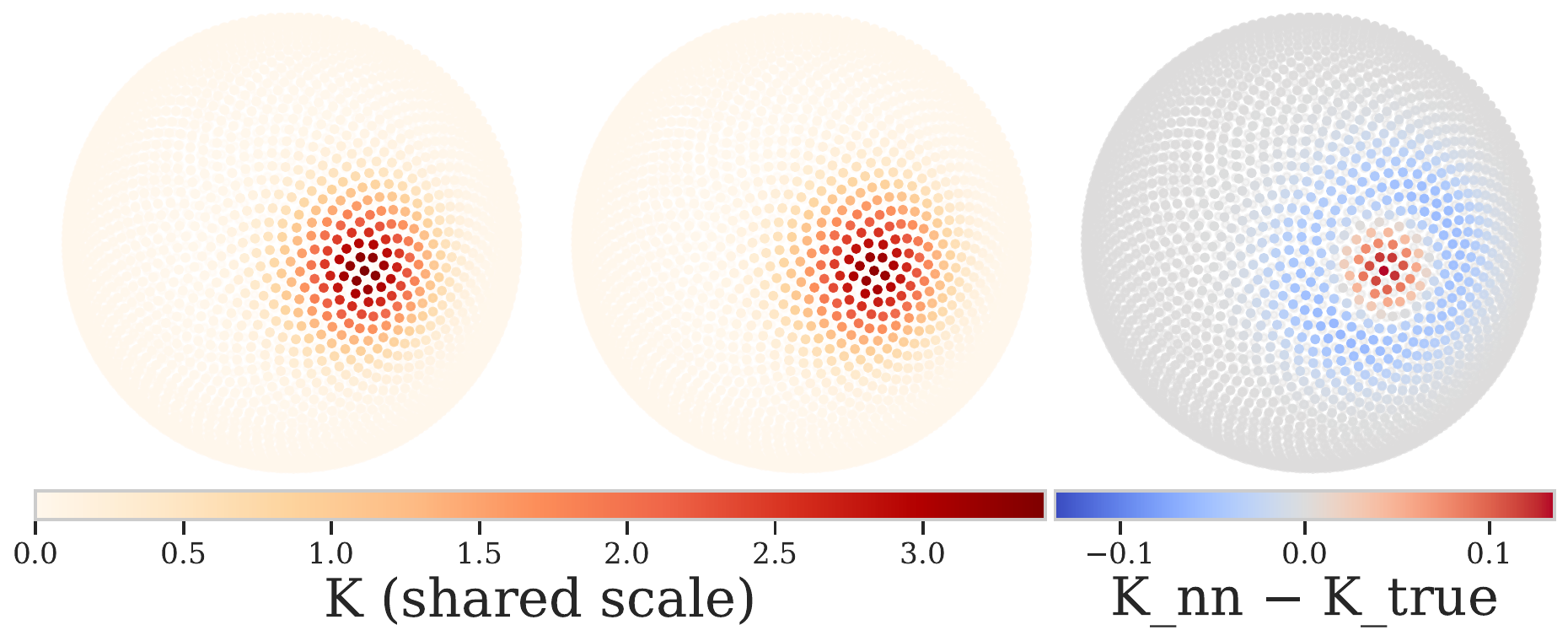}
    \includegraphics[width=0.33\linewidth]{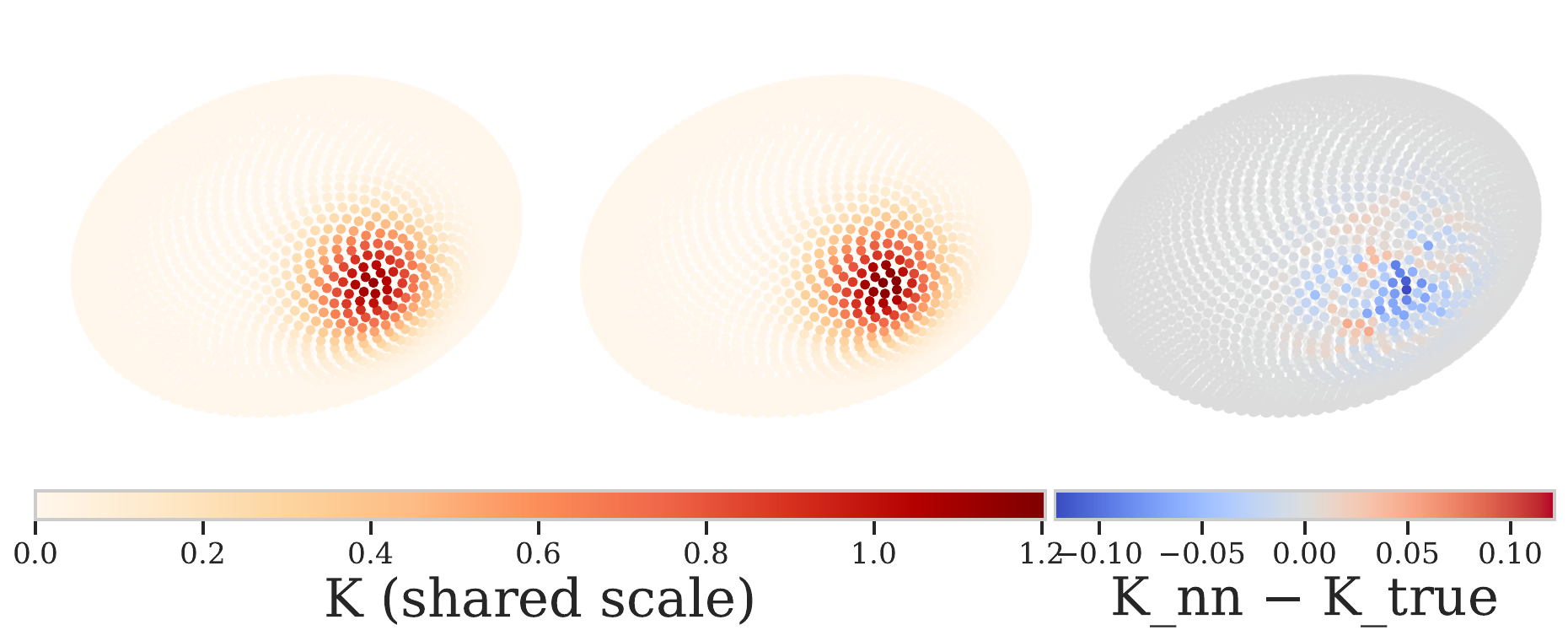}
    \includegraphics[width=0.33\linewidth]{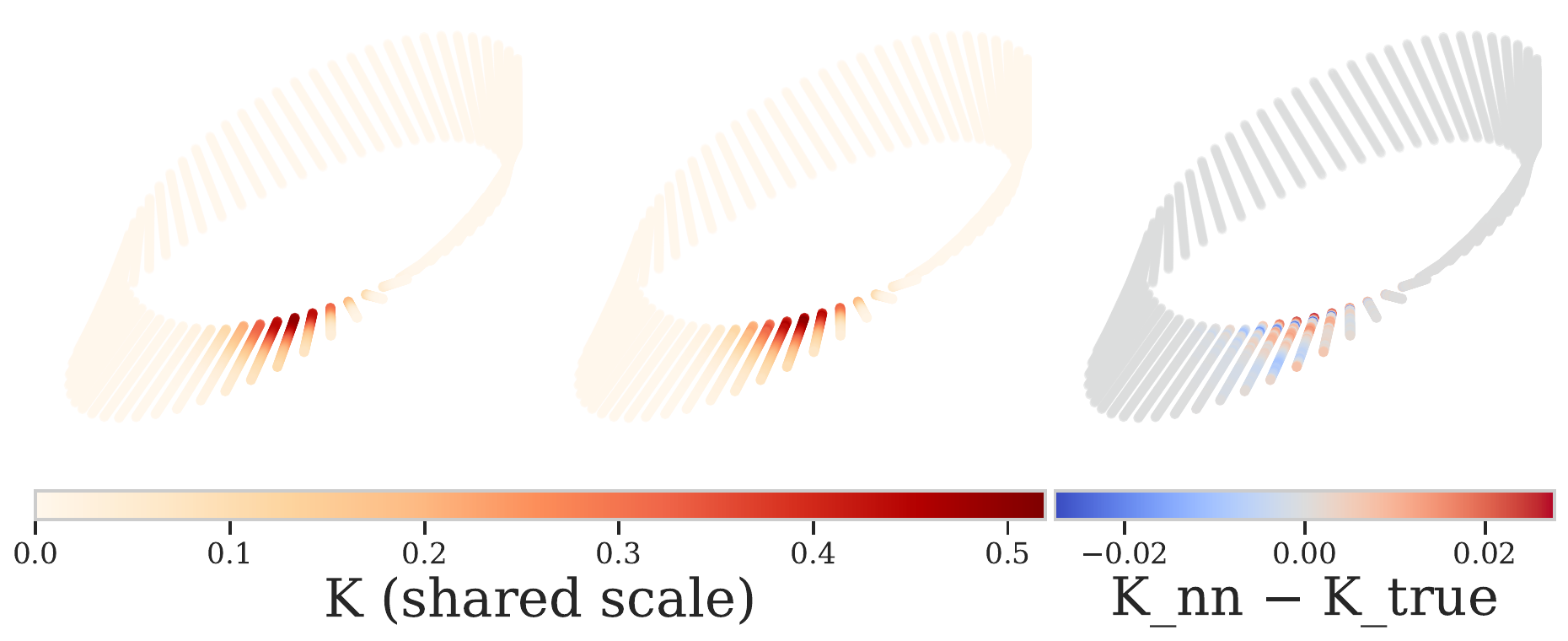}
    \caption{{Qualitative field and heat kernel comparisons for a sphere (left), an ellipsoid (middle), and a M\"obius strip (right). Example validation start (top row): predicted \(g_\theta(\mathbf{x},\cdot)\), ground truth \(\phi_f(x)[\cdot]\), and residual \(g_\theta(\mathbf{x},\cdot)-\phi_f(x)[\cdot]\) on the geometry. Reference graph heat kernel (bottom row) \(\mathrm{K}^{\mathrm{heat}}(\mathbf{x},\cdot)\) compared to Frobenius-aligned induced kernel $\widetilde{\mathrm{K}}_{\theta}(\mathbf{x},\cdot)=\widetilde{\mathrm{K}}_{\mathrm{NN}}(\mathbf{x},\cdot)$, and their difference, for a representative start point.}}
    \label{fig:simulated_manifolds}
    \vspace{-2.5mm}
\end{figure*}

For all considered manifolds, the 3D field plots confirm that the learned predictor reproduces the localized, structured heat profile (for a sphere, radially symmetric) and that residuals are concentrated near the peak (Fig.~\ref{fig:simulated_manifolds}). After Frobenius alignment, the induced kernel matrix $\widetilde{\mathbf{K}}_{\theta}$ matches the analytic kernel matrix, with remaining error again localized near the center. Validation scatter plots show strong monotonic agreement with the ground truth values, with the largest deviations at the highest-intensity region near the source point where the field is most sharply peaked. 
% Detailed results for the torus are in the Appendix \ref{sec:manifolds-synthetic-extra}; the same conclusions as for other shapes). All the results are summarized in Table \ref{table:main}. 
Importantly, MRFs offer significant speedups over the brute-force approaches, from $\mathbf{37x}$ to $\mathbf{61x}+$ in inference.

A natural question is whether manifold-aware features are necessary, or whether one can simply ignore the intrinsic geometry and apply standard Euclidean random features to the ambient coordinates. We therefore compare MRFs against Random Fourier Features (RFF;~\cite{rf-1}), Orthogonal Random Features (ORF;~\cite{ortrfs}), Positive Random Features (PRF;~\cite{performers}),
and Positive Orthogonal Random Features (PORF) (with the best ambient RF baseline after tuning appearing in Table~\ref{table:main}). Each baseline approximates the ambient Gaussian kernel $\exp(-\|\mathbf{x}-\mathbf{y}\|^2/2\sigma^2)$ in \(\mathbb R^3\), with $\sigma$ tuned by grid search, while the target remains the intrinsic manifold heat kernel. The full comparison, including all RF methods and feature budgets \(m\in\{4000,8000,16000\}\), is reported in Appendix~\ref{app:ambient-rf}, Table~\ref{tab:tuned-sweep-best}. MRFs consistently outperform these other methods, showing that the gains  come from respecting the intrinsic geometry of the manifold.

\subsection{Interpolation on Meshes} \label{subsec:vertex_normals}
% \dl{MOVE DETAILS TO APPENDIX}
As a downstream application for MRFs, we choose the task of interpolating various fields defined on meshes representing rigid and deformable objects. 

\textbf{Vertex normal prediction:} Here we focus on predicting normal vectors in the vertices of the mesh representing rigid body. Given a mesh $\mathrm{G}$ with vertex set $\mathrm{V}$ and vertex positions $\{\mathbf{x}_i\in\mathbb{R}^3\}_{i\in \mathrm{V}}$, let $\{\mathbf{n}_i\in\mathbb{R}^3\}_{i\in \mathrm{V}}$ denote the unit vertex normals computed from incident faces. We randomly mask a subset $M\subseteq \mathrm{V}$ of size $|M|=0.8|\mathrm{V}|$ (80\% missing normals), set masked normals to zero, and predict normals at masked vertices using kernel-based field integration on $\mathrm{G}$. Using meshes from the Thingi10k dataset \cite{thingi10k_dataset}, we compare our MRF approach against the baseline method of explicitly forming the full heat kernel matrix via spectral computations (see Appendix \ref{app:vertex_normal_prediction_setup} for more details). 
% speedup in interpolation time, defined as $\text{speedup} = t_{\text{baseline}}/t_{\text{MRFs}}$.
%     The spectral baseline becomes rapidly expensive due to $O(|\tilde V|^3)$ preprocessing and $O(|\tilde V|^2)$ application/memory,
%     whereas MRF inference remains substantially faster once features are precomputed.

% \ap{added ambient RF stuff}
Fig.~\ref{fig:vertex_normal_prediction_results} shows that the full-kernel construction cost grows rapidly with \(|\mathrm{V}|\), whereas MRF preprocessing remains comparatively stable because the walk and training budgets are fixed. 
In our largest completed run \((|\mathrm{V}|=47{,}024)\), full-kernel construction took \(2936\) seconds, while MRF data generation and training took \(171\) seconds; full-kernel interpolation took \(0.42\) seconds, while MRF interpolation took \(0.45\) seconds. This preprocessing gap widens for larger meshes because explicitly storing \(K\in\mathbb{R}^{|\mathrm{V}|\times|\mathrm{V}|}\) requires \(O(|\mathrm{V}|^2)\) memory, making full kernel formation impractical beyond \(10^5\) vertices. At the same time, Table~\ref{tab:vertex_velocity_random_feature_comparison} shows that MRFs achieve the highest average cosine similarity and the best interpolation-time scaling among RF methods. 
Thus, MRFs retain the accuracy of intrinsic heat-kernel interpolation while scaling more favorably than explicit full-kernel construction and outperforming ambient Euclidean RF features.
\begin{table}[t]
\centering
\caption{Random-feature comparison for vertex normals and velocity prediction. Values are averaged across mesh size $|V|$. Growth exponents $\alpha$ are fitted from interpolation time $t_{\mathrm{interp}}(N) \propto N^{\alpha}$. For each task and each fixed feature budget $m$, RFF/ORF/PRF/PORF use the best tuned $\sigma$.}
\resizebox{0.85\textwidth}{!}{%%
\begin{tabular}{llccccccccc}
\toprule
 & & & \multicolumn{4}{c}{$m=256$} & \multicolumn{4}{c}{$m=1024$} \\
\cmidrule(lr){4-7}\cmidrule(lr){8-11}
\textbf{Application} & \textbf{Metric} & \textbf{MRF} & \textbf{RFF} & \textbf{ORF} & \textbf{PRF} & \textbf{PORF} & \textbf{RFF} & \textbf{ORF} & \textbf{PRF} & \textbf{PORF} \\
\midrule
\multirow{2}{*}{\textbf{Vertex norm. prediction}} & cosine similarity $\uparrow$ & $\textbf{0.85}$ & $0.71$ & $0.72$ & $0.40$ & $0.41$ & $0.78$ & $0.78$ & $0.47$ & $0.43$ \\
% Vertex normal interp. time (ms) $\downarrow$ & $\textbf{-}$ & $3.6$ & $3.3$ & $4.2$ & $3.3$ & $11$ & $10$ & $13$ & $12$ \\
& growth $\alpha$ $\downarrow$ & $\textbf{0.50}$ & $1.06$ & $1.01$ & $1.09$ & $0.97$ & $1.10$ & $1.07$ & $1.13$ & $1.14$ \\
\midrule
\multirow{2}{*}{\textbf{Velocity prediction}} & cosine similarity $\uparrow$ & $\textbf{0.93}$ & $0.65$ & $0.66$ & $0.68$ & $0.69$ & $0.66$ & $0.79$ & $0.70$ & $0.70$ \\
% Velocity interp. time (ms) $\downarrow$ & $\textbf{4.5}$ & $10$ & $10$ & $10$ & $11$ & $37$ & $37$ & $38$ & $38$ \\
& growth $\alpha$ $\downarrow$ & $\textbf{0.71}$ & $1.18$ & $1.20$ & $1.19$ & $1.19$ & $1.03$ & $1.02$ & $1.04$ & $1.03$ \\
\bottomrule
\end{tabular}
}
\label{tab:vertex_velocity_random_feature_comparison}
\end{table}
\begin{figure}[t]
    \centering
    \includegraphics[width=0.24\linewidth]{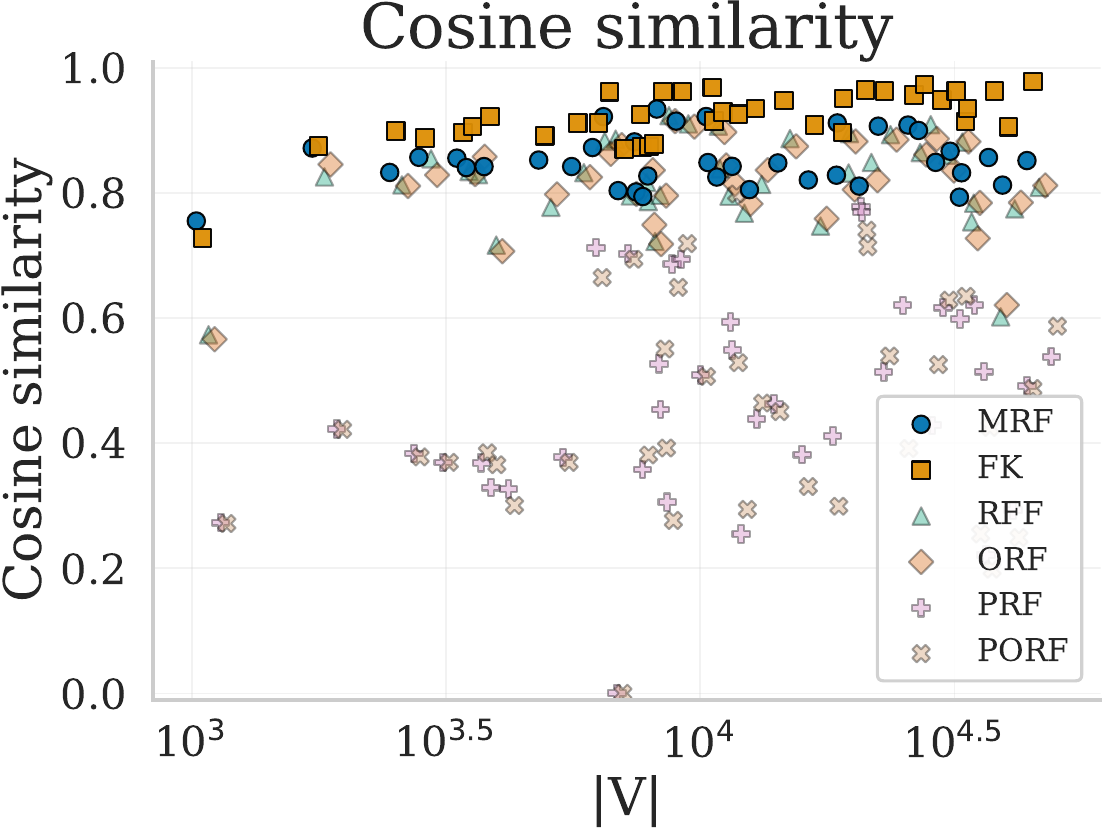}
    \includegraphics[width=0.24\linewidth]{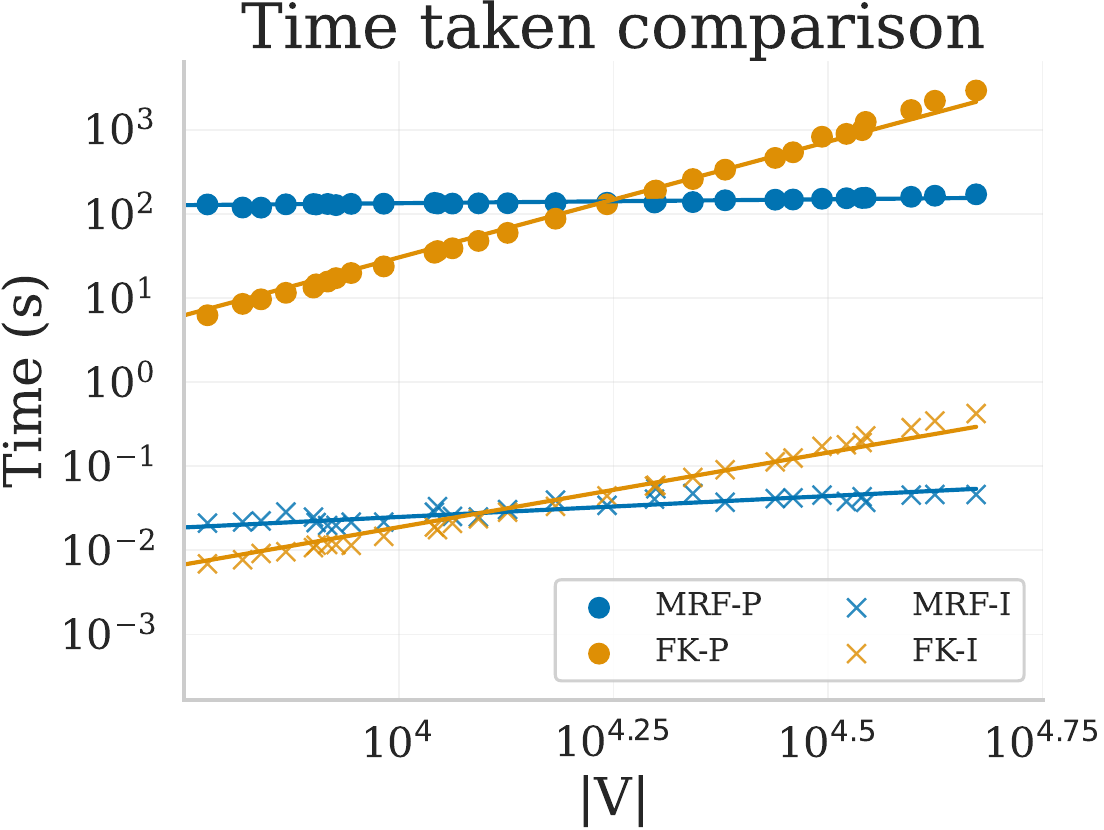}
    \includegraphics[width=0.24\linewidth]{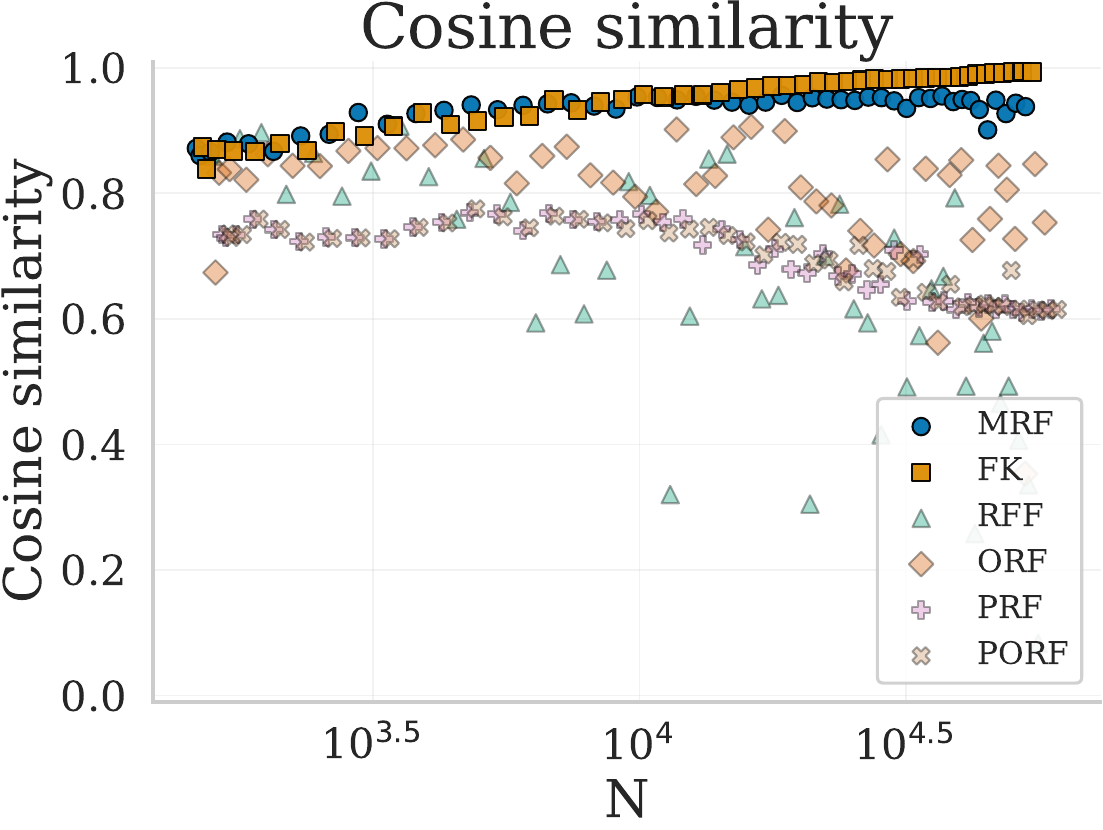}
    \includegraphics[width=0.24\linewidth]{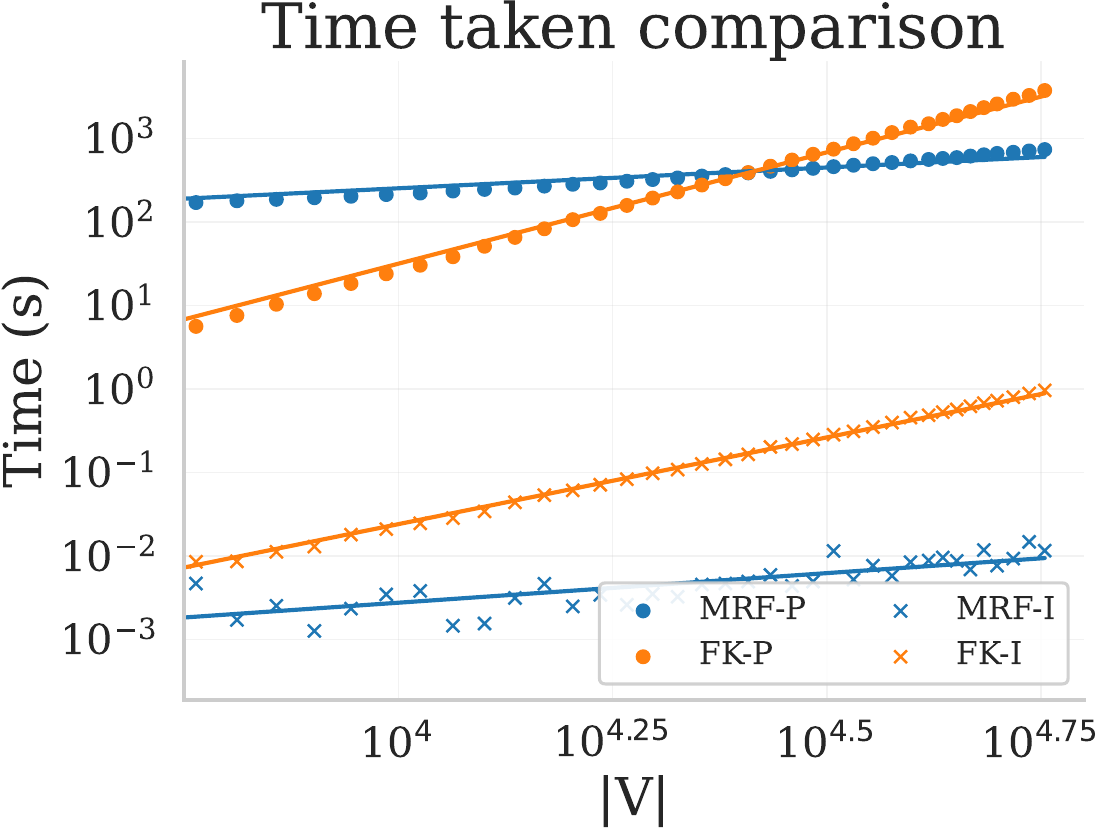}
    \caption{{Comparison across methods for vertex normal prediction on \texttt{Thingi10K} (left) and velocity prediction on \texttt{flag\_simple} (right); preprocessing and interpolation times denoted by `-P' and `-I'.}}
\label{fig:vertex_normal_prediction_results}
\vspace{-3mm}
\end{figure}

\textbf{Velocity prediction:} We also evaluate interpolation on the \texttt{flag\_simple} dataset \cite{flag_simple_dataset} where the signal of interest is the per-node velocity field. For a given mesh with node set $\mathrm{V}$ and node positions $\{\mathbf{x}_i\in\mathbb{R}^3\}_{i\in \mathrm{V}}$, let $\{\mathbf{u}_i\in\mathbb{R}^3\}_{i\in \mathrm{V}}$ denote the ground truth velocity vectors at a fixed time step. We randomly mask a small subset $M\subseteq \mathrm{V}$ with $|M|=0.05|\mathrm{V}|$ (5\% missing velocities), set $u_i=0$ for $i\in M$, and predict velocities at masked nodes using heat-kernel field integration (setup as in \cite{pmlr-v202-choromanski23b}). The \texttt{flag\_simple} meshes have only $\approx$1.5k vertices, which can be too small to expose inference-time scaling trends. To probe larger problem sizes, we construct a denser node set $\mathrm{\widetilde V}$ by sampling additional points on triangle faces, using area-weighted barycentric sampling. The mesh surface is represented by a set of triangular faces.
Concretely, for a target size $|\mathrm{\widetilde V}|=n_{\text{dense}}$, we draw $n_{\text{dense}}-|\mathrm{V}|$ faces i.i.d. (with probability proportional to triangle area). For each sampled face, we draw barycentric weights $(\beta_0,\beta_1,\beta_2)$ (uniform on the $2$-simplex) and define the sampled point's position as a convex combination of the triangle's three incident vertex positions.
At each time step, we assign a velocity to each sampled point by applying the same barycentric weights to the velocities of the corresponding face vertices, yielding a dense velocity field $\mathbf{U}\in\mathbb{R}^{|\widetilde V|\times 3}$.

We then build a sparse graph on $\mathrm{\widetilde V}$ (using the reference geometry from the first evaluated frame) via a $k$NN construction in $\mathbb{R}^3$
and define a symmetric, normalized affinity matrix $\mathbf{W}_f := \mathbf{D}^{-1/2}\mathbf{WD}^{-1/2}$ where $\mathbf{W}$ has Gaussian weights
$W_{ij}=\exp(-\|\mathbf{x}_i-\mathbf{x}_j\|_2^2/\sigma^2)$ and $\mathbf{D}$ is the degree matrix of $\mathbf{W}$.
The corresponding heat kernel is \(\mathbf{K} = \exp(\tau \mathbf{W}_f)\,.\)
% \ap{dunno if need more than this for RF vs ambient, I added this sentence: }
Table~\ref{tab:vertex_velocity_random_feature_comparison} shows that MRFs again achieve the best average cosine similarity and the best scaling exponent among RF baselines. Figure~\ref{fig:vertex_normal_prediction_results} further shows that the dense spectral baseline exhibits the expected steep growth in preprocessing time, while MRF preprocessing remains comparatively stable across the tested range since training/feature budgets are fixed per discretization. At inference time, MRFs consistently outperform the baseline, with speedups that exceed $\mathbf{10x}$ as $n_{\text{dense}}$ increases.
   
\subsection{Attention on Manifold-Valued Tokens and Manifolds Beyond 2D Surfaces}
\label{sec:attention}
\textbf{Linear-attention:} A key motivation for random features is linearizing attention, replacing a nonlinear query-key similarity with an inner product of feature maps. Standard RF-based linear attention assumes Euclidean geometry. To test whether MRFs provide a useful geometry-aware alternative, we isolate the geometric kernel component in a masked reconstruction task on manifold-valued tokens sampled from a M\"obius strip. The reference operator is an intrinsic graph-geodesic heat kernel, and each method replaces this operator by either an exact kernel approximation or a low-rank feature approximation. Figure~\ref{fig:attention-and-brodatz-experiment-results} shows that MRFs obtain the best MSE for every source count and the best relative \(\ell_2\) error on five of six source counts over other RF methods and the exact Gaussian kernel in ambient Euclidean space (ambient RBF); full results with additional baselines are in Appendix~\ref{app:attention}, Table~\ref{tab:attention-full}. These results suggest MRFs can serve as geometry-aware random features for linearized attention when token embeddings lie on, or are constrained by, a non-Euclidean geometry.

\begin{figure}[t]
    \centering
    \includegraphics[width=0.49\linewidth]{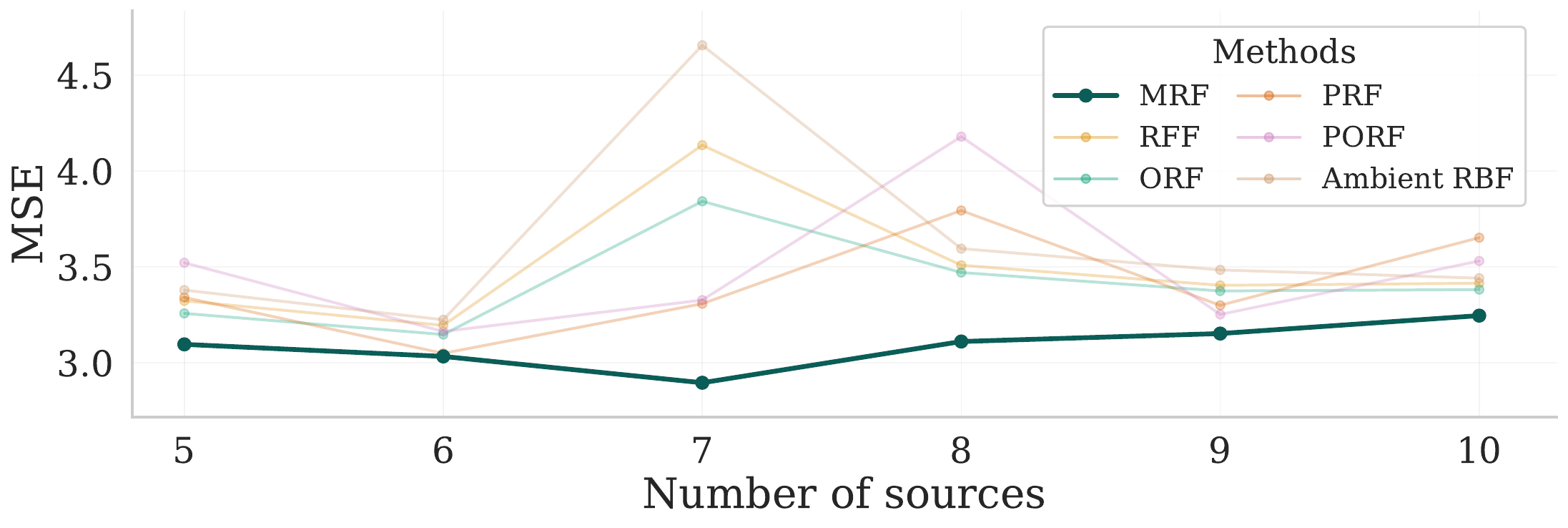}
    \includegraphics[width=0.49\linewidth]{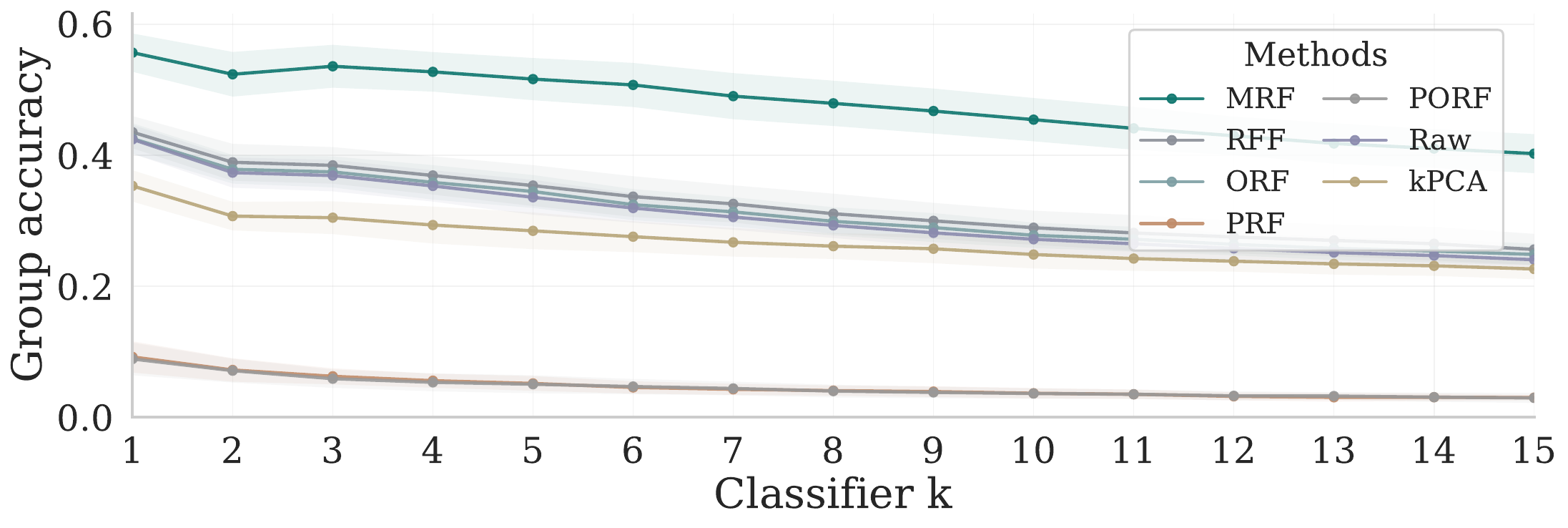}
    \caption{\textbf{Left}: Masked reconstruction on M\"obius-strip tokens. \textbf{Right}: Performance on Brodatz dataset across $k$ and over multiple seeds.}
    \label{fig:attention-and-brodatz-experiment-results}
    \vspace{-3mm}
\end{figure}

\textbf{Higher-dimensional descriptor manifolds:}
To test whether MRFs remain useful beyond embedded two-dimensional surfaces, we evaluate them on a Brodatz transformed-view classification task using 15-dimensional log-SPD covariance descriptors \citep{brodatz1966textures,Varma03}. Following the procedure in the paper, the graph is built on an unlabeled descriptor cloud generated from transformed texture windows, GRF signatures are computed on this graph, and the learned MRF features are used for $k$NN classification with group-vote evaluation. We compare against using raw log-SPD descriptors (Raw), a Gaussian-kernel PCA embedding ($k$PCA), and other RF methods. We report group accuracy, where predictions from 50 windows belonging to the same texture-angle block are aggregated by majority vote. As shown in Fig.~\ref{fig:attention-and-brodatz-experiment-results} (and Table~\ref{tab:brodatz-transform-group}, Appendix~\ref{app:brodatz}, where there are further experiment details), MRFs achieve the best-tuned accuracy, with mean accuracy (over multiple seeds) of $55.7\%$, compared with the second highest accuracy of $43.5\%$ achieved by RFF.

\textbf{Non-compact manifolds:} Compactness is natural in many practical use cases---in applied settings, the manifolds of interest (molecular surfaces, mesh surfaces, etc.) are often compact---and the present theory and main experiments focus on compact manifolds. This assumption is used to obtain finite discretizations, bounded learned features, and clean heat-kernel convergence statements. While out of the scope of this work, we provide some insight on how MRFs can be applied to non-compact manifolds (see Appendix~\ref{app:non-compact-manifolds} for details). In particular, we perform a preliminary experiment on hyperbolic 3-space $\mathbb{H}^3$, a canonical non-compact manifold with exponential volume growth, and show that MRFs extend to non-compact manifolds in a \textit{local-on-compacts} sense: when evaluation is restricted to a fixed bounded region, the MRF approximation improves as the finite support used to construct the features is enlarged. This does not constitute a global uniform guarantee on $\mathbb{H}^3$, but suggests a natural exhaustion-based route for extending MRFs beyond the compact setting. Developing the corresponding theory, and testing it across broader classes of non-compact manifolds, is an interesting direction for future work.

% That being said, compactness is natural in many practical use cases—in applied settings, the manifolds of interest (molecular surfaces, mesh surfaces, etc.) are often compact—further work would provide full theoretical extension or a uniform global guarantee, in addition to preliminary evidence that the framework is compatible with standard exhaustion techniques used in analysis on non-compact manifolds.

% \dl{SUMMARIZE RESULTS AND MOVE REST OF DETAIL TO APPENDIX:} Non-compact manifolds require additional care because a finite graph can only approximate bounded regions. Rather than attempting a global discretization, we adopt an exhaustion-based (local-on-compacts) approach: we approximate the heat kernel using finite constructions on balls $B_R(o)$, and evaluate accuracy on a fixed inner region $B_r(o)$. Empirically, we observe that approximation error on $B_r(o)$ decreases as outer radii $R$ increases. We evaluate the approximation on a fixed inner ball with $r=1.5$, using 128 query nodes sampled from that inner region and comparing against the analytic heat kernel at time $t=0.5$. GRF and MRF walk budgets are increased with $R$.
\vspace{-1mm}
\section{Conclusion}
\vspace{-1mm}
\label{sec:conclusion}

We have presented a new paradigm of \textit{Manifold Random Features} (MRFs), designed for scalable computations involving bi-variate functions operating on manifolds. MRFs are built with the use of Graph Random Features (GRFs), operating on the discretized manifolds. GRFs are used as teachers for the supervised learning of the continuous fields encoding the MRF-mechanisms. We provide strong empirical validation of MRFs, complemented with rigorous mathematical analysis.
As a byproduct of our analysis, we show deep connection between combinatorial concepts of random walks on graphs representing discretized manifolds and continuous kernel functions on these manifolds, in particular by re-discovering recently proposed new class of estimators approximating Gaussian kernels with positive and bounded random features.

\vspace{-1mm}
\section*{Impact Statement \& Limitations}
\vspace{-1mm}
This paper presents work whose goal is to advance the field of Machine Learning. There are no immediate societal impacts of our work, as we focus on a methodological framework for approximating functions on manifolds. We do believe that this work might lead to the discovery of new mechanisms of random features operating in continuous spaces and encoded as neural networks that are trained with supervision provided by Graph Random Features teachers. Extending MRFs beyond compact manifolds is left to future work as noted above at the end of Sec.~\ref{sec:attention}.

%%%%%%%%%%%%%%%%%%%%%%%%%%%%%%%%%%%%%%%%%%%%%%%%%%%%%%%%%%%%%%%%%%%%%%%%%%%%%%%
% References
%%%%%%%%%%%%%%%%%%%%%%%%%%%%%%%%%%%%%%%%%%%%%%%%%%%%%%%%%%%%%%%%%%%%%%%%%%%%%%%
\bibliography{references}

@inproceedings{rf-1,
  author       = {Ali Rahimi and
                  Benjamin Recht},
  editor       = {John C. Platt and
                  Daphne Koller and
                  Yoram Singer and
                  Sam T. Roweis},
  title        = {Random Features for Large-Scale Kernel Machines},
  booktitle    = {Advances in Neural Information Processing Systems 20, Proceedings
                  of the Twenty-First Annual Conference on Neural Information Processing
                  Systems, Vancouver, British Columbia, Canada, December 3-6, 2007},
  pages        = {1177--1184},
  publisher    = {Curran Associates, Inc.},
  year         = {2007},
  bibsource    = {dblp computer science bibliography, https://dblp.org}
}

@unpublished{laplace-beltrami,
	title= {Laplace-Beltrami: The Swiss Army Knife of Geometry Processing},
	author = {Justin Solomon and Keenan Crane and Etienne Vouga},
	year = {2014},
	note= {Symposium on Geometry Processing},
	URL= {http://ddg.cs.columbia.edu/SGP2014/LaplaceBeltrami.pdf},
}

@article{jlt-main,
  author       = {Sanjoy Dasgupta and
                  Anupam Gupta},
  title        = {An elementary proof of a theorem of Johnson and Lindenstrauss},
  journal      = {Random Struct. Algorithms},
  volume       = {22},
  number       = {1},
  pages        = {60--65},
  year         = {2003},
  url          = {https://doi.org/10.1002/rsa.10073},
  doi          = {10.1002/RSA.10073},
  bibsource    = {dblp computer science bibliography, https://dblp.org}
}

@inproceedings{rf-2,
  author       = {Ali Rahimi and
                  Benjamin Recht},
  editor       = {Daphne Koller and
                  Dale Schuurmans and
                  Yoshua Bengio and
                  L{\'{e}}on Bottou},
  title        = {Weighted Sums of Random Kitchen Sinks: Replacing minimization with
                  randomization in learning},
  booktitle    = {Advances in Neural Information Processing Systems 21, Proceedings
                  of the Twenty-Second Annual Conference on Neural Information Processing
                  Systems, Vancouver, British Columbia, Canada, December 8-11, 2008},
  pages        = {1313--1320},
  publisher    = {Curran Associates, Inc.},
  year         = {2008},
  bibsource    = {dblp computer science bibliography, https://dblp.org}
}

@inproceedings{rf-3,
  author       = {Quoc V. Le and
                  Tam{\'{a}}s Sarl{\'{o}}s and
                  Alexander J. Smola},
  title        = {Fastfood - Computing Hilbert Space Expansions in loglinear time},
  booktitle    = {Proceedings of the 30th International Conference on Machine Learning,
                  {ICML} 2013, Atlanta, GA, USA, 16-21 June 2013},
  series       = {{JMLR} Workshop and Conference Proceedings},
  volume       = {28},
  pages        = {244--252},
  publisher    = {JMLR.org},
  year         = {2013},
  url          = {http://proceedings.mlr.press/v28/le13.html},
  bibsource    = {dblp computer science bibliography, https://dblp.org}
}

@inproceedings{rf-4,
  author       = {Jiyan Yang and
                  Vikas Sindhwani and
                  Haim Avron and
                  Michael W. Mahoney},
  title        = {Quasi-Monte Carlo Feature Maps for Shift-Invariant Kernels},
  booktitle    = {Proceedings of the 31th International Conference on Machine Learning,
                  {ICML} 2014, Beijing, China, 21-26 June 2014},
  series       = {{JMLR} Workshop and Conference Proceedings},
  volume       = {32},
  pages        = {485--493},
  publisher    = {JMLR.org},
  year         = {2014},
  url          = {http://proceedings.mlr.press/v32/yangb14.html},
  bibsource    = {dblp computer science bibliography, https://dblp.org}
}

@inproceedings{grfs-general,
  author       = {Isaac Reid and
                  Krzysztof Marcin Choromanski and
                  Eli Berger and
                  Adrian Weller},
  title        = {General Graph Random Features},
  booktitle    = {The Twelfth International Conference on Learning Representations,
                  {ICLR} 2024, Vienna, Austria, May 7-11, 2024},
  publisher    = {OpenReview.net},
  year         = {2024},
  url          = {https://openreview.net/forum?id=viftsX50Rt},
  bibsource    = {dblp computer science bibliography, https://dblp.org}
}

@article{grfs-plusplus,
  author       = {Krzysztof Choromanski and
                  Avinava Dubey and
                  Arijit Sehanobish and
                  Isaac Reid},
  title        = {Computationally-efficient Graph Modeling with Refined Graph Random
                  Features},
  journal      = {CoRR},
  volume       = {abs/2510.07716},
  year         = {2025},
  url          = {https://doi.org/10.48550/arXiv.2510.07716},
  doi          = {10.48550/ARXIV.2510.07716},
  eprinttype    = {arXiv},
  eprint       = {2510.07716},
  bibsource    = {dblp computer science bibliography, https://dblp.org}
}

@article{rf-6,
  author       = {Fanghui Liu and
                  Xiaolin Huang and
                  Yudong Chen and
                  Johan A. K. Suykens},
  title        = {Random Features for Kernel Approximation: {A} Survey on Algorithms,
                  Theory, and Beyond},
  journal      = {{IEEE} Trans. Pattern Anal. Mach. Intell.},
  volume       = {44},
  number       = {10},
  pages        = {7128--7148},
  year         = {2022},
  url          = {https://doi.org/10.1109/TPAMI.2021.3097011},
  doi          = {10.1109/TPAMI.2021.3097011},
  bibsource    = {dblp computer science bibliography, https://dblp.org}
}

@article{rf-7,
  author       = {Ninh Pham and
                  Rasmus Pagh},
  title        = {Tensor Sketch: Fast and Scalable Polynomial Kernel Approximation},
  journal      = {CoRR},
  volume       = {abs/2505.08146},
  year         = {2025},
  url          = {https://doi.org/10.48550/arXiv.2505.08146},
  doi          = {10.48550/ARXIV.2505.08146},
  eprinttype    = {arXiv},
  eprint       = {2505.08146},
  bibsource    = {dblp computer science bibliography, https://dblp.org}
}

@inproceedings{gcmcsampling,
  author       = {Mark Rowland and
                  Krzysztof Choromanski and
                  Fran{\c{c}}ois Chalus and
                  Aldo Pacchiano and
                  Tam{\'{a}}s Sarl{\'{o}}s and
                  Richard E. Turner and
                  Adrian Weller},
  editor       = {Samy Bengio and
                  Hanna M. Wallach and
                  Hugo Larochelle and
                  Kristen Grauman and
                  Nicol{\`{o}} Cesa{-}Bianchi and
                  Roman Garnett},
  title        = {Geometrically Coupled Monte Carlo Sampling},
  booktitle    = {Advances in Neural Information Processing Systems 31: Annual Conference
                  on Neural Information Processing Systems 2018, NeurIPS 2018, December
                  3-8, 2018, Montr{\'{e}}al, Canada},
  pages        = {195--205},
  year         = {2018},
  bibsource    = {dblp computer science bibliography, https://dblp.org}
}

@inproceedings{slim-performers,
  author       = {Valerii Likhosherstov and
                  Krzysztof Marcin Choromanski and
                  Jared Quincy Davis and
                  Xingyou Song and
                  Adrian Weller},
  editor       = {Marc'Aurelio Ranzato and
                  Alina Beygelzimer and
                  Yann N. Dauphin and
                  Percy Liang and
                  Jennifer Wortman Vaughan},
  title        = {Sub-Linear Memory: How to Make Performers SLiM},
  booktitle    = {Advances in Neural Information Processing Systems 34: Annual Conference
                  on Neural Information Processing Systems 2021, NeurIPS 2021, December
                  6-14, 2021, virtual},
  pages        = {6707--6719},
  year         = {2021},
  bibsource    = {dblp computer science bibliography, https://dblp.org}
}

@inproceedings{rrws,
  author       = {Isaac Reid and
                  Eli Berger and
                  Krzysztof Marcin Choromanski and
                  Adrian Weller},
  title        = {Repelling Random Walks},
  booktitle    = {The Twelfth International Conference on Learning Representations,
                  {ICLR} 2024, Vienna, Austria, May 7-11, 2024},
  publisher    = {OpenReview.net},
  year         = {2024},
  url          = {https://openreview.net/forum?id=31IOmrnoP4},
  bibsource    = {dblp computer science bibliography, https://dblp.org}
}

@inproceedings{hedgehog,
  author       = {Michael Zhang and
                  Kush Bhatia and
                  Hermann Kumbong and
                  Christopher R{\'{e}}},
  title        = {The Hedgehog {\&} the Porcupine: Expressive Linear Attentions
                  with Softmax Mimicry},
  booktitle    = {The Twelfth International Conference on Learning Representations,
                  {ICLR} 2024, Vienna, Austria, May 7-11, 2024},
  publisher    = {OpenReview.net},
  year         = {2024},
  url          = {https://openreview.net/forum?id=4g02l2N2Nx},
  bibsource    = {dblp computer science bibliography, https://dblp.org}
}

@inproceedings{uortmc,
  author       = {Krzysztof Choromanski and
                  Mark Rowland and
                  Wenyu Chen and
                  Adrian Weller},
  editor       = {Kamalika Chaudhuri and
                  Ruslan Salakhutdinov},
  title        = {Unifying Orthogonal Monte Carlo Methods},
  booktitle    = {Proceedings of the 36th International Conference on Machine Learning,
                  {ICML} 2019, 9-15 June 2019, Long Beach, California, {USA}},
  series       = {Proceedings of Machine Learning Research},
  volume       = {97},
  pages        = {1203--1212},
  publisher    = {{PMLR}},
  year         = {2019},
  url          = {http://proceedings.mlr.press/v97/choromanski19a.html},
  bibsource    = {dblp computer science bibliography, https://dblp.org}
}

@inproceedings{hrfs,
  author       = {Krzysztof Marcin Choromanski and
                  Han Lin and
                  Haoxian Chen and
                  Arijit Sehanobish and
                  Yuanzhe Ma and
                  Deepali Jain and
                  Jake Varley and
                  Andy Zeng and
                  Michael S. Ryoo and
                  Valerii Likhosherstov and
                  Dmitry Kalashnikov and
                  Vikas Sindhwani and
                  Adrian Weller},
  title        = {Hybrid Random Features},
  booktitle    = {The Tenth International Conference on Learning Representations, {ICLR}
                  2022, Virtual Event, April 25-29, 2022},
  publisher    = {OpenReview.net},
  year         = {2022},
  url          = {https://openreview.net/forum?id=EMigfE6ZeS},
  bibsource    = {dblp computer science bibliography, https://dblp.org}
}

@article{unified-rf-attention,
  author       = {Duke Nguyen and
                  Aditya Joshi and
                  Flora D. Salim},
  title        = {Spectraformer: {A} Unified Random Feature Framework for Transformer},
  journal      = {CoRR},
  volume       = {abs/2405.15310},
  year         = {2024},
  url          = {https://doi.org/10.48550/arXiv.2405.15310},
  doi          = {10.48550/ARXIV.2405.15310},
  eprinttype    = {arXiv},
  eprint       = {2405.15310},
  bibsource    = {dblp computer science bibliography, https://dblp.org}
}

@misc{kim2024magnituder,
      title={Magnituder Layers for Implicit Neural Representations in 3D}, 
      author={Sang Min Kim and Byeongchan Kim and Arijit Sehanobish and Krzysztof Choromanski and Dongseok Shim and Avinava Dubey and Min-hwan Oh},
      year={2024},
      eprint={2410.09771},
      archivePrefix={arXiv},
      primaryClass={cs.CV},
      url={https://arxiv.org/abs/2410.09771}, 
}

@inproceedings{rf-attention,
  author       = {Hao Peng and
                  Nikolaos Pappas and
                  Dani Yogatama and
                  Roy Schwartz and
                  Noah A. Smith and
                  Lingpeng Kong},
  title        = {Random Feature Attention},
  booktitle    = {9th International Conference on Learning Representations, {ICLR} 2021,
                  Virtual Event, Austria, May 3-7, 2021},
  publisher    = {OpenReview.net},
  year         = {2021},
  url          = {https://openreview.net/forum?id=QtTKTdVrFBB},
  bibsource    = {dblp computer science bibliography, https://dblp.org}
}

@inproceedings{rfsvm,
  author       = {Yitong Sun and
                  Anna C. Gilbert and
                  Ambuj Tewari},
  editor       = {Samy Bengio and
                  Hanna M. Wallach and
                  Hugo Larochelle and
                  Kristen Grauman and
                  Nicol{\`{o}} Cesa{-}Bianchi and
                  Roman Garnett},
  title        = {But How Does It Work in Theory? Linear {SVM} with Random Features},
  booktitle    = {Advances in Neural Information Processing Systems 31: Annual Conference
                  on Neural Information Processing Systems 2018, NeurIPS 2018, December
                  3-8, 2018, Montr{\'{e}}al, Canada},
  pages        = {3383--3392},
  year         = {2018},
  bibsource    = {dblp computer science bibliography, https://dblp.org}
}

@inproceedings{performers,
  author       = {Krzysztof Marcin Choromanski and
                  Valerii Likhosherstov and
                  David Dohan and
                  Xingyou Song and
                  Andreea Gane and
                  Tam{\'{a}}s Sarl{\'{o}}s and
                  Peter Hawkins and
                  Jared Quincy Davis and
                  Afroz Mohiuddin and
                  Lukasz Kaiser and
                  David Benjamin Belanger and
                  Lucy J. Colwell and
                  Adrian Weller},
  title        = {Rethinking Attention with Performers},
  booktitle    = {9th International Conference on Learning Representations, {ICLR} 2021,
                  Virtual Event, Austria, May 3-7, 2021},
  publisher    = {OpenReview.net},
  year         = {2021},
  url          = {https://openreview.net/forum?id=Ua6zuk0WRH},
  bibsource    = {dblp computer science bibliography, https://dblp.org}
}

@article{linear-attention,
  author       = {Yutao Sun and
                  Zhenyu Li and
                  Yike Zhang and
                  Tengyu Pan and
                  Bowen Dong and
                  Yuyi Guo and
                  Jianyong Wang},
  title        = {Efficient Attention Mechanisms for Large Language Models: {A} Survey},
  journal      = {CoRR},
  volume       = {abs/2507.19595},
  year         = {2025},
  url          = {https://doi.org/10.48550/arXiv.2507.19595},
  doi          = {10.48550/ARXIV.2507.19595},
  eprinttype    = {arXiv},
  eprint       = {2507.19595},
  bibsource    = {dblp computer science bibliography, https://dblp.org}
}

@inproceedings{recycling,
  author       = {Krzysztof Choromanski and
                  Vikas Sindhwani},
  editor       = {Maria{-}Florina Balcan and
                  Kilian Q. Weinberger},
  title        = {Recycling Randomness with Structure for Sublinear time Kernel Expansions},
  booktitle    = {Proceedings of the 33nd International Conference on Machine Learning,
                  {ICML} 2016, New York City, NY, USA, June 19-24, 2016},
  series       = {{JMLR} Workshop and Conference Proceedings},
  volume       = {48},
  pages        = {2502--2510},
  publisher    = {JMLR.org},
  year         = {2016},
  url          = {http://proceedings.mlr.press/v48/choromanski16.html},
  bibsource    = {dblp computer science bibliography, https://dblp.org}
}

@inproceedings{SNNK,
  author       = {Arijit Sehanobish and
                  Krzysztof Marcin Choromanski and
                  Yunfan Zhao and
                  Kumar Avinava Dubey and
                  Valerii Likhosherstov},
  title        = {Scalable Neural Network Kernels},
  booktitle    = {The Twelfth International Conference on Learning Representations,
                  {ICLR} 2024, Vienna, Austria, May 7-11, 2024},
  publisher    = {OpenReview.net},
  year         = {2024},
  url          = {https://openreview.net/forum?id=4iPw1klFWa},
  bibsource    = {dblp computer science bibliography, https://dblp.org}
}

@inproceedings{timit,
  author       = {Po{-}Sen Huang and
                  Haim Avron and
                  Tara N. Sainath and
                  Vikas Sindhwani and
                  Bhuvana Ramabhadran},
  title        = {Kernel methods match Deep Neural Networks on {TIMIT}},
  booktitle    = {{IEEE} International Conference on Acoustics, Speech and Signal Processing,
                  {ICASSP} 2014, Florence, Italy, May 4-9, 2014},
  pages        = {205--209},
  publisher    = {{IEEE}},
  year         = {2014},
  url          = {https://doi.org/10.1109/ICASSP.2014.6853587},
  doi          = {10.1109/ICASSP.2014.6853587},
  bibsource    = {dblp computer science bibliography, https://dblp.org}
}

@inproceedings{ortrfs,
  author       = {Felix X. Yu and
                  Ananda Theertha Suresh and
                  Krzysztof Marcin Choromanski and
                  Daniel N. Holtmann{-}Rice and
                  Sanjiv Kumar},
  editor       = {Daniel D. Lee and
                  Masashi Sugiyama and
                  Ulrike von Luxburg and
                  Isabelle Guyon and
                  Roman Garnett},
  title        = {Orthogonal Random Features},
  booktitle    = {Advances in Neural Information Processing Systems 29: Annual Conference
                  on Neural Information Processing Systems 2016, December 5-10, 2016,
                  Barcelona, Spain},
  pages        = {1975--1983},
  year         = {2016},
  bibsource    = {dblp computer science bibliography, https://dblp.org}
}

@inproceedings{unreas,
  author       = {Krzysztof Marcin Choromanski and
                  Mark Rowland and
                  Adrian Weller},
  editor       = {Isabelle Guyon and
                  Ulrike von Luxburg and
                  Samy Bengio and
                  Hanna M. Wallach and
                  Rob Fergus and
                  S. V. N. Vishwanathan and
                  Roman Garnett},
  title        = {The Unreasonable Effectiveness of Structured Random Orthogonal Embeddings},
  booktitle    = {Advances in Neural Information Processing Systems 30: Annual Conference
                  on Neural Information Processing Systems 2017, December 4-9, 2017,
                  Long Beach, CA, {USA}},
  pages        = {219--228},
  year         = {2017},
  bibsource    = {dblp computer science bibliography, https://dblp.org}
}

@article{grfs-gps,
  author       = {Matthew Zhang and
                  Jihao Andreas Lin and
                  Krzysztof Choromanski and
                  Adrian Weller and
                  Richard E. Turner and
                  Isaac Reid},
  title        = {Graph Random Features for Scalable Gaussian Processes},
  journal      = {CoRR},
  volume       = {abs/2509.03691},
  year         = {2025},
  url          = {https://doi.org/10.48550/arXiv.2509.03691},
  doi          = {10.48550/ARXIV.2509.03691},
  eprinttype    = {arXiv},
  eprint       = {2509.03691},
  bibsource    = {dblp computer science bibliography, https://dblp.org}
}

@inproceedings{de-rfs,
  author       = {Valerii Likhosherstov and
                  Krzysztof Marcin Choromanski and
                  Kumar Avinava Dubey and
                  Frederick Liu and
                  Tam{\'{a}}s Sarl{\'{o}}s and
                  Adrian Weller},
  editor       = {Alice Oh and
                  Tristan Naumann and
                  Amir Globerson and
                  Kate Saenko and
                  Moritz Hardt and
                  Sergey Levine},
  title        = {Dense-Exponential Random Features: Sharp Positive Estimators of the
                  Gaussian Kernel},
  booktitle    = {Advances in Neural Information Processing Systems 36: Annual Conference
                  on Neural Information Processing Systems 2023, NeurIPS 2023, New Orleans,
                  LA, USA, December 10 - 16, 2023},
  year         = {2023},
  bibsource    = {dblp computer science bibliography, https://dblp.org}
}

@inproceedings{grfs-1,
  author       = {Krzysztof Marcin Choromanski},
  editor       = {Andreas Krause and
                  Emma Brunskill and
                  Kyunghyun Cho and
                  Barbara Engelhardt and
                  Sivan Sabato and
                  Jonathan Scarlett},
  title        = {Taming graph kernels with random features},
  booktitle    = {International Conference on Machine Learning, {ICML} 2023, 23-29 July
                  2023, Honolulu, Hawaii, {USA}},
  series       = {Proceedings of Machine Learning Research},
  volume       = {202},
  pages        = {5964--5977},
  publisher    = {{PMLR}},
  year         = {2023},
  url          = {https://proceedings.mlr.press/v202/choromanski23a.html},
  bibsource    = {dblp computer science bibliography, https://dblp.org}
}

@inproceedings{chefs,
  author       = {Valerii Likhosherstov and
                  Krzysztof Marcin Choromanski and
                  Kumar Avinava Dubey and
                  Frederick Liu and
                  Tam{\'{a}}s Sarl{\'{o}}s and
                  Adrian Weller},
  editor       = {Sanmi Koyejo and
                  S. Mohamed and
                  A. Agarwal and
                  Danielle Belgrave and
                  K. Cho and
                  A. Oh},
  title        = {Chefs' Random Tables: Non-Trigonometric Random Features},
  booktitle    = {Advances in Neural Information Processing Systems 35: Annual Conference
                  on Neural Information Processing Systems 2022, NeurIPS 2022, New Orleans,
                  LA, USA, November 28 - December 9, 2022},
  year         = {2022},
  bibsource    = {dblp computer science bibliography, https://dblp.org}
}

@book{grigoryan2009heat,
  title     = {Heat Kernel and Analysis on Manifolds},
  author    = {Grigor'yan, Alexander},
  series    = {AMS/IP Studies in Advanced Mathematics},
  volume    = {47},
  year      = {2009},
  publisher = {American Mathematical Society},
}

@article{Zhao_2018,
   title={Exact Heat Kernel on a Hypersphere and Its Applications in Kernel SVM},
   volume={4},
   ISSN={2297-4687},
   url={http://dx.doi.org/10.3389/fams.2018.00001},
   DOI={10.3389/fams.2018.00001},
   journal={Frontiers in Applied Mathematics and Statistics},
   publisher={Frontiers Media SA},
   author={Zhao, Chenchao and Song, Jun S.},
   year={2018},
   month=jan }

@article{thingi10k_dataset,
  author       = {Qingnan Zhou and
                  Alec Jacobson},
  title        = {Thingi10K: {A} Dataset of 10, 000 3D-Printing Models},
  journal      = {CoRR},
  volume       = {abs/1605.04797},
  year         = {2016},
  url          = {http://arxiv.org/abs/1605.04797},
  eprinttype    = {arXiv},
  eprint       = {1605.04797},
  bibsource    = {dblp computer science bibliography, https://dblp.org}
}

@article{flag_simple_dataset,
  author       = {Tobias Pfaff and
                  Meire Fortunato and
                  Alvaro Sanchez{-}Gonzalez and
                  Peter W. Battaglia},
  title        = {Learning Mesh-Based Simulation with Graph Networks},
  journal      = {CoRR},
  volume       = {abs/2010.03409},
  year         = {2020},
  url          = {https://arxiv.org/abs/2010.03409},
  eprinttype    = {arXiv},
  eprint       = {2010.03409},
  bibsource    = {dblp computer science bibliography, https://dblp.org}
}

@InProceedings{pmlr-v202-choromanski23b,
  title = 	 {Efficient Graph Field Integrators Meet Point Clouds},
  author =       {Choromanski, Krzysztof Marcin and Sehanobish, Arijit and Lin, Han and Zhao, Yunfan and Berger, Eli and Parshakova, Tetiana and Pan, Alvin and Watkins, David and Zhang, Tianyi and Likhosherstov, Valerii and Basu Roy Chowdhury, Somnath and Dubey, Kumar Avinava and Jain, Deepali and Sarlos, Tamas and Chaturvedi, Snigdha and Weller, Adrian},
  booktitle = 	 {Proceedings of the 40th International Conference on Machine Learning},
  pages = 	 {5978--6004},
  year = 	 {2023},
  editor = 	 {Krause, Andreas and Brunskill, Emma and Cho, Kyunghyun and Engelhardt, Barbara and Sabato, Sivan and Scarlett, Jonathan},
  volume = 	 {202},
  series = 	 {Proceedings of Machine Learning Research},
  month = 	 {23--29 Jul},
  publisher =    {PMLR},
  url = 	 {https://proceedings.mlr.press/v202/choromanski23b.html}
}

@misc{brandolini2010quadraturerulesdistributionpoints,
      title={Quadrature rules and distribution of points on manifolds}, 
      author={Luca Brandolini and Christine Choirat and Leonardo Colzani and Giacomo Gigante and Raffaello Seri and Giancarlo Travaglini},
      year={2010},
      eprint={1012.5409},
      archivePrefix={arXiv},
      primaryClass={math.NT},
      url={https://arxiv.org/abs/1012.5409}, 
}

@misc{cadavid2022approachevaluatingcertaintrigonometric,
      title={On an approach for evaluating certain trigonometric character sums using the discrete time heat kernel}, 
      author={Carlos A. Cadavid and Paulina Hoyos and Jay Jorgenson and Lejla Smajlović and Juan D. Vélez},
      year={2022},
      eprint={2201.07878},
      archivePrefix={arXiv},
      primaryClass={math.CO},
      url={https://arxiv.org/abs/2201.07878}, 
}

@InProceedings{Varma03,
  author       = "Manik Varma and Andrew Zisserman",
  title        = "Texture Classification: {A}re Filter Banks Necessary?",
  booktitle    = "IEEE Conference on Computer Vision and Pattern Recognition",
  volume       = "2",
  pages        = "691--698",
  year         = "2003",
}

@book{brodatz1966textures,
  author    = {Brodatz, Phil},
  title     = {Textures: A Photographic Album for Artists and Designers},
  publisher = {Dover Publications},
  address   = {New York},
  year      = {1966}
}
\bibliographystyle{plainnat}

%%%%%%%%%%%%%%%%%%%%%%%%%%%%%%%%%%%%%%%%%%%%%%%%%%%%%%%%%%%%%%%%%%%%%%%%%%%%%%%
% Appendix
%%%%%%%%%%%%%%%%%%%%%%%%%%%%%%%%%%%%%%%%%%%%%%%%%%%%%%%%%%%%%%%%%%%%%%%%%%%%%%%
\newpage
\appendix
\section{Appendix}
\label{sec:appendix}

This appendix first reports auxiliary GRF/MRF implementation details and the Gaussian-grid theory supporting the Euclidean limiting case. Then we report supplementary experimental protocols and full results in the same order as the main experimental section for consistency.

\subsection{Additional GRF/MRF preliminaries}
\label{app:preliminaries}

We collect two auxiliary items used throughout the paper: an illustration of manifold heat-kernel fields and the regular GRF random-walk estimator.

\subsubsection{Heat-kernel illustration}

\begin{figure}[h]
    \centering
    \vspace{-3mm}
    \includegraphics[width=0.5\linewidth]{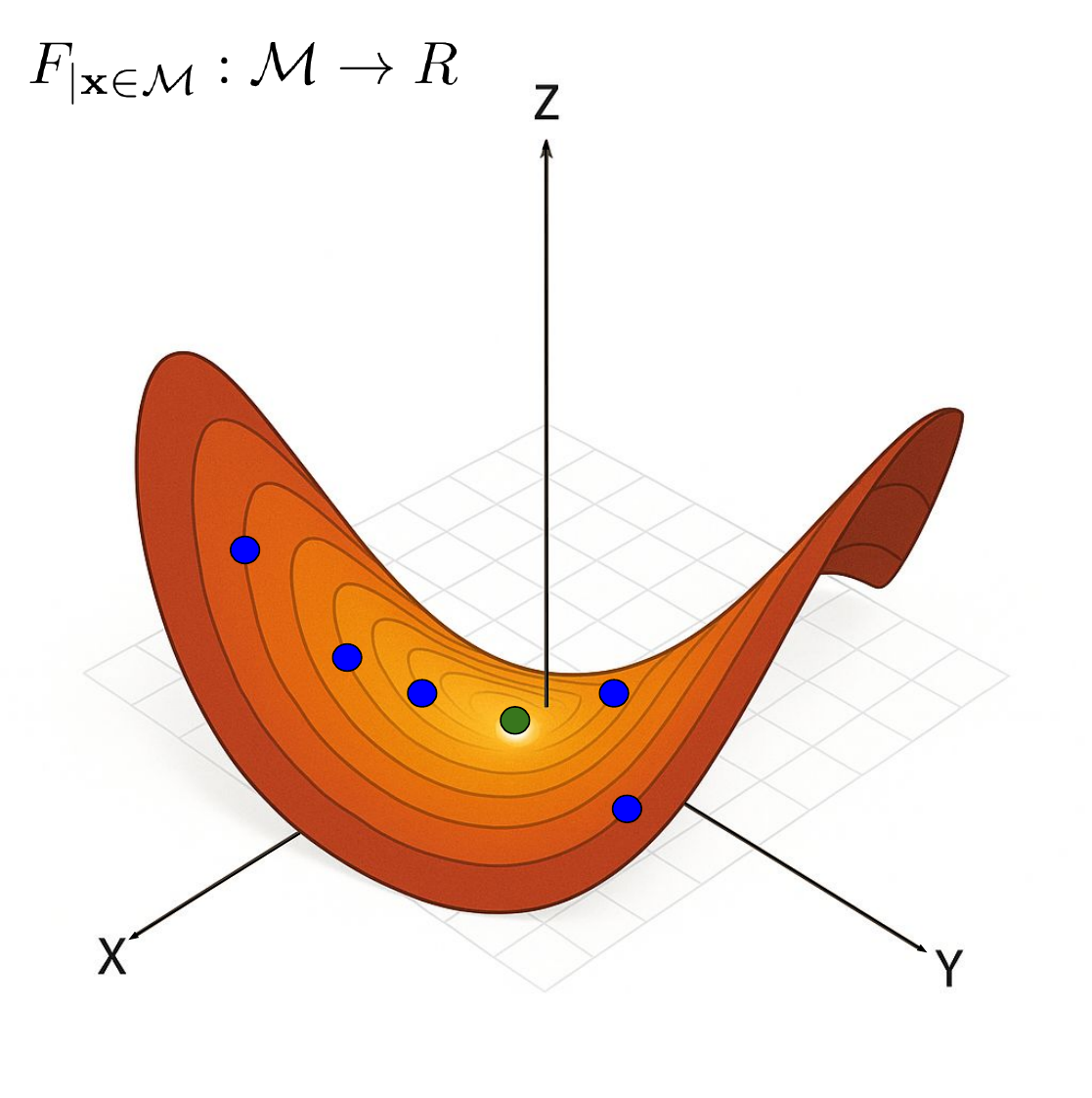}
    \vspace{-6mm}
    \caption{{An illustration of the heat kernel $F:\mathcal{M} \times \mathcal{M} \rightarrow \mathbb{R}$ defined on the saddle-like surface $\mathcal{M}$ with a distinguished point $\mathbf{x} \in \mathcal{M}$ marked green. The values $F(\mathbf{x}, \mathbf{y})$ for different $\mathbf{y} \in \mathcal{M}$ are color-coded (different shades of red), with a couple of samples $\mathbf{y}$ highlighted as blue dots. For a general manifold, $F$ is given by a system of partial differential equations and neither it nor the corresponding RF-mechanism are given by the closed-form expressions.}}
    \vspace{-3mm}
    \label{fig:first_figure}
\end{figure}

\subsubsection{Regular GRF random-walk construction}
% \vspace{-2mm}
\begin{algorithm*}[h]
\caption{\textbf{\textcolor{violet}{Regular GRFs:}} Construct vectors $\phi_{f}(i) \in \mathbb{R}^N$ to approximate $\mathbf{K}_{\boldsymbol{\alpha}}(\mathbf{W})$}\label{alg:constructing_rfs_for_k_alpha}
\textbf{Input:} weighted adjacency matrix $\mathbf{W} \in \mathbb{R}^{N \times N}$, vector of unweighted node degrees (number of out-neighbours) $\mathrm{deg} \in \mathbb{R}^N$, modulation function $f:(\mathbb{N} \cup \{0\})  \to \mathbb{R}$, termination probability $p_\textrm{halt} \in (0,1)$, node $i \in \mathcal{N}$, number of random walks to sample $m \in \mathbb{N}$.

\textbf{Output:} signature vector $\phi_f(i) \in \mathbb{R}^N$. \\
% \vspace{-4mm}
\begin{algorithmic}[1]
\STATE initialize: $\phi_f(i) \leftarrow \boldsymbol{0}$
\FOR{$w = 1, ..., m$}
\STATE initialize: \lstinline{load} $\leftarrow 1$, \lstinline{current_node} $\leftarrow i$, \lstinline{terminated} $\leftarrow$ \lstinline{False}, \lstinline{walk_length} $\leftarrow 0$ 
\WHILE{ \lstinline{terminated} $=$ \lstinline{False}}
\STATE $\phi_f(i)$[\lstinline{current_node}] $\leftarrow$ $\phi_f(i)[\texttt{current$\_$node}]+\texttt{load}\times f\left(\right.$\lstinline{walk_length}$\left.\right)$
\STATE \lstinline{walk_length} $\leftarrow$ $\texttt{walk\_length} + 1$
\STATE \lstinline{new_node} $\leftarrow \lstinline{Unif} \left[ \mathcal{N}(\right.$\lstinline{current_node}$\left.)\right]$ \AlgComment{\textcolor{blue}{assign to one of neighbours}}
\STATE \lstinline{load} $\leftarrow$ $\texttt{load}\times \frac{\textrm{deg}\scriptsize{[\lstinline{current_node}]}}{1-p_\textrm{halt}} \times\mathbf{W}\left[\right.$\lstinline{current_node,new_node}$\left.\right]$ \AlgComment{\textcolor{blue}{update load}}
\STATE $\lstinline{current_node} \leftarrow \lstinline{new_node}$
\STATE \lstinline{terminated} $\leftarrow \left( t \sim \lstinline{False}(0,1) < 
p_\textrm{halt}\right)$ \AlgComment{\textcolor{blue}{draw RV $t$ to decide on termination}}
\ENDWHILE
\ENDFOR
\STATE normalize: $\phi_f(i) \leftarrow \phi_f(i)/m$
\end{algorithmic}
\end{algorithm*}

\subsection{Kernel reconstruction and scale alignment}
\label{sec:align}
In order to address the scaling challenge, as we go from the original continuous space to the discretized space, we conduct scale alignment that can be thought of as extra normalization.
In matrix form, letting $\mathbf{G}_\theta \in \mathbb{R}^{N\times N}$ with $(\mathbf{G}_\theta)_{ij}=g_\theta(\mathbf{x}_i,\boldsymbol{\omega}_j)$, the induced kernel is
$\mathbf{K}_\theta = \mathbf{G}_\theta\, \mathbf{G}_\theta^\top.$
To compare against a ground-truth kernel matrix $\mathbf{K}_{\mathrm{GT}}$ (see Sec. \ref{sec:diffusion_kernel} for how these values are computed), we apply a single global scaling that matches Frobenius norms:
\begin{equation}
\widetilde{\mathbf{K}}_\theta \;=\; \alpha \mathbf{K}_\theta,
\qquad
\alpha \;=\; \frac{\|\mathbf{K}_{\mathrm{GT}}\|_F}{\|\mathbf{K}_\theta\|_F}.
\label{eq:frob_scale}
\end{equation}
All kernel errors and visualizations use $\widetilde{\mathbf{K}}_\theta$.

% \newpage
\subsection{Gaussian-grid limit: from graph diffusion to continuous features}
\label{app:gaussian-limit}
This section consolidates the discrete-to-continuous argument behind the Gaussian-kernel sanity check in Sec.~\ref{sec:grf_euclidean}. We first show convergence of the rescaled grid Laplacian and its diffusion kernel, then recall the continuous positive Gaussian feature factorization, and finally give empirical validation on finite grids.

\subsubsection{Diffusion kernel on the discrete Euclidean grid}
Let $d\in\mathbb{N}$ and, for each $n\in\mathbb{N}$, let $h_n := \frac{1}{n}$. Let $\mathbb{T}^d$ be the $d$–dimensional torus and let $\mathrm{V}_n\subset\mathbb{T}^d$ be the regular $d$–dimensional grid
\[
\mathrm{V}_n := \Big\{x = (k_1 h_n,\dots,k_d h_n) : k_j\in\{0,\dots,n-1\}\Big\}.
\]
Consider the graph whose vertices are $\mathrm{V}_n$ and whose neighbors are the $2d$
wrap–around nearest neighbors. 
\label{subsec:convergence_kernel}

\begin{theorem}[Convergence of Discrete to Continuous Laplacian]\label{thm:convergenceoflaplacian}
    Define the averaging operator $\mathbf{T}_n:\mathbb{R}^{\mathrm{V}_n}\to \mathbb{R}^{\mathrm{V}_n}$ by
\[
(\mathbf{T}_n f)(x)
:= \frac{1}{2d}\sum_{y\sim x} f(y),
\]
where the sum is over the $2d$ neighbors $y$ of $x$ in the wrap-around grid. Furthermore, denote
the (random–walk) graph and rescaled Laplacian as
\[
\mathbf{L}_{\mathrm{curr}}^{(n)} := \mathbf{I} - \mathbf{T}_n \qquad 
\mathbf{L}_n := \frac{2d}{h_n^2}\,\mathbf{L}_{\mathrm{curr}}^{(n)}
      = \frac{2d}{h_n^2}\,(\mathbf{I} - \mathbf{T}_n).
\]
Let $\Delta$ be the usual Laplacian on $\mathbb{T}^d$. Then we have
\begin{equation}
    \mathbf{L}_n \xrightarrow[n\to\infty]{} -\Delta.
\end{equation}
\end{theorem}
\begin{proof}
For $f\in C^4(\mathbb{T}^d)$ define its restriction to the grid $\mathrm{V}_n$ by $f_n(x) := f(x)$, $x\in \mathrm{V}_n$.
Fix $x\in\mathbb{T}^d$ and $f\in C^4(\mathbb{T}^d)$. For each coordinate direction
$\mathbf{e}_k$ ($k=1,\dots,d$), we have the Taylor expansions
\[
f(x \pm h_n \mathbf{e}_k)
= f(x) \pm h_n \partial_k f(x)
  + \frac{h_n^2}{2}\,\partial_{kk} f(x)
  \pm \frac{h_n^3}{6}\,\partial_{kkk} f(\xi_\pm)
  + \frac{h_n^4}{24}\,\partial_{kkkk} f(\eta_\pm),
\]
for some points $\xi_\pm,\eta_\pm$ on the segment between $x$ and $x\pm h_n \mathbf{e}_k$.
Adding these two expansions gives
\[
f(x + h_n \mathbf{e}_k) + f(x - h_n \mathbf{e}_k)
= 2 f(x) + h_n^2 \partial_{kk} f(x) + O(h_n^4).
\]

Summing over $k=1,\dots,d$,
\begin{align*}
\sum_{y\sim x} f(y)
&= \sum_{k=1}^d \big( f(x + h_n \mathbf{e}_k) + f(x - h_n \mathbf{e}_k) \big) \\
&= \sum_{k=1}^d \big( 2 f(x) + h_n^2 \partial_{kk}f(x) + O(h_n^4) \big) \\
&= 2d\,f(x) + h_n^2 \sum_{k=1}^d \partial_{kk} f(x) + O(d h_n^4).
\end{align*}

Furthermore, we know
\begin{align*}
(\mathbf{L}_{\mathrm{curr}}^{(n)} f_n)(x)
&= f(x) - \frac{1}{2d}\sum_{y\sim x} f(y) \\
&= -\,\frac{h_n^2}{2d} \sum_{k=1}^d \partial_{kk} f(x) + O(h_n^4) \\
&= -\,\frac{h_n^2}{2d}\,\Delta f(x) + O(h_n^4),
\end{align*}
where we plugged in our earlier result in the second line. Multiplying by $\frac{2d}{h_n^2}$ gives
\[
(\mathbf{L}_n f_n)(x)
= \frac{2d}{h_n^2} (\mathbf{L}_{\mathrm{curr}}^{(n)} f_n)(x)
= -\,\Delta f(x) + O(h_n^2),
\]
for any $x\in \mathrm{V}_n$. Thus, for each smooth $f$,
\[
\mathbf{L}_n f_n \;\xrightarrow[n\to\infty]{}\; -\Delta f.
\]
\end{proof}

\begin{theorem}[Convergence of Discrete Diffusion Kernel to Periodized Gaussian]
\label{thm:discrete-diffusion-torus-convergence-gaussian}
Fix $\sigma>0$ and let $f\in C^4(\mathbb{T}^d)$.
Then we have
\[
\exp\!\Big(-\tfrac{\sigma^2}{2}\mathbf{L}_n\Big)\, f\big|_{\mathrm{V}_n}
 \;\xrightarrow[n\to\infty]{}\;
\Big(\exp\!\big(\tfrac{\sigma^2}{2}\Delta\big) f\Big)\big|_{\mathrm{V}_n}.
\]
Moreover, the limiting operator $\exp(\tfrac{\sigma^2}{2}\Delta)$ acts by
convolution with a periodized Gaussian kernel
\[
k_\sigma(x,y)
=
\sum_{k\in\mathbb{Z}^d}
\frac{1}{(2\pi\sigma^2)^{d/2}}
\exp\!\Big(-\frac{\|x-y+k\|^2}{2\sigma^2}\Big),
\qquad x,y\in\mathbb{T}^d.
\]
\end{theorem}

\begin{proof}
Using the result of Theorem \ref{thm:convergenceoflaplacian} and exponentiating the operators implies that for any fixed $t\ge 0$,
\[
e^{-t \mathbf{L}_n} f\big|_{\mathrm{V}_n}
\;\longrightarrow\;
\big(e^{t\Delta} f\big)\big|_{\mathrm{V}_n}.
\]
Taking $t=\sigma^2/2$ yields the stated convergence.

The operator $e^{t\Delta}$ on the torus $\mathbb{T}^d$ is the heat semigroup and
admits an integral representation
\[
(e^{t\Delta} f)(x)
=
\int_{\mathbb{T}^d} p_t(x,y)\,f(y)\,dy,
\]
where $p_t$ is the heat kernel on the torus. This kernel is obtained by
periodizing the Euclidean heat kernel \cite{brandolini2010quadraturerulesdistributionpoints}
\[
p_t(x,y)
=
\sum_{k\in\mathbb{Z}^d}
\frac{1}{(4\pi t)^{d/2}}
\exp\!\Big(-\frac{\|x-y+k\|^2}{4t}\Big).
\]

Setting $t=\sigma^2/2$ gives
\[
p_{\sigma^2/2}(x,y)
=
\sum_{k\in\mathbb{Z}^d}
\frac{1}{(2\pi\sigma^2)^{d/2}}
\exp\!\Big(-\frac{\|x-y+k\|^2}{2\sigma^2}\Big),
\]
which is a Gaussian kernel with variance $\sigma^2$ in each coordinate,
periodized to respect the torus geometry.
\end{proof}

\textbf{Remark:} If instead of the discrete grid, we worked on $\mathbb{R}^d$ (no wrap–around), the corresponding
kernel would be the standard Gaussian
\[
\mathrm{K}_\sigma(\mathbf{x},\mathbf{y})
:= \frac{1}{(2\pi\sigma^2)^{d/2}}
   \exp\!\Big(-\,\frac{\|\mathbf{x}-\mathbf{y}\|^2}{2\sigma^2}\Big),
\]
which is the usual Gaussian kernel with bandwidth $\sigma$. On the torus we obtain its periodic version by summing over $k\in\mathbb{Z}^d$.

\subsubsection{Euclidean Gaussian feature representation}
\label{subsection:whatisg}
\begin{theorem}[RBF Kernel Factorization] \label{thm:gaussian-feature-representation-kernel}
Let $d \in \mathbb{N}$ and $\sigma > 0$, and define
\[
g_\sigma(\mathbf{x},\boldsymbol{\omega})
:= \left(\frac{2}{\pi\sigma^2}\right)^{d/4}
\exp\!\left(-\frac{\|\mathbf{x}-\boldsymbol{\omega}\|^2}{\sigma^2}\right),
\qquad \mathbf{x},\boldsymbol{\omega} \in \mathbb{R}^d.
\]
Then for all $\mathbf{x},\mathbf{y} \in \mathbb{R}^d$,
\[
\int_{\mathbb{R}^d} g_\sigma(\mathbf{x},\boldsymbol{\omega})\, g_\sigma(\mathbf{y},\boldsymbol{\omega})\, d\boldsymbol{\omega}
= \exp\!\left(-\frac{\|\mathbf{x}-\mathbf{y}\|^2}{2\sigma^2}\right).
\]
In particular, $g_\sigma$ is a continuous feature map for the Gaussian kernel with bandwidth~$\sigma$.
\end{theorem}

\begin{proof}
Set
\( C_\sigma := \left(\frac{2}{\pi\sigma^2}\right)^{d/4} \) and $\mathbf{m} := \frac{\mathbf{x}+\mathbf{y}}{2}$ for fixed $\mathbf{x},\mathbf{y} \in \mathbb{R}^d$. By Lemma \ref{lemma:helpfulnorm} we have
\[
\| \mathbf{x}-\boldsymbol{\omega}\|^2 + \|\mathbf{y}-\boldsymbol{\omega}\|^2
= 2\|\boldsymbol{\omega}-\mathbf{m}\|^2 + \tfrac12 \|\mathbf{x}-\mathbf{y}\|^2.
\]
Therefore,
\begin{align*}
\int_{\mathbb{R}^d} g_\sigma(\mathbf{x},\boldsymbol{\omega})\,g_\sigma(\mathbf{y},\boldsymbol{\omega})\, d\boldsymbol{\omega}
&=
C_\sigma^2
\int_{\mathbb{R}^d}
\exp\!\left(
-\frac{\|\mathbf{x}-\boldsymbol{\omega}\|^2}{\sigma^2}
-\frac{\|\mathbf{y}-\boldsymbol{\omega}\|^2}{\sigma^2}
\right)\,d\boldsymbol{\omega}
\\
&=
C_\sigma^2
\exp\!\left(-\frac{\|\mathbf{x}-\mathbf{y}\|^2}{2\sigma^2}\right)
\int_{\mathbb{R}^d}
\exp\!\left(-\frac{2\|\boldsymbol{\omega}-\mathbf{m}\|^2}{\sigma^2}\right)\,d\boldsymbol{\omega}
\\
&=
C_\sigma^2
\exp\!\left(-\frac{\|\mathbf{x}-\mathbf{y}\|^2}{2\sigma^2}\right)
\left(\frac{\pi\sigma^2}{2}\right)^{d/2}
\\
&=
\exp\!\left(-\frac{\|\mathbf{x}-\mathbf{y}\|^2}{2\sigma^2}\right),
\end{align*}
noting that for $a>0$, $\int_{\mathbb{R}} e^{-a u^2}\,du = \sqrt{\frac{\pi}{a}}$.
\end{proof}

\begin{lemma}\label{lemma:helpfulnorm}
For any $\mathbf{x},\mathbf{y}$ in $\mathbb{R}^d$ let $\mathbf{m} := \frac{\mathbf{x}+\mathbf{y}}{2}$, then we have
\begin{equation}\label{eq:cs}
\|\mathbf{x}-\boldsymbol{\omega}\|^2 + \|\mathbf{y}-\boldsymbol{\omega}\|^2
= 2\|\boldsymbol{\omega} - \mathbf{m}\|^2 + \tfrac12 \|\mathbf{x}-\mathbf{y}\|^2,
\qquad \boldsymbol{\omega} \in \mathbb{R}^d.
\end{equation}
\end{lemma}

\begin{proof}
We have
\begin{align*}
\|\mathbf{x}-\boldsymbol{\omega}\|^2 + \|\mathbf{y}-\boldsymbol{\omega}\|^2
&= (\|\mathbf{x}\|^2 - 2\langle \mathbf{x},\boldsymbol{\omega}\rangle + \|\boldsymbol{\omega}\|^2)
 + (\|\mathbf{y}\|^2 - 2\langle \mathbf{y},\boldsymbol{\omega}\rangle + \|\boldsymbol{\omega}\|^2) \\
&= 2\|\boldsymbol{\omega}\|^2 -2\langle \mathbf{x}+\mathbf{y},\boldsymbol{\omega}\rangle + \|\mathbf{x}\|^2+\|\mathbf{y}\|^2 \\
&= 2\left(\|\boldsymbol{\omega}\|^2 - \Big\langle \mathbf{x}+\mathbf{y},\boldsymbol{\omega}\Big\rangle + \Big\|\frac{\mathbf{x}+\mathbf{y}}{2}\Big\|^2\right)
 +\|\mathbf{x}\|^2+\|\mathbf{y}\|^2- \tfrac12\|\mathbf{x}+\mathbf{y}\|^2\\
&= 2\|\boldsymbol{\omega} - \mathbf{m}\|^2 + \tfrac12 \|\mathbf{x}-\mathbf{y}\|^2.
\end{align*}
\end{proof}

\subsubsection{Empirical Gaussian-grid validation} \label{app:gaussian-grid}
Following the discussion in Sec. \ref{sec:grf_euclidean}, for each dimension $d\in\{2,4,...,32\}$, we discretize the hypercube $[0,1]^d$ via the regular grid $\mathrm{V}_n$ with spacing $h_n = 1/n$ and $n^d$ nodes, connecting each node to its $2d$ wrap–around nearest neighbors. For each choice of $n$ and $d$, we construct signature vectors $\phi_f(x)$ at every grid node $x\in \mathrm{V}_n$ using the random walk procedure of Algorithm 1 with a fixed diffusion scale $\sigma=0.2$, termination probability $p_{\text{halt}} = 0.005$, $m=100{,}000$ walks per node, and vary $n$ from $5$ to $105$ in steps of $10$. We then form rescaled vectors $\psi_f(x) = c_{d,\sigma,n}\,\phi_f(x)$. 

\textbf{Note:} On the ($d$-dimensional grid) $\mathrm{V}_n$, the random-walk Laplacian decomposes as a Kronecker sum of 1D operators. Hence, the corresponding terms $e^{-tL/2}$ and the induced kernel $e^{-tL}$ factorize across coordinates; see Prop. 20 \cite{cadavid2022approachevaluatingcertaintrigonometric} for further details. We exploit this to compute the 1D quantities and tensorize to obtain the $d$-dimensional results.

We compute the relative mean squared error (MSE) of approximating  $g_\sigma(\mathbf{c},*)$ with $\psi_f(c)[*]$, as well as that of approximating true kernel values with those induced by MRFs (in the grid points). 
The latter are given as:
\begin{equation}
\widehat{\mathrm{K}}^{\mathrm{Gauss}}_{\sigma,n}(x,y)
  :=
  \bigl\langle \psi_f(x),\psi_f(y)\bigr\rangle,
  \qquad x,y\in \mathrm{V}_n.
\end{equation}
Across all tested dimensions, we observe that errors decrease rapidly as the grid is refined (see: Fig.~\ref{fig:MSEs-euclidean}).
\begin{figure}[h]
    \centering
    \includegraphics[width=0.49\linewidth]{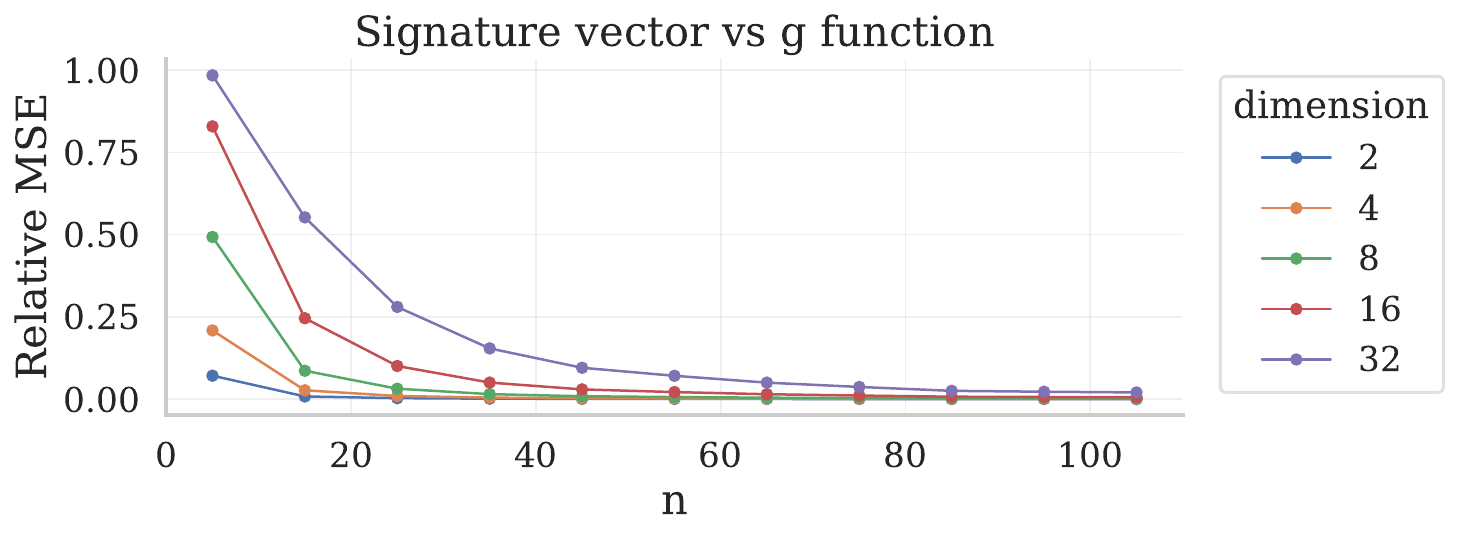}
    \includegraphics[width=0.49\linewidth]{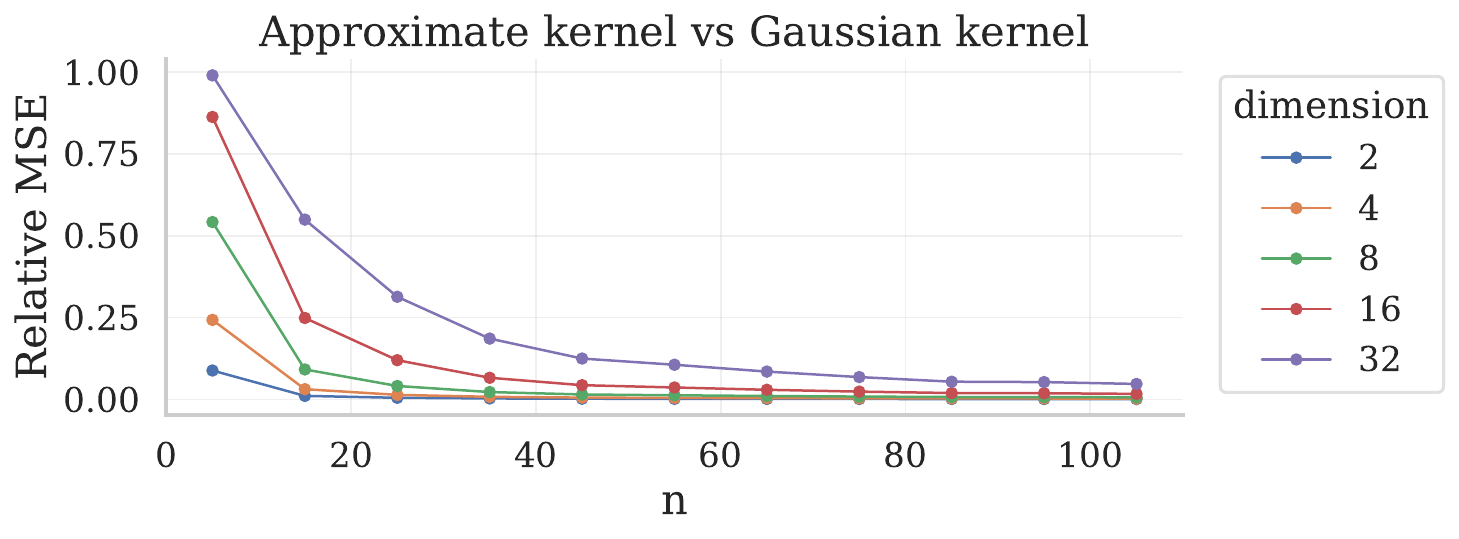}
    % \vspace{-6mm}
    \caption{{Empirical relative mean squared errors (MSEs) of the approximations of $g$-functions and true Gaussian kernel values with signature vectors and kernels induced by them for $d \in \{2,4,...,32\}$ and $n=5,15,...,105$. Relative MSE of the approximation of vector $\mathbf{y}$ with random vector $\mathbf{x}$ is defined as $\mathbb{E}[||\mathbf{x}-\mathbf{y}||^2_2/||\mathbf{y}||^2_2]$. We use $s=30$ repetitions for each $n$ value (standard deviations not visible; negligible relative to means).}}
    \label{fig:MSEs-euclidean}
    % \vspace{-3mm}
\end{figure}

\subsubsection{Convergence visualization for MRF-induced Gaussian kernels}
\label{subsec:convergenceofgrf}

We use the Gaussian-grid setting above.
For a representative grid point $c$ (we choose the center $c=(\tfrac{1}{2},\dots,\tfrac{1}{2})$), we visualize the field $\omega\mapsto \psi_n(c)[\omega]$ over $\mathrm{V}_n$ and compare it to the analytic feature $g_\sigma(c,\omega)$ evaluated on the same grid.
As $n$ increases, the discrete field becomes increasingly smooth and radially symmetric, approaching the shape of the Gaussian bump predicted by the theory (see Fig.~\ref{fig:convergence-visualization-gaussian-kernel-2d} below for the two–dimensional case).

\begin{figure*}[h]
    \centering
    \includegraphics[width=0.95\linewidth]{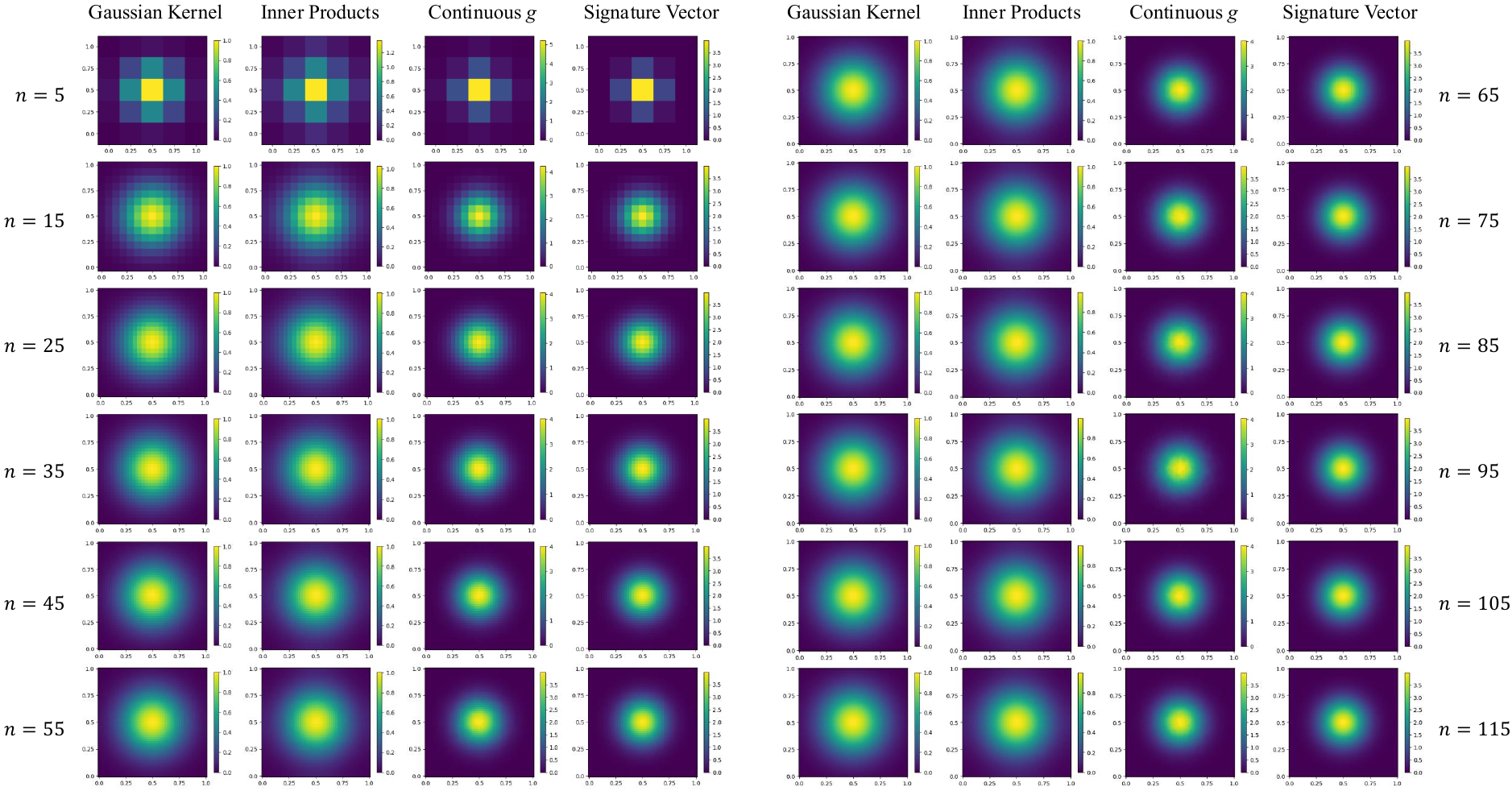}
    \caption{{\textbf{First column:} Exact values of the Gaussian kernel, computed in the grid points and broadcast to all the points in the continuous rectangle corresponding to that grid point. \textbf{Second column:} Inner(dot)-products of the signature vectors, leading to the approximate Gaussian kernel values. \textbf{Third column:} The values of the exact function $g_{\sigma}(c,\omega)$ (for a fixed $c$), computed in the grid points and broadcast to all the points in the continuous rectangle corresponding to that grid point. \textbf{Fourth column:} Values of the renormalized signature vector $\psi(c)[\omega]$ for a fixed $c$. We use the following setting: $d=2$, $\sigma=0.2$, $p_\text{halt}=0.005$, and $m=100{,}000$.}}
    \label{fig:convergence-visualization-gaussian-kernel-2d}
\end{figure*}

\subsection{Surface heat-kernel reconstruction experiments}\label{app:surfaces}
This section gives the mathematical ground truth and full experimental details for the 2D embedded-surface experiments, followed by the ambient Euclidean RF baseline sweep and validation diagnostics.

\subsubsection{Sphere ground-truth heat kernel} \label{app:spheres}
For the unit hypersphere $S^{d} \subset \mathbb{R}^{d+1}$, the heat kernel
admits an explicit spectral expansion in terms of Gegenbauer polynomials.
Following Theorem~1 of \citet{Zhao_2018}, the exact hyperspherical heat
kernel $G^{\mathrm{ext}}(\mathbf{x},\mathbf{y};t)$ on $S^{n-1}$ (with $n=d+1$)
can be written as the uniformly and absolutely convergent series
\begin{equation}
  G^{\mathrm{ext}}(\mathbf{x},\mathbf{y};t)
  \;=\;
  \sum_{\ell=0}^{\infty}
    e^{-\ell(\ell+n-2)t}\,
    \frac{2\ell+n-2}{n-2}\,
    \frac{1}{A_{S^{n-1}}}\,
    C_{\ell}^{\frac{n-2}{2}}(\mathbf{x}\cdot\mathbf{y}),
  \label{eq:heat-kernel-Sd}
\end{equation}
where $C_{\ell}^{\alpha}$ denotes the Gegenbauer polynomial of order $\ell$
and index $\alpha$, $\mathbf{x}\cdot\mathbf{y} \in [-1,1]$, $t>0$, and
\[
  A_{S^{n-1}}
  \;=\;
  \frac{2\pi^{\frac{n}{2}}}{\Gamma\!\left(\frac{n}{2}\right)}
\]
is the surface area of $S^{n-1}$.

In the case $d=2$ (so $n=3$ and we are on $S^{2}$), we have
$\alpha = \tfrac{n-2}{2} = \tfrac{1}{2}$, and the Gegenbauer polynomials
reduce to Legendre polynomials via
$C_{\ell}^{1/2}(z) = P_{\ell}(z)$. Therefore \eqref{eq:heat-kernel-Sd}
specializes to the classical spherical expansion
\begin{equation} \label{eq:heat_kernel_sphere_2D}
  G^{\mathrm{ext}}(\mathbf{x},\mathbf{y};t)
  \;=\;
  \sum_{\ell=0}^{\infty}
    \frac{2\ell+1}{4\pi}e^{-\ell(\ell+1)t}\,
    P_{\ell}(\langle\mathbf{x},\mathbf{y}\rangle).
\end{equation}
For numerical ground truth, we truncate \eqref{eq:heat_kernel_sphere_2D} at
$\ell \le L_{\max}$:
\[
\mathrm{K}^{\text{true}}_t(\mathbf{x}_i,\mathbf{x}_j)
\;\approx\;
\sum_{\ell=0}^{L_{\max}}
  \frac{2\ell+1}{4\pi}\,
  e^{-\ell(\ell+1)t}\,
  P_\ell(\langle \mathbf{x}_i, \mathbf{x}_j\rangle),
\]
where $L_{\max}$ is chosen so that the tail mass is negligible. In our experiments, we choose $L_{\max}=50$.

\subsubsection{Common setup and surface-specific details}\label{app:surfaces-setup}
For each geometry, we discretize the surface into \(N=4000\) points, build a \(k\)NN graph (with \(k=24\) for the M\"obius strip and \(k=8\) for the other three), generate signature feature vectors from 1000 (sampled) start points with parameters: $m=100,000$ (number of random walks), $p_{\text{halt}}=0.01$, and $\sigma^2=20$. We then train a neural network predictor \(g_\theta(\mathbf{x},\omega)\) to learn \(\phi_f(x)[\omega]\).
Our network architecture is a small multilayer perceptron that takes as input the coordinates of the pair $(\mathbf{x},\omega)$ along with an approximation of the geodesic distance $d_{\mathcal{M}}(\mathbf{x},\omega)$, (computed via a shortest path calculation between $x$ and $\omega$), and outputs the signature-vector values. Training was run for $1000$ epochs with an absolute relative error loss (clamped at $\varepsilon=0.1$), defined as follows:
\begin{equation}
    \mathcal{L}(g_\theta(\mathbf{x},\omega),\phi_{f}(x)(\omega))=\frac{|g_\theta(\mathbf{x},\omega)-\phi_{f}(x)(\omega)|}{\max\{\phi_{f}(x)(\omega),\varepsilon\}},
\end{equation} 
and optimized with Adam. We report root mean square error (RMSE), validation diagnostics (prediction vs.\ actual and error vs.\ truth; those in the Appendix: Sec. \ref{sec:val-diagnostics}), qualitative 3D visualizations of \(g_\theta(\mathbf{x},\cdot)\), \(\phi_f(x)[\cdot]\), and their difference for representative validation starts, and heat-kernel comparisons using the kernel induced by the learned features. From the learned feature matrix \(\mathbf{G}_\theta\) we form \(\mathbf{K}_{\theta} = \mathbf{G}_\theta \mathbf{G}_\theta^\top\), and compare it against the ground-truth kernel matrix \(\mathbf{K}_{\mathrm{GT}}\). Before reporting kernel errors, we apply the Frobenius-norm alignment (see: Appendix, Sec. \ref{sec:align}) and evaluate the renormalized version \(\widetilde{\mathbf{K}}_\theta\).

\textbf{Sphere:} We embed a sphere \(x^2+y^2+z^2=1\) in \(\mathbb{R}^3\). For the ground truth, we use the analytic spherical heat kernel (see: Sec. \ref{app:spheres}) truncated at \(L_{\max}=50\) spherical harmonics. The diffusion-time used in the analytic kernel is set to $t_{\text{analytical}}\approx0.25$ to match the diffusion length observed on the discretized graph (and account for the number of nodes $N$). 

\textbf{Ellipsoid:} We embed an ellipsoid \(\frac{x^2}{a^2}+\frac{y^2}{b^2}+\frac{z^2}{c^2}=1\) in \(\mathbb{R}^3\) (with \(a=1.0, b=1.3, c=0.7\))
and discretize it via scaled Fibonacci construction.
% We use the same signature-walk parameters as on the sphere.
As a reference kernel, we use the graph heat kernel \(\mathrm{K}_{\mathrm{true}}=\exp(-t\mathbf{L})\) computed from a \(k\)NN graph Laplacian on the ellipsoid.

\textbf{M\"obius strip:} We embed a M\"obius strip in \(\mathbb{R}^3\) and discretize it using a uniform grid in the intrinsic parameters \((u,v)\).
The ground truth is constructed similarly as for the ellipsoid. The M\"obius strip is in principle more sensitive to discretization and graph construction (e.g., near-duplicate points, irregular sampling density, or disconnected \(k\)NN graphs can destabilize training).

\subsubsection{Additional Results: Torus}
\label{sec:torus}
We embed a torus in \(\mathbb{R}^3\) and discretize it using a uniform grid. We then proceed as for the sphere and ellipsoid case. The conclusions are the same as for the sphere and the ellipsoid
(see: Fig. \ref{fig:torus_field}).

\begin{figure}[H]
    \centering
    \includegraphics[width=0.45\linewidth]{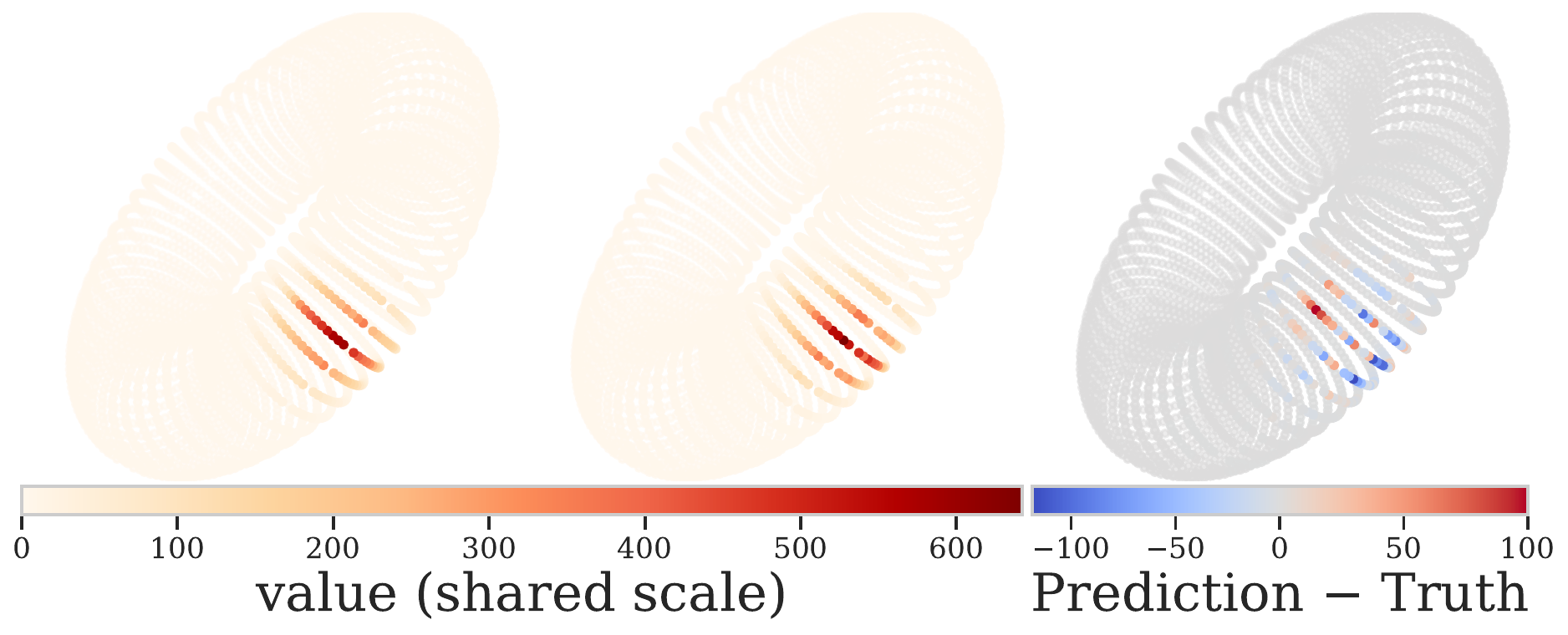}
    \includegraphics[width=0.45\linewidth]{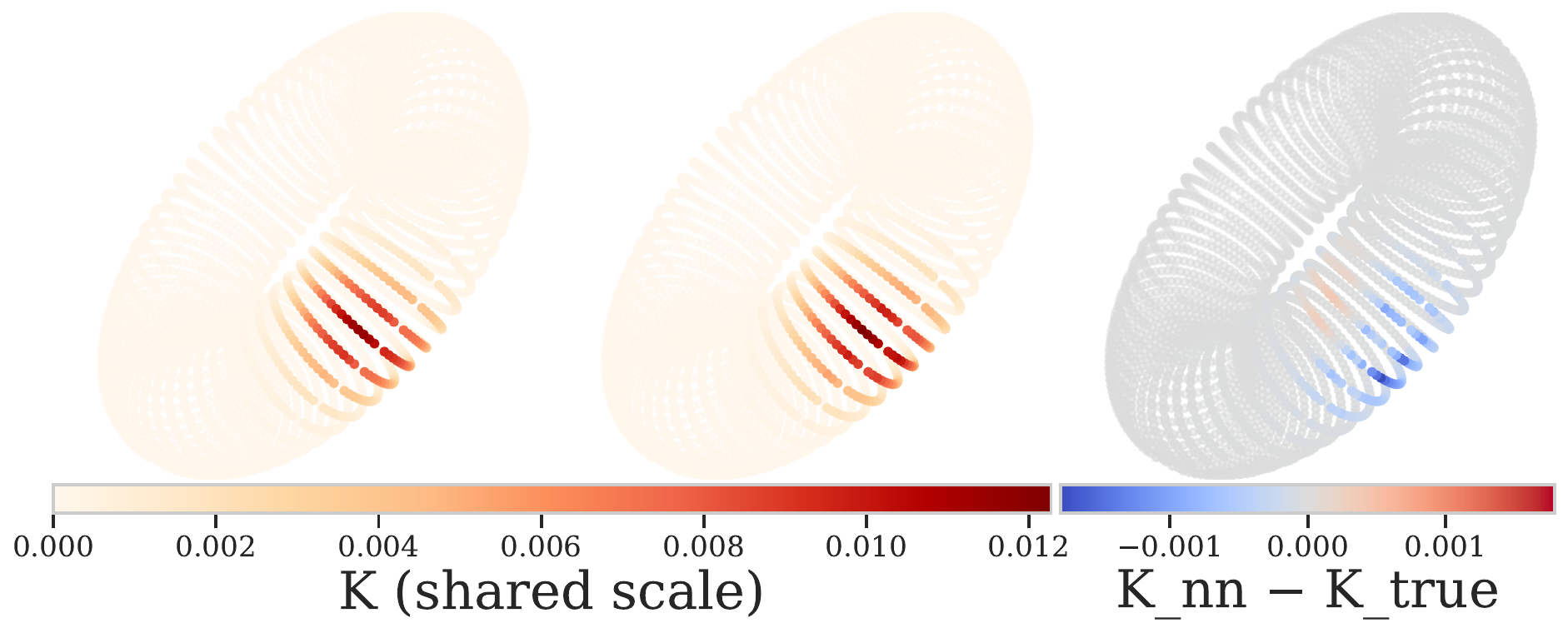}
    \caption{Qualitative field and heat kernel comparison for a torus. Example validation start (left):
    predicted \(g_\theta(\mathbf{x},\cdot)\), ground truth \(\phi_f(\mathbf{x})[\cdot]\), and residual on the torus. Reference graph heat kernel (right) \(\mathrm{K}^{\mathrm{heat}}(\mathbf{x},\cdot)\)
    compared to Frobenius-aligned induced kernel $\widetilde{\mathrm{K}}_{\theta}(\mathbf{x},\cdot)=\widetilde{\mathrm{K}}_{\mathrm{NN}}(\mathbf{x},\cdot)$, and their difference, for a representative start point.}
    \label{fig:torus_field}
\end{figure}
% \vspace{-3mm}

\subsubsection{Validation Diagnostics for All Shapes}
\label{sec:val-diagnostics}

Validation diagnostics results for all the shapes, confirming strong performance of the MRF method, are presented in Fig. \ref{fig:sphere_val_diag}, Fig. \ref{fig:ellipsoid_val_diag}, Fig. \ref{fig:mobius_val_diag} and Fig. \ref{fig:torus_val_diag}.

\begin{figure*}[!h]
    \centering
    \includegraphics[width=0.32\linewidth]{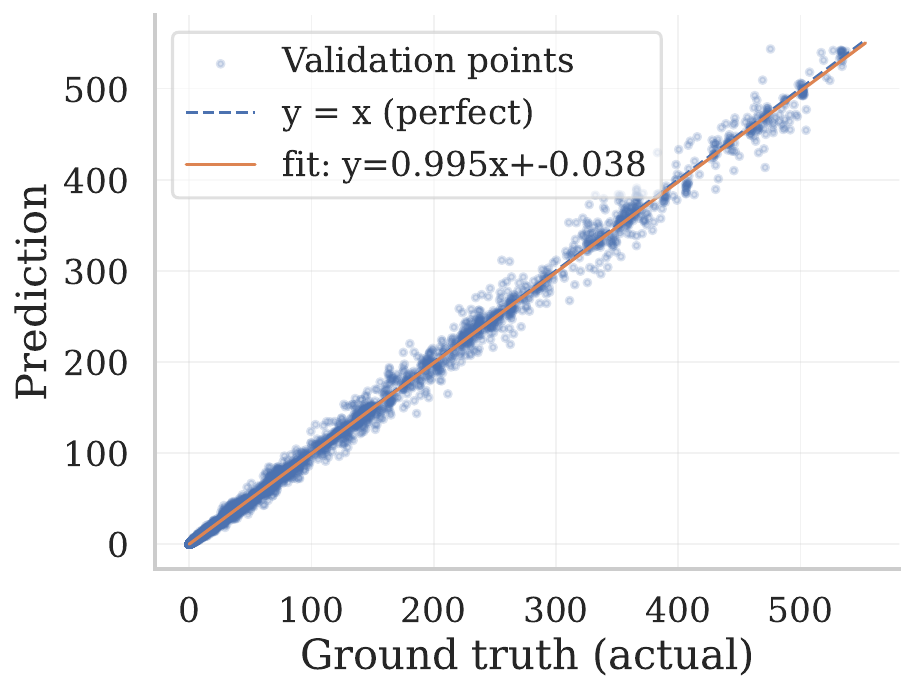}
    \includegraphics[width=0.32\linewidth]{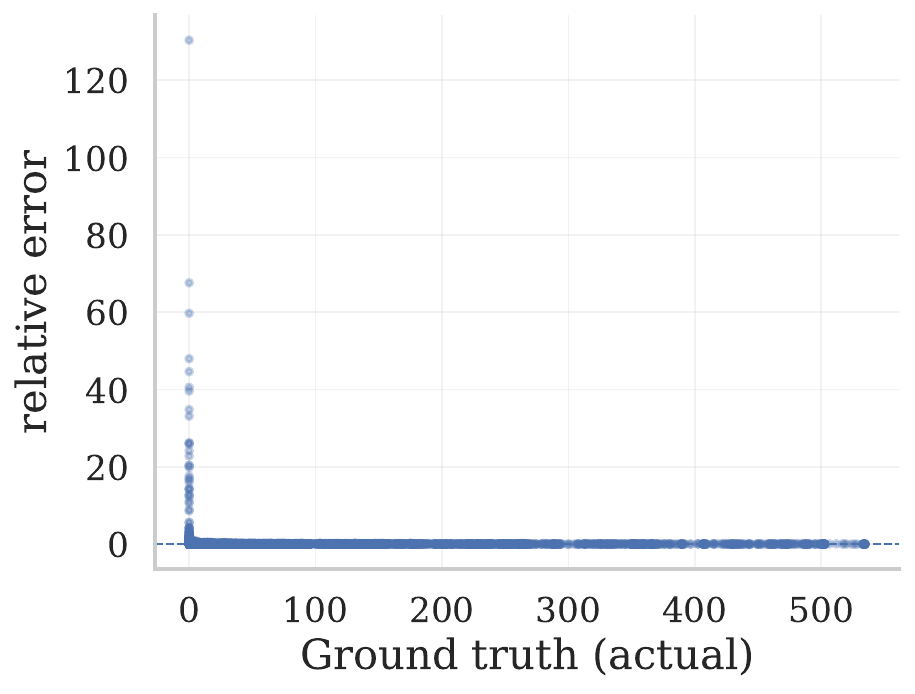}
    \includegraphics[width=0.32\linewidth]{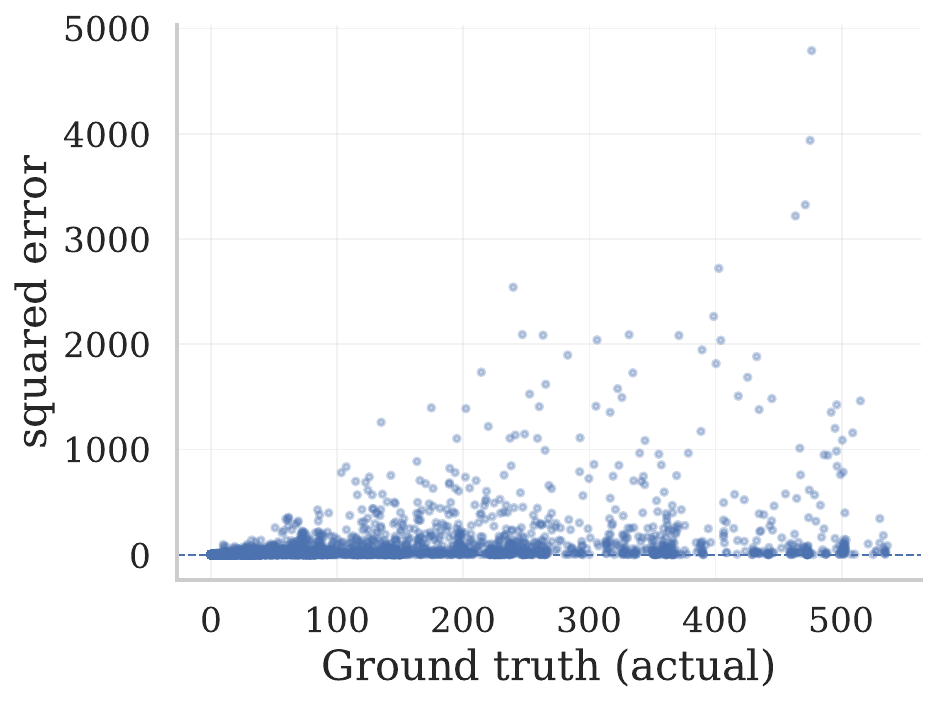}
    \caption{\textbf{Sphere validation diagnostics.} Left: predicted vs.\ ground-truth \(\varphi_t\) values.
    Middle/right: relative error and squared error as a function of ground-truth magnitude, highlighting where the model under/over-estimates.}
    \label{fig:sphere_val_diag}
\end{figure*}
\begin{figure*}[!h]
    \centering
    \includegraphics[width=0.32\linewidth]{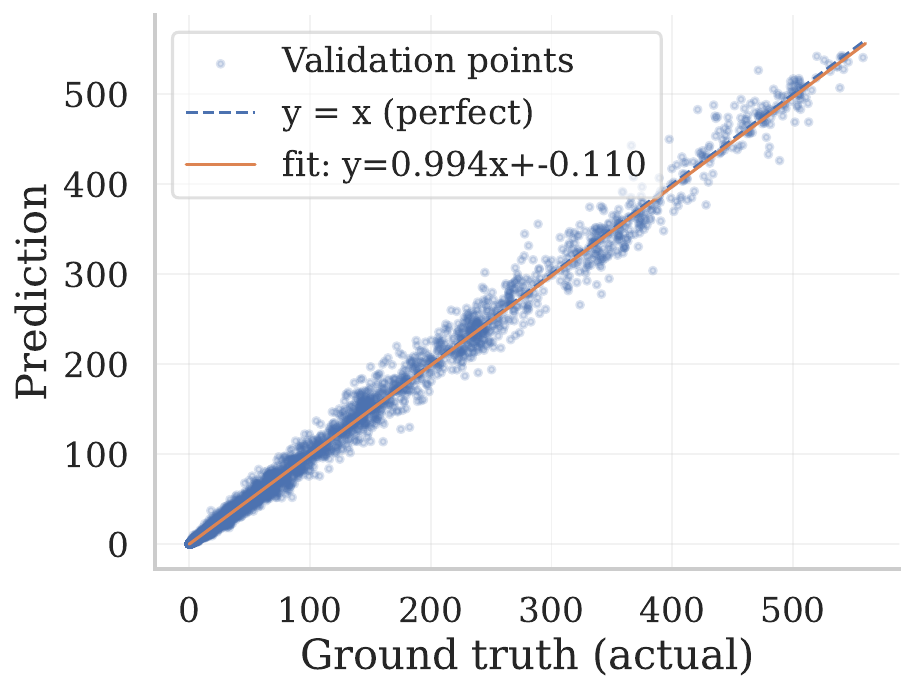}
    \includegraphics[width=0.32\linewidth]{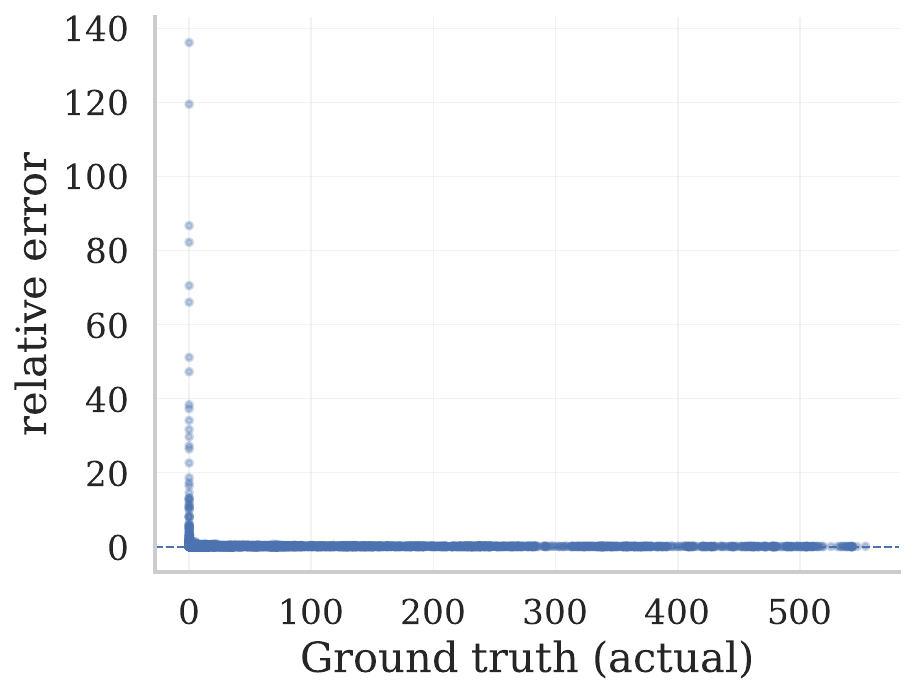}
    \includegraphics[width=0.32\linewidth]{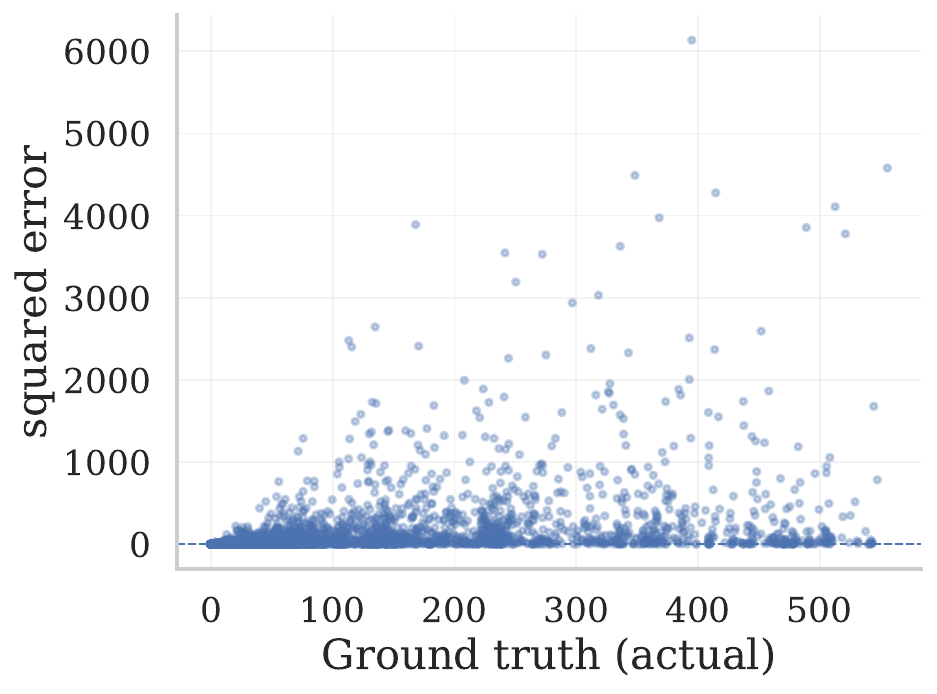}
    \caption{\textbf{Ellipsoid validation diagnostics.} Left: predicted vs.\ ground-truth \(\varphi_t\) values.
    Middle/right: relative and squared errors as a function of ground-truth magnitude.}
    \label{fig:ellipsoid_val_diag}
\end{figure*}
\begin{figure*}[!h]
    \centering
    \includegraphics[width=0.32\linewidth]{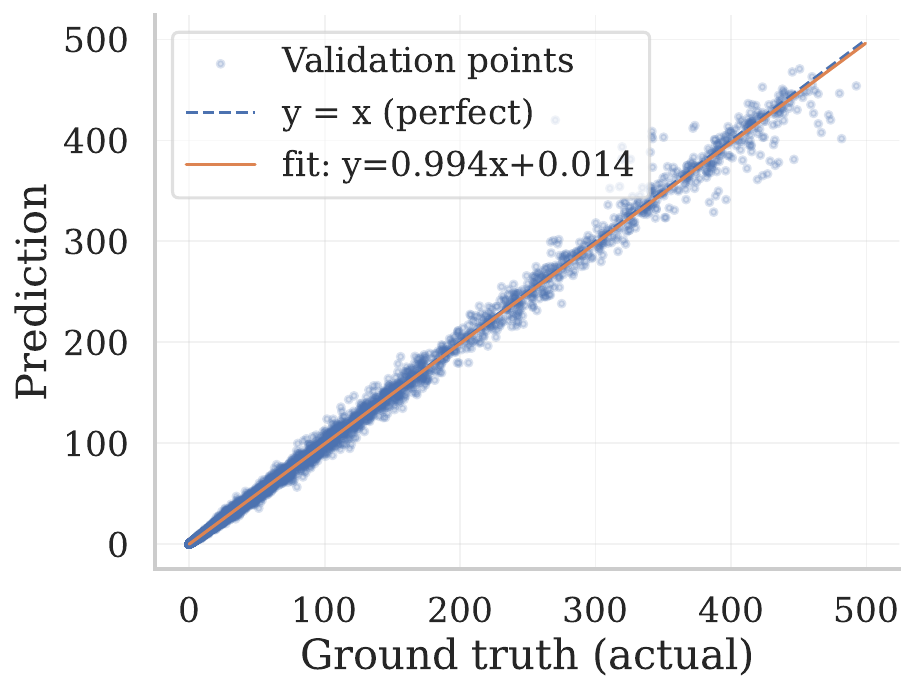}
    \includegraphics[width=0.32\linewidth]{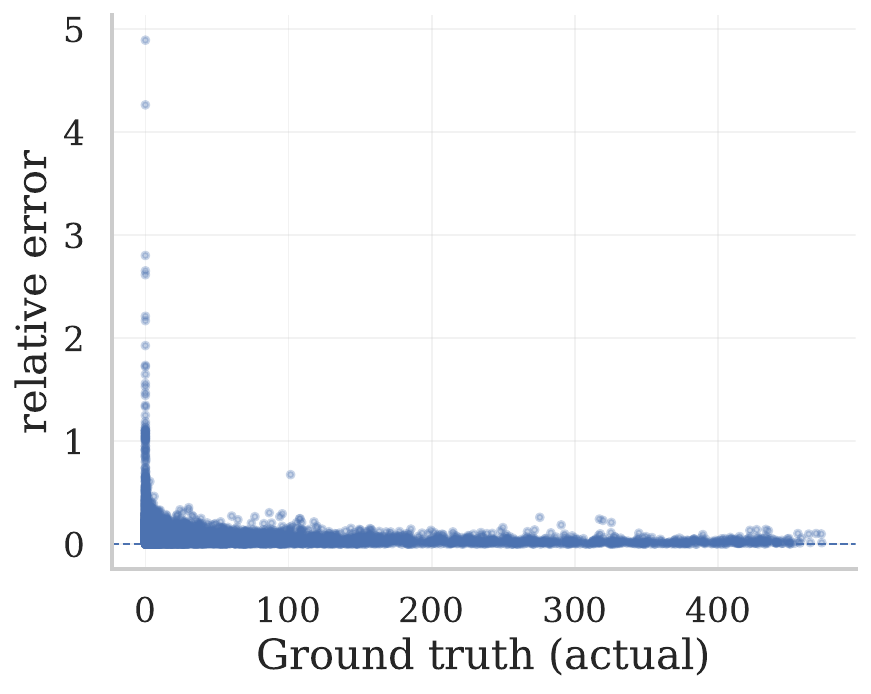}
    \includegraphics[width=0.32\linewidth]{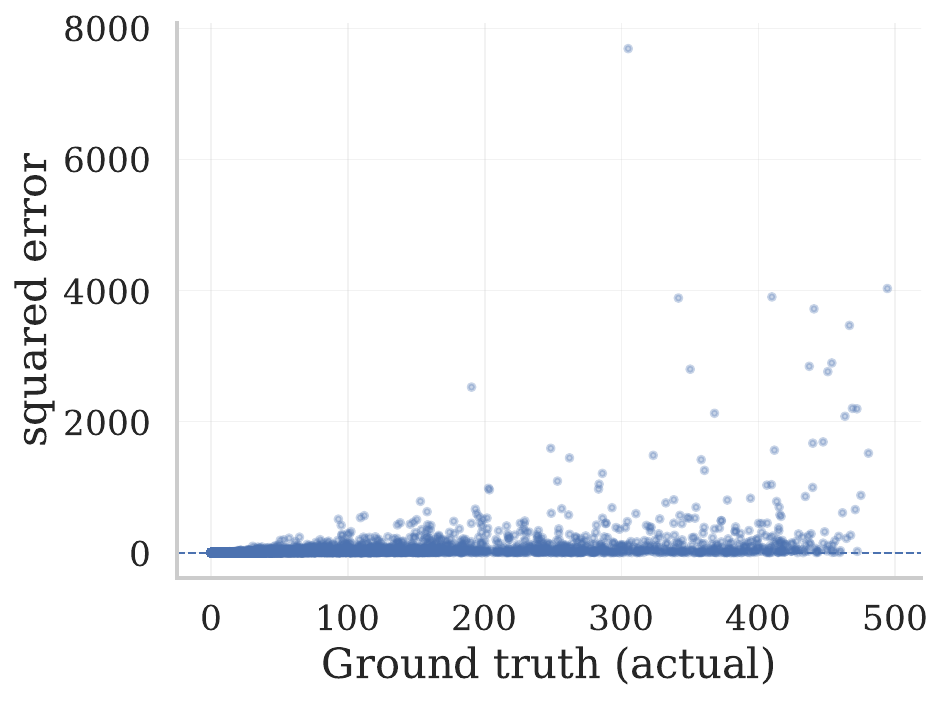}
    \caption{\textbf{M\"obius validation diagnostics.} Predicted vs.\ ground-truth \(\varphi_t\) values (left) and error vs.\ truth (middle/right),
    illustrating calibration and where errors concentrate.}
    \label{fig:mobius_val_diag}
\end{figure*}
\begin{figure*}[!h]
    \centering
    \includegraphics[width=0.32\linewidth]{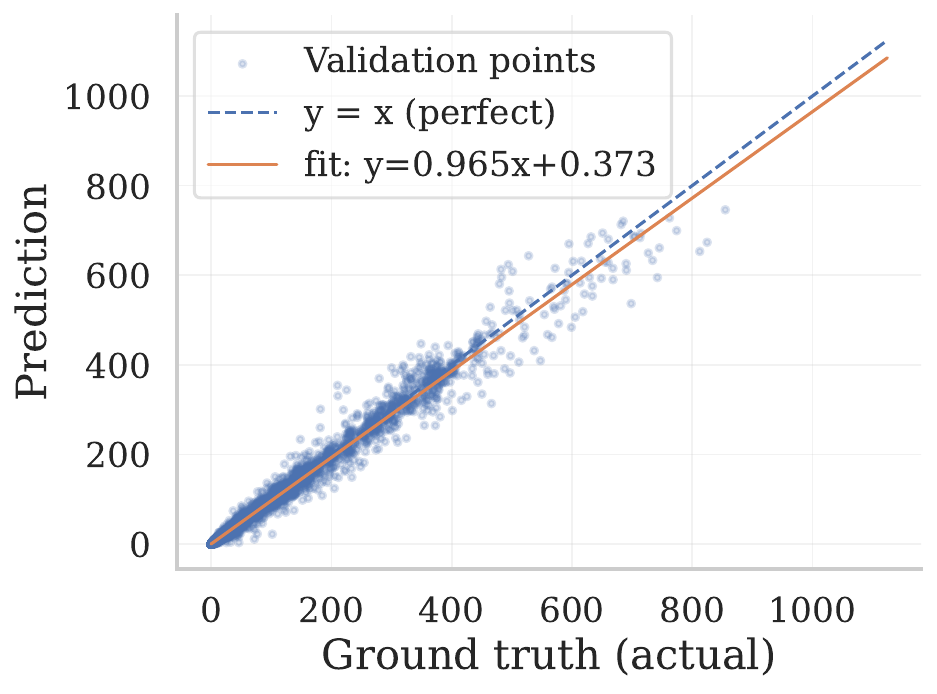}
    \includegraphics[width=0.32\linewidth]{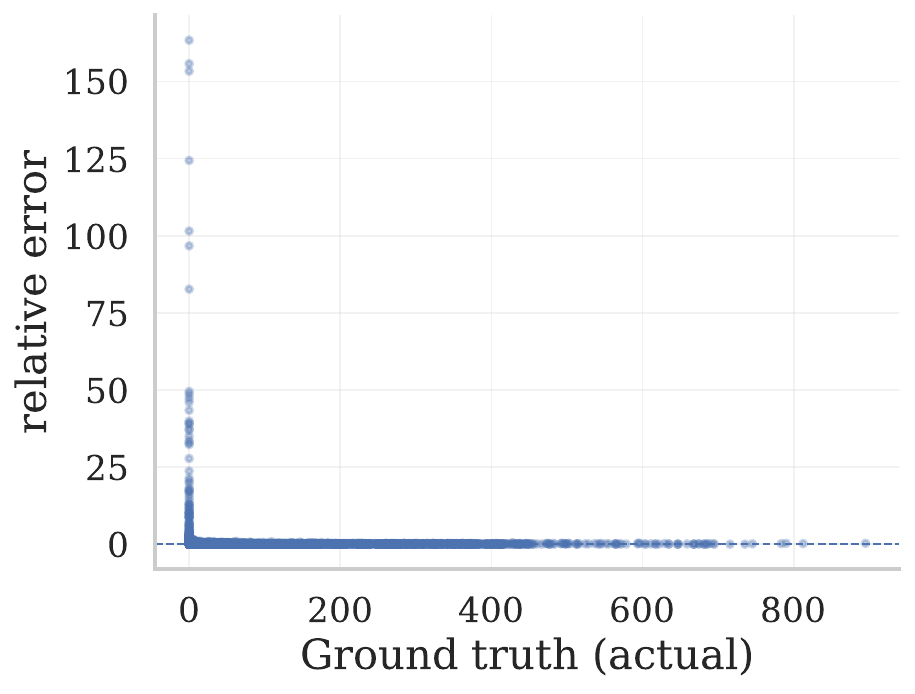}
    \includegraphics[width=0.32\linewidth]{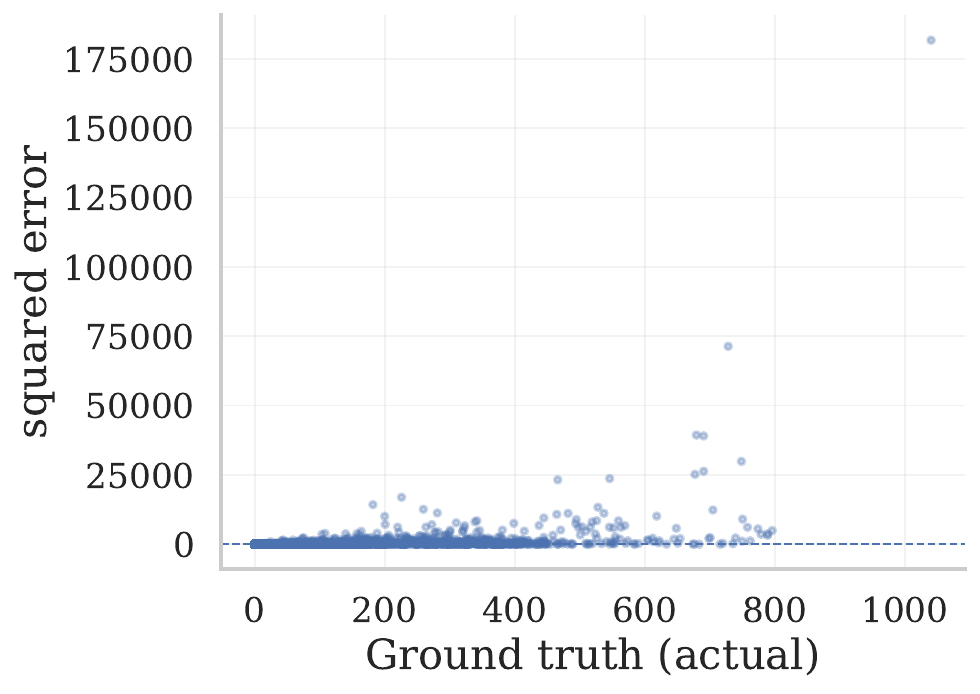}
    \caption{\textbf{Torus validation diagnostics.} Predicted vs.\ ground-truth \(\varphi_t\) values (left) and error vs.\ truth (middle/right),
    illustrating calibration and where errors concentrate.}
    \label{fig:torus_val_diag}
\end{figure*}

\subsubsection{Ambient-space Euclidean random-feature baselines} \label{app:ambient-rf}
We evaluate all methods on kernel reconstruction against the ground-truth manifold heat kernel, reporting $R^2$, RMSE, MAE, relative error (mean and median), and relative Frobenius norm.

Table~\ref{tab:tuned-sweep-best} expands the compact main-text comparison in Table~\ref{table:main} by reporting all ambient RF methods and all tested feature budgets. The downstream RF comparison is reported separately in Table~\ref{tab:vertex_velocity_random_feature_comparison}. Clearly, MRFs outperform all Euclidean baselines across every surface and every metric, even when the baselines are allowed up to $4\times$ more features.
\begin{table*}
\centering
\caption{Kernel reconstruction accuracy: MRFs vs. ambient-space
Euclidean random features with tuned $\sigma$.
MRF vectors have dimension 4000; baselines are evaluated at 4000,
8000, and 16000 features.
RMSE and MAE in units of $10^{-4}$.
Best results \textbf{bolded}.}
% \caption{Comparison between MRFs and other RF methods (with tuned parameters). \# features indicates the dimension of the random feature vectors (the MRF vectors here have a dimension of 4000, so even with additional features the other methods do not perform better). RMSE and MAE are reported in units of $10^{-4}$. Higher $R^2$ is better; lower values are better for all other metrics; top performance is \textbf{bolded}.}
\resizebox{\textwidth}{!}{%%
\begin{tabular}{lllccccccc}
\toprule
\textbf{Surface} & \textbf{Method} & \textbf{\# features} & \textbf{$R^2 \uparrow$} & \textbf{RMSE $\downarrow$} & \textbf{MAE $\downarrow$} & \textbf{Mean RE $\downarrow$} & \textbf{Median RE $\downarrow$} & \textbf{Rel.\ Frob.\ $\downarrow$} \\
\midrule

\multirow{13}{*}{\textbf{Sphere}} & \textbf{MRF} & -- & \textbf{0.995} & \textbf{242.71} & \textbf{59.11} & \textbf{0.039} & \textbf{0.005} & \textbf{0.067} \\
\cmidrule(lr){2-9}
 & RFF & \multirow{4}{*}{4000} & 0.975 & 556.09 & 433.66 & 33.418 & 26.056 & 0.153 \\
 & ORF &  & 0.976 & 553.63 & 419.85 & 29.771 & 23.221 & 0.153 \\
 & PRF &  & 0.147 & 3266.96 & 1647.41 & 89.222 & 78.716 & 0.901 \\
 & PORF &  & 0.257 & 3048.50 & 935.74 & 20.537 & 4.146 & 0.841 \\
\cmidrule(lr){2-9}
 & RFF & \multirow{4}{*}{8000} & 0.985 & 432.81 & 323.86 & 23.246 & 18.064 & 0.119 \\
 & ORF &  & 0.983 & 463.96 & 331.72 & 21.359 & 16.734 & 0.128 \\
 & PRF &  & 0.235 & 3094.36 & 890.47 & 16.485 & 3.055 & 0.853 \\
 & PORF &  & 0.351 & 2850.61 & 1019.82 & 28.066 & 6.216 & 0.786 \\
\cmidrule(lr){2-9}
 & RFF & \multirow{4}{*}{16000} & 0.989 & 367.93 & 253.48 & 15.913 & 12.253 & 0.101 \\
 & ORF &  & 0.987 & 403.97 & 270.26 & 15.857 & 12.390 & 0.111 \\
 & PRF &  & 0.391 & 2759.94 & 1027.42 & 28.759 & 6.300 & 0.761 \\
 & PORF &  & 0.269 & 3023.77 & 931.91 & 20.888 & 4.484 & 0.834 \\
\midrule
\multirow{13}{*}{\textbf{Ellipsoid}} & \textbf{MRF} & -- & \textbf{0.998} & \textbf{0.56} & \textbf{0.10} & \textbf{0.004} & \textbf{0.000} & \textbf{0.047} \\
\cmidrule(lr){2-9}
 & RFF & \multirow{4}{*}{4000} & 0.904 & 3.63 & 1.93 & 0.159 & 0.109 & 0.304 \\
 & ORF &  & 0.898 & 3.74 & 1.94 & 0.158 & 0.102 & 0.313 \\
 & PRF &  & 0.144 & 10.81 & 3.05 & 0.161 & 0.040 & 0.905 \\
 & PORF &  & 0.247 & 10.14 & 3.50 & 0.224 & 0.068 & 0.849 \\
\cmidrule(lr){2-9}
 & RFF & \multirow{4}{*}{8000} & 0.909 & 3.52 & 1.64 & 0.130 & 0.076 & 0.295 \\
 & ORF &  & 0.902 & 3.66 & 1.70 & 0.135 & 0.074 & 0.306 \\
 & PRF &  & 0.305 & 9.74 & 3.70 & 0.253 & 0.090 & 0.815 \\
 & PORF &  & 0.304 & 9.75 & 3.80 & 0.265 & 0.092 & 0.816 \\
\cmidrule(lr){2-9}
 & RFF & \multirow{4}{*}{16000} & 0.911 & 3.48 & 1.46 & 0.112 & 0.054 & 0.291 \\
 & ORF &  & 0.906 & 3.58 & 1.53 & 0.118 & 0.056 & 0.300 \\
 & PRF &  & 0.326 & 9.59 & 3.88 & 0.276 & 0.099 & 0.803 \\
 & PORF &  & 0.327 & 9.58 & 3.94 & 0.283 & 0.107 & 0.802 \\
\midrule
\multirow{13}{*}{\textbf{M\"obius strip}} & \textbf{MRF} & -- & \textbf{0.999} & \textbf{0.39} & \textbf{0.06} & \textbf{0.002} & \textbf{0.000} & \textbf{0.035} \\
\cmidrule(lr){2-9}
 & RFF & \multirow{4}{*}{4000} & 0.738 & 5.67 & 2.08 & 0.140 & 0.093 & 0.499 \\
 & ORF &  & 0.738 & 5.67 & 2.08 & 0.141 & 0.094 & 0.499 \\
 & PRF &  & 0.120 & 10.39 & 3.06 & 0.172 & 0.026 & 0.915 \\
 & PORF &  & 0.154 & 10.19 & 3.21 & 0.192 & 0.031 & 0.898 \\
\cmidrule(lr){2-9}
 & RFF & \multirow{4}{*}{8000} & 0.745 & 5.60 & 1.83 & 0.116 & 0.068 & 0.493 \\
 & ORF &  & 0.744 & 5.60 & 1.80 & 0.113 & 0.065 & 0.493 \\
 & PRF &  & 0.203 & 9.89 & 3.51 & 0.233 & 0.051 & 0.871 \\
 & PORF &  & 0.192 & 9.96 & 3.43 & 0.221 & 0.043 & 0.877 \\
\cmidrule(lr){2-9}
 & RFF & \multirow{4}{*}{16000} & 0.748 & 5.56 & 1.65 & 0.098 & 0.049 & 0.489 \\
 & ORF &  & 0.748 & 5.56 & 1.63 & 0.096 & 0.047 & 0.490 \\
 & PRF &  & 0.216 & 9.81 & 3.55 & 0.239 & 0.051 & 0.864 \\
 & PORF &  & 0.212 & 9.83 & 3.58 & 0.241 & 0.053 & 0.866 \\
\midrule
\multirow{13}{*}{\textbf{Torus}} & \textbf{MRF} & -- & \textbf{0.996} & \textbf{0.73} & \textbf{0.12} & \textbf{0.005} & \textbf{0.000} & \textbf{0.062} \\
\cmidrule(lr){2-9}
 & RFF & \multirow{4}{*}{4000} & 0.817 & 4.96 & 2.25 & 0.180 & 0.106 & 0.418 \\
 & ORF &  & 0.803 & 5.16 & 2.35 & 0.188 & 0.103 & 0.434 \\
 & PRF &  & 0.138 & 10.77 & 3.30 & 0.190 & 0.057 & 0.908 \\
 & PORF &  & 0.189 & 10.45 & 3.66 & 0.237 & 0.076 & 0.881 \\
\cmidrule(lr){2-9}
 & RFF & \multirow{4}{*}{8000} & 0.821 & 4.91 & 2.00 & 0.155 & 0.072 & 0.414 \\
 & ORF &  & 0.806 & 5.11 & 2.14 & 0.167 & 0.073 & 0.430 \\
 & PRF &  & 0.221 & 10.24 & 3.97 & 0.276 & 0.107 & 0.863 \\
 & PORF &  & 0.220 & 10.24 & 3.92 & 0.269 & 0.098 & 0.863 \\
\cmidrule(lr){2-9}
 & RFF & \multirow{4}{*}{16000} & 0.818 & 4.95 & 1.91 & 0.146 & 0.052 & 0.417 \\
 & ORF &  & 0.812 & 5.02 & 1.98 & 0.152 & 0.054 & 0.423 \\
 & PRF &  & 0.232 & 10.17 & 4.02 & 0.281 & 0.108 & 0.857 \\
 & PORF &  & 0.222 & 10.23 & 4.01 & 0.279 & 0.109 & 0.862 \\
\bottomrule
\end{tabular}}
\label{tab:tuned-sweep-best}
\end{table*}

\subsection{Mesh interpolation experimental details}
\label{app:mesh-interpolation}

\subsubsection{Vertex normal prediction setup}
\label{app:vertex_normal_prediction_setup}

We consider vertex normal interpolation on a triangular mesh with vertex set $\mathrm{V}$, faces $F$, and vertex positions $\{\mathbf{x}_i\in\mathbb{R}^3\}_{i\in \mathrm{V}}$. Ground-truth vertex normals $\{\mathbf{n}_i\in\mathbb{R}^3\}_{i\in \mathrm{V}}$ are computed by accumulating incident (unnormalized) face normals and normalizing each vertex vector. We randomly mask a subset $M\subseteq \mathrm{V}$ with $|M|=0.8|\mathrm{V}|$ (80\% missing), set $\widetilde{\mathbf{n}}_i=\mathbf{0}$ for $i\in M$ and $\widetilde{\mathbf{n}}_i=\mathbf{n}_i$ otherwise, and predict normals at the masked vertices via kernel-based field integration. Performance is reported as the mean cosine similarity on the masked set,
$\frac{1}{|M|}\sum_{i\in M}(\mathbf{n}_i^{\mathrm{pred}})^\top \mathbf{n}_i$,
where $\mathbf{n}_i^{\mathrm{pred}}$ is the unit-normalized prediction. The mesh graph $G=(\mathrm{V},E)$ uses the mesh edges induced by faces. For each edge $(i,j)\in E$, we define an affinity
\[
(\mathbf{W})_{ij}=\exp\!\Big(-\frac{\|\mathbf{x}_i-\mathbf{x}_j\|^2}{\sigma^2}\Big),\qquad
\sigma^2=\mathrm{median}_{(i,j)\in E}\|\mathbf{x}_i-\mathbf{x}_j\|^2,
\]
compute degrees $(\mathbf{D})_{ii}=\sum_j (\mathbf{W})_{ij}$, and form the symmetrically-normalized matrix $\mathbf{W}_f=\mathbf{D}^{-1/2}\mathbf{W}\mathbf{D}^{-1/2}$. The full-kernel (FK) baseline explicitly constructs the dense kernel $\mathbf{K}=\exp(\tau \mathbf{W}_f)$, with $\tau=20$, by eigen-decomposition of the dense $\mathbf{W}_f$ (double precision) and then applies $\mathbf{N}^{\mathrm{pred}} = \mathbf{K}\widetilde{\mathbf{N}}$ with $\widetilde{\mathbf{N}}\in\mathbb{R}^{|\mathrm{V}|\times 3}$ the masked normal field (masked rows set to $\mathbf{0}$), followed by row-wise unit normalization. We record FK preprocessing time as dense $\mathbf{W}_f$ construction plus eigendecomposition and kernel formation, and FK interpolation time as the matrix--field product $\mathbf{K}\widetilde{\mathbf{N}}$ plus normalization.

MRFs target the same kernel $\mathbf{K}=\exp(\tau \mathbf{W}_f)$ but avoid explicit construction by learning a low-rank factorization. We first estimate signature values corresponding to $\exp(\tfrac{\tau}{2}\mathbf{W}_f)$ using the g-GRF random-walk estimator with modulation coefficients $\alpha_k=(\tfrac{\tau}{2})^k/k!$ (truncated when $\alpha_k<10^{-300}$), halting probability $p_{\mathrm{halt}}=0.01$, and $m=10{,}000$ walks per start node. We sample 1000 start nodes uniformly without replacement (restricted to vertices when a denser discretization is used). For each start, we sample candidate anchors $\omega$ uniformly, keep all samples with target value at least $0.1$, and retain smaller targets with probability $0.025$. We also train a continuous surrogate $g_\theta(\mathbf{x}_{\text{start}},\mathbf{x}_\omega)$ using a 3-layer MLP with two 128-wide ReLU hidden layers. Optimization uses Adam with learning rate $10^{-3}$, batch size $32{,}768$, $1000$ epochs, a relative-error loss with $\varepsilon=0.1$, and a $20\%$ validation split.

For training $g_\theta$ we use a denser point set on the surface: by default we set $N_{\mathrm{dense}}=\max(|\mathrm{V}|,5000)$ by adding area-weighted samples on faces and build a $k$NN graph on these points with $k=16$. Inference remains on the original vertices: we sample $n_{\mathrm{rf}}=256$ anchors $\{\omega_\ell\}$ uniformly from the dense set and approximate $\mathbf{K}$ via $\mathbf{K}\approx \boldsymbol{\Phi}\boldsymbol{\Phi}^\top$ with $(\boldsymbol{\Phi})_{i\ell}=g_\theta(\mathbf{x}_i,\mathbf{x}_{\omega_\ell},d(i,\omega_\ell))/\sqrt{n_{\mathrm{rf}}}$, applying $\mathbf{N}^{\mathrm{pred}}\approx \boldsymbol{\Phi}(\boldsymbol{\Phi}^\top \widetilde{\mathbf{N}})$ and normalizing row-wise. MRF preprocessing time includes supervision generation and training, and MRF interpolation time includes anchor sampling, distance computation, network evaluation, and accumulation. We evaluate on meshes from Thingi10k with the following mesh IDs:

\texttt{[368622, 42435, 65282, 116878, 409624, 101902, 73410, 87602, 255172, 98480, 57140, 285606, 96123, 203289, 87601, 409629, 37384, 57084, 136024, 202267, 101619, 72896, 90064, 127243, 78671, 285610, 75667, 80597, 75651, 75654, 75657, 75665, 75652, 123472, 88855, 444375, 208741, 73877]}

\subsubsection{More details regarding velocity prediction setup}
Since velocity magnitudes are meaningful (unlike unit normals), we use a normalized kernel interpolant.
Let $\mathbf{m}\in\{0,1\}^{|\widetilde{\mathrm V}|}$ be the indicator of observed nodes ($m_i=1$ if $i\notin M$, else $0$), and let $\mathbf{m}\odot \mathbf{U}$ denote row-wise masking.
We predict the full field via normalization:
\begin{equation}
\mathbf{U}^{\mathrm{pred}}
\;=\;
\frac{\mathbf{K}(\mathbf{m}\odot \mathbf{U})}{\mathbf{K}\mathbf{m}}\,,
\end{equation}
where the division is row-wise and applied to each of the $3$ velocity channels.

To match the full-kernel baseline used in vertex-normal prediction, we compute $\mathbf{K}$ via a full spectral decomposition.
Specifically, we form $\mathbf{W}_f$ and compute $\mathbf{W}_f = \mathbf{V}\operatorname{diag}(\boldsymbol{\lambda})\mathbf{V}^\top$, so that
$\mathbf{K} = \mathbf{V}\operatorname{diag}(\exp(\tau\boldsymbol{\lambda}))\mathbf{V}^\top$.
Baseline preprocessing time therefore includes graph construction and the dense eigendecomposition (dominated by $O(|\widetilde{\mathrm V}|^3)$ time and
$O(|\widetilde{\mathrm V}|^2)$ memory to store $\mathbf{V}$), and each per-frame interpolation requires two dense applications of $\mathbf{K}$
(we apply $\mathbf{K}$ once to the stacked right-hand side $[\mathbf{m},\;\mathbf{m}\odot \mathbf{U}]$ to obtain both numerator and denominator).

For MRFs, we train the regressor $g_\theta$ once per discretization on the same dense graph,
and precompute a random-feature matrix $\mathbf{Z}\in\mathbb{R}^{|\widetilde{\mathrm V}|\times M}$ for a fixed number of features $M$.
This yields the approximation $\mathbf{K} \approx \mathbf{Z}\mathbf{Z}^\top$, so that per-frame interpolation reduces to two thin matrix multiplications:
$\mathbf{Z}\mathbf{Z}^\top(\mathbf{m}\odot \mathbf{U}) = \mathbf{Z}(\mathbf{Z}^\top(\mathbf{m}\odot \mathbf{U}))$ and similarly for $\mathbf{K}\mathbf{m}$.
We report MRF preprocessing as supervision generation, training, and feature precomputation, and MRF interpolation time as the rest.

\subsection{Manifold-valued attention: full masked-reconstruction results}
\label{app:attention}
Table~\ref{tab:attention-full} reports all Euclidean RF, graph/geodesic low-rank, exact-kernel, and linear-kernel baselines for every source count. Specifically, we compared MRFs against the full graph heat kernel (Exact heat), a low-rank Nyström approximation of that same manifold heat kernel (Geodesic Nystrom), a low-rank spectral approximation of that same manifold heat kernel (Graph diffusion map), the exact Gaussian kernel in ambient Euclidean space (Exact RBF), a low-rank anchor-feature approximation of that ambient Gaussian kernel (Linear RBF), and other RF methods. 

\begin{table*}[h] 
\centering
\caption{Masked reconstruction results for source counts 5--10. Lower is better for both metrics.}
\label{tab:attention-full}
\resizebox{\textwidth}{!}{%
\begin{tabular}{lcccccccccccccc}
\toprule
 & \multicolumn{2}{c}{Sources 5} & \multicolumn{2}{c}{Sources 6} & \multicolumn{2}{c}{Sources 7} & \multicolumn{2}{c}{Sources 8} & \multicolumn{2}{c}{Sources 9} & \multicolumn{2}{c}{Sources 10} \\
\cmidrule(lr){2-3} \cmidrule(lr){4-5} \cmidrule(lr){6-7} \cmidrule(lr){8-9} \cmidrule(lr){10-11} \cmidrule(lr){12-13}
Method & MSE & Rel. $\ell_2$ & MSE & Rel. $\ell_2$ & MSE & Rel. $\ell_2$ & MSE & Rel. $\ell_2$ & MSE & Rel. $\ell_2$ & MSE & Rel. $\ell_2$ \\
\midrule
MRF & \textbf{3.10} & \textbf{1.97} & \textbf{3.03} & 2.90 & \textbf{2.90} & \textbf{4.87} & \textbf{3.11} & \textbf{2.15} & \textbf{3.15} & \textbf{3.46} & \textbf{3.25} & \textbf{2.46} \\
Geodesic Nystr\"om & 3.77 & 2.05 & 4.76 & 4.49 & 6.38 & 7.54 & 5.74 & 4.40 & 4.52 & 6.87 & 4.59 & 5.22 \\
Graph diffusion map & 3.61 & 2.67 & 3.67 & 3.37 & 3.98 & 7.27 & 3.96 & 4.11 & 3.72 & 5.37 & 3.70 & 4.52 \\
RFF & 3.32 & 2.36 & 3.20 & 3.40 & 4.14 & 6.76 & 3.51 & 3.30 & 3.40 & 5.03 & 3.42 & 3.52 \\
ORF & 3.26 & 2.24 & 3.15 & 3.16 & 3.84 & 6.69 & 3.47 & 3.20 & 3.37 & 4.92 & 3.38 & 3.38 \\
PRF & 3.34 & 2.52 & 3.05 & \textbf{2.23} & 3.31 & 5.64 & 3.79 & 4.63 & 3.30 & 4.51 & 3.65 & 4.20 \\
PORF & 3.52 & 2.75 & 3.16 & 3.18 & 3.33 & 6.47 & 4.18 & 5.26 & 3.25 & 3.97 & 3.53 & 4.01 \\
Exact RBF & 3.38 & 2.47 & 3.22 & 3.53 & 4.66 & 7.36 & 3.60 & 3.48 & 3.48 & 5.41 & 3.44 & 3.65 \\
Linear RBF & 3.29 & 2.42 & 3.11 & 3.20 & 3.92 & 6.72 & 3.40 & 3.07 & 3.35 & 4.90 & 3.38 & 3.42 \\
Exact heat & 3.87 & 2.00 & 5.00 & 4.36 & 6.91 & 8.08 & 5.93 & 4.47 & 4.47 & 5.82 & 4.81 & 5.67 \\
\bottomrule
\end{tabular}%
}
\end{table*}

\subsection{Higher-dimensional descriptor manifolds: Brodatz} \label{app:brodatz}
We evaluate MRFs on a harder Brodatz protocol designed to test whether unlabeled transformed samples help recover a useful descriptor manifold. Each grayscale Brodatz image defines one class, and each sampled window is mapped to a 15-dimensional log-SPD covariance descriptor computed from intensity, first-order gradients, and second-order derivatives (i.e., each window becomes a $5\times5$ SPD covariance matrix). Supervision is intentionally sparse: for each class we use only 6 labeled windows drawn from 2 canonical rotations $(0^{\circ}$ and $90^{\circ})$. The manifold is built by sampling a large unlabeled pool from intermediate rotations $(0^{\circ},15^{\circ},45^{\circ},75^{\circ},90^{\circ},105^{\circ},135^{\circ},165^{\circ})$. Evaluation is performed on held-out rotations $(30^{\circ},60^{\circ},120^{\circ},150^{\circ})$. This yields 111 classes, 666 labeled training descriptors, 66{,}600 unlabeled graph descriptors, and 22{,}200 test descriptors. As in the paper, we build a graph on the unlabeled transformed descriptor cloud, compute GRF-style random-walk signatures, train the continuous surrogate, then use the resulting MRF features for $k$NN classification. MRFs noticeably outperform other methods. The metric is group accuracy, where predictions from 50 windows belonging to the same texture-angle block are aggregated by majority vote. Tuned $k \in {1,...,15}$ for $k$NN; $k=1$ was best in every case.

% \begin{table}[h]
% \centering
% \caption{Brodatz transformed-view classification with group-vote evaluation. Each method family was swept over $k \in \{1,\ldots,15\}$ for $k$-NN; $k=1$ was best in every case, so we report the corresponding best-tuned result for each method.}
% \label{tab:brodatz}
% \begin{tabular}{lcc}
% \toprule
% \textbf{Method} & \textbf{Best setting} & \textbf{Group acc. (\%)} \\
% \midrule
% MRF std & $k=1$ & \textbf{58.33} \\
% RFF-RBF & dim$=64$, $k=1$ & 44.37 \\
% raw log-SPD & $k=1$ & 41.89 \\
% Nystr\"om-RBF & dim$=128$, $k=1$ & 41.89 \\
% kPCA-RBF & $l=15$, $k=1$ & 32.43 \\
% \bottomrule
% \end{tabular}
% \label{tab:brodatz-transform-group}
% \end{table}

\begin{table*}[h]
\centering
\small
\setlength{\tabcolsep}{4.5pt}
\caption{Brodatz transformed-view classification with group-vote evaluation. Entries are group accuracy in percent, reported as mean $\pm$ standard deviation over 30 seeds. For RF and kernel methods, the feature dimension or embedding dimension is tuned within each method and $k$.}
\label{tab:brodatz-transform-group}
\begin{tabular}{lccccc}
\toprule
\textbf{Method} & $\mathbf{k=1}$ & $\mathbf{k=3}$ & $\mathbf{k=5}$ & $\mathbf{k=10}$ & $\mathbf{k=15}$ \\
\midrule
MRF & \textbf{55.7 $\pm$ 2.9} & \textbf{53.6 $\pm$ 3.3} & \textbf{51.6 $\pm$ 3.2} & \textbf{45.4 $\pm$ 3.3} & \textbf{40.3 $\pm$ 3.0} \\
Raw & 42.4 $\pm$ 2.2 & 36.9 $\pm$ 2.4 & 33.6 $\pm$ 2.7 & 27.2 $\pm$ 2.2 & 24.1 $\pm$ 1.6 \\
% Nystr\"om & 42.4 $\pm$ 2.1 & 37.0 $\pm$ 2.4 & 33.8 $\pm$ 2.7 & 27.7 $\pm$ 2.2 & 24.7 $\pm$ 1.7 \\
ORF & 42.6 $\pm$ 2.3 & 37.5 $\pm$ 2.4 & 34.5 $\pm$ 2.5 & 27.8 $\pm$ 1.9 & 24.9 $\pm$ 1.7 \\
RFF & 43.5 $\pm$ 2.5 & 38.5 $\pm$ 2.8 & 35.4 $\pm$ 3.1 & 28.9 $\pm$ 2.6 & 25.6 $\pm$ 2.4 \\
kPCA & 35.3 $\pm$ 2.4 & 30.5 $\pm$ 2.5 & 28.5 $\pm$ 2.8 & 24.8 $\pm$ 2.1 & 22.6 $\pm$ 1.6 \\
PRF & 9.2 $\pm$ 2.4 & 6.3 $\pm$ 1.2 & 5.2 $\pm$ 1.1 & 3.7 $\pm$ 0.8 & 3.0 $\pm$ 0.4 \\
PORF & 8.9 $\pm$ 2.5 & 5.9 $\pm$ 1.4 & 5.0 $\pm$ 1.4 & 3.6 $\pm$ 0.8 & 3.0 $\pm$ 0.8 \\
\bottomrule
\end{tabular}
\end{table*}

\subsection{Non-Compact Manifolds}
\label{app:non-compact-manifolds}

Non-compact manifolds require additional care because a finite graph can only approximate bounded regions. Rather than attempting a global discretization, we adopt an exhaustion-based (local-on-compacts) approach: we approximate the heat kernel using finite constructions on balls $B_R(o)$, and evaluate accuracy on a fixed inner region $B_r(o)$. Empirically, we observe that approximation error on $B_r(o)$ decreases as outer radii $R$ increases. We evaluate the approximation on a fixed inner ball with $r=1.5$, using 128 query nodes sampled from that inner region and comparing against the analytic heat kernel at time $t=0.5$. GRF and MRF walk budgets are increased with $R$.

\begin{table}[h]
\centering
\small
\setlength{\tabcolsep}{4.5pt}
\caption{Local-on-compacts exhaustion results for the $\mathbb{H}^3$ MRF experiment. Errors are evaluated on a fixed inner test ball while increasing the outer truncation radius $R$. RMSE values are reported in units of $10^{-3}$.}
\label{tab:h3-mrf-local-exhaustion}
\begin{tabular}{lccccccc}
\toprule
 & $\mathbf{R=2.0}$ & $\mathbf{R=2.5}$ & $\mathbf{R=3.0}$ & $\mathbf{R=3.5}$ & $\mathbf{R=4.0}$ & $\mathbf{R=4.5}$ & $\mathbf{R=5.0}$ \\
\midrule
Rel. $\ell_2$ & 0.31 & 0.29 & 0.25 & 0.26 & 0.19 & 0.21 & 0.18 \\
RMSE $(10^{-3})$ & 3.3 & 3.1 & 2.7 & 2.8 & 2.1 & 2.3 & 1.9 \\
\bottomrule
\end{tabular}
\end{table}

\subsection{Code availability}
\label{app:code}
Code for all our experiments can be found \href{https://anonymous.4open.science/r/graph-kernel-convergence-0612}{here}.

%%%%%%%%%%%%%%%%%%%%%%%%%%%%%%%%%%%%%%%%%%%%%%%%%%%%%%%%%%%%%%%%%%%%%%%%%%%%%%%
% NeurIPS checklist
%%%%%%%%%%%%%%%%%%%%%%%%%%%%%%%%%%%%%%%%%%%%%%%%%%%%%%%%%%%%%%%%%%%%%%%%%%%%%%%
% \newpage
% \input{checklist.tex}

\end{document}